\documentclass[twoside,11pt]{article}

\usepackage{jmlr2e}
\usepackage[toc,page]{appendix}
\usepackage{comment}
\usepackage{times}
\usepackage{amsmath}
\usepackage{mathtools, stmaryrd}
\usepackage{caption}
\usepackage{subcaption}
\usepackage{ifthen} 
\usepackage{textcomp} 
 
\usepackage[linesnumbered, ruled,vlined]{algorithm2e}
\usepackage{caption,booktabs,array}
\usepackage{multirow}
\usepackage{dsfont}
\usepackage{tikz}
\usepackage{xcolor}
\usepackage[utf8]{inputenc}

\definecolor{kiliancolor}{HTML}{00802b}

\newcommand{\1}{\hspace{0.2mm}\text{I}\hspace{0.2mm}}
\newcommand{\II}{\text{I \hspace{-2.13mm} I} }

\newcommand{\N}{\mathbb{N}}
\newcommand{\E}{\mathbb{E}}
\newcommand{\GW}{\mathcal{GW}}
\renewcommand{\aa}{\mathbf{a}}
\newcommand{\bb}{\mathbf{b}}
\newcommand{\uu}{\mathbf{u}}
\renewcommand{\S}{\Sigma}%

\newcommand\V{\vert}
\renewcommand{\a}{\alpha}

\newcommand{\Pa}{\mathcal{P}_m(\alpha_n)}

\newcommand{\Pim}{\mathcal{P}^m}
\newcommand{\BPa}{\overline{\mathcal{P}_m(\alpha_n)}}

\newcommand{\Pb}{\mathcal{P}_m(\beta_n)}

\newcommand{\BPb}{\overline{\mathcal{P}_m(\beta_n)}}
\newcommand{\an}{\alpha _n}
\newcommand{\bn}{\beta _n}
\newcommand{\s}{\sigma}
\renewcommand{\P}{\mathbb{P}}
\newcommand{\e}{\varepsilon}
\newcommand{\la}{\lambda}
\newcommand{\del}{\delta}
\newcommand{\ind}{\mathbf{1}}

\def\xx{{\boldsymbol x}}
\def\yy{{\boldsymbol y}}
\def\zz{{\boldsymbol z}}

\def\DD{{\boldsymbol D}}

\def\XX{{\boldsymbol X}}
\def\ZZ{{\boldsymbol Z}}

\def\bXX{{\mathbb X}}
\def\bYY{{\mathbb Y}}
\def\YY{{\boldsymbol Y}}
\def\UU{{\boldsymbol U}}

\def\R{{\mathbb{R}}}
\def\PP{{\boldsymbol P}}
\def\balpha{{\boldsymbol \alpha}}

\def\expect{\mathop{\mathbb{E}}}
\def\simplex{\Sigma}
\newtheorem{innercustomgeneric}{\customgenericname}
\providecommand{\customgenericname}{}
\newcommand{\newcustomtheorem}[2]{%
  \newenvironment{#1}[1]
  {%
   \renewcommand\customgenericname{#2}%
   \renewcommand\theinnercustomgeneric{##1}%
   \innercustomgeneric
  }
  {\endinnercustomgeneric}
}

\newcustomtheorem{customdef}{Definition}
\newcustomtheorem{customtheorem}{Theorem}
\newcustomtheorem{customcorollary}{Corollary}
\newcustomtheorem{customexample}{Example}
\newcustomtheorem{customlemma}{Lemma}
\newcustomtheorem{customprop}{Proposition}
   
\newcommand{\rf}[1]{{\color{blue} #1}}
\newcommand{\yz}[1]{{\color{magenta} #1}}
\newcommand{\nc}[1]{{\color{violet} {\sc nico:} #1}}
\newcommand{\kf}[1]{{\color{kiliancolor} #1}}
\newcommand{\rg}[1]{{\color{red} #1}}

\newcommand{\defas}{\;\mathrel{\!\!{:}{=}\,}}
\newcommand{\fact}[1]{#1\mathpunct{}!}

\jmlrheading{1}{2020}{1-48}{4/00}{10/00}{meila00a}{Fatras, Zine, Majewski, Flamary, Gribonval and Courty}

\ShortHeadings{Minibatch optimal transport distances; analysis and applications}{Fatras, Zine, Majewski, Flamary, Gribonval and Courty}
\firstpageno{1}

\begin{document}

\title{Minibatch optimal transport distances; analysis and applications}

\author{\name Kilian Fatras \email kilian.fatras@irisa.fr \\
       \addr Univ. Bretagne-Sud, CNRS, \textsc{Inria}, IRISA, France
       \AND
       \name Younes Zine \email y.p.zine@sms.ed.ac.uk
 \\
       \addr School of Mathematics, The University of Edinburgh and \\
    The Maxwell Institute for the Mathematical Sciences\\
Edinburgh, United Kingdom\\
Univ Rennes, CNRS, IRMAR,\\
UMR 6625, F-35000 Rennes, France
       \AND
       \name Szymon Majewski \email szymon.majewski@polytechnique.edu \\
       \addr École Polytechnique, CMAP, France
       \AND
       \name Rémi Flamary \email remi.flamary@polytechnique.edu \\
       \addr École Polytechnique, CMAP, France
       \AND
       \name Rémi Gribonval \email remi.gribonval@inria.fr \\
       \addr Univ Lyon, Inria, CNRS, ENS de Lyon, UCB Lyon 1, \\
             LIP UMR 5668, F-69342, Lyon, France
       \AND
       \name Nicolas Courty \email nicolas.courty@irisa.fr \\
       \addr Univ. Bretagne-Sud, CNRS, \textsc{Inria}, IRISA, France
       }


\maketitle

\begin{abstract}
Optimal transport distances have become a classic tool to compare probability distributions and have found many applications in machine learning.
Yet, despite recent algorithmic developments, their complexity prevents their direct use on large scale datasets. To overcome this challenge, a common workaround is to compute these distances on minibatches {\em i.e.} to average the outcome of several smaller optimal transport problems. We propose in this paper an extended analysis of this practice, which effects were previously studied in restricted cases. We first consider a large variety of Optimal Transport kernels. We notably argue that the minibatch strategy comes with appealing properties such as unbiased estimators, gradients and a concentration bound around the expectation, but also with limits: the minibatch OT is not a distance. To recover some of the lost distance axioms, we introduce a debiased minibatch OT function and study its statistical and optimisation properties. Along with this theoretical analysis, we also conduct empirical experiments on gradient flows, generative adversarial networks (GANs) or color transfer that highlight the practical interest of this strategy.
\end{abstract}

\section{Introduction}

Comparing probability distributions is a fundamental problem in machine learning. The difficulty is to find a relevant distance with good statistical and optimization properties to obtain such comparisons. The Wasserstein distance has been used for this purpose in several machine learning problems such as: generative modeling, where one wants to fit a generated data distribution to a training data distribution \citep{goodfellow_gan}; domain adaptation, with the goal to leverage on existing labelled data on a given source domain to perform classification on a target domain where none or few labels is available \citep{DACourty} ; classification, for multi-label \citep{frogner_2015} and adversarial robustness, where the Wasserstein distance has been shown to be more robust to rotations and translations of data than $l_1$ norm for instance \citep{wong19a}. In order to define a measure between two probability distributions, the Wasserstein distance, based on optimal transport (OT), takes advantage of a ground cost on the space where the probability distributions lie. 
One particularly interesting property of the Wasserstein distance is that it can be used between distributions that do not share the same support, which is frequently the case when dealing with empirical distributions in many Machine Learning problems.

Computing the Wasserstein distance between empirical probability distributions with $n$ points has a complexity of $\mathcal{O}(n^3\log(n))$ (Chapter 3 \citep{COT_Peyre}), which implies that it can not be used in practice in a big data scenario. To decrease this complexity, an appealing technique is to regularize the Wasserstein distance with an entropic term \citep{CuturiSinkhorn}. This allowed the use of the efficient Sinkhorn-Knopp algorithm that can be implemented in parallel and has a lower computational complexity of $\mathcal{O}(n^2)$ \citep{altschuler2017near}, which is still prohibitive for many large scale applications. Greedy variants of the Sinkhorn-Knopp algorithm can be found in \citep{altschuler2017near, abid18a}. Many strategies have been deployed to accelerate the computation of optimal transport, for instance stochastic solvers have been investigated to solve the entropic regularized OT in \citep{Genevay_2016, Ballu2020Stochastic, seguy2018large}. Other variants take advantage of the 1D closed form of optimal transport with the so called \emph{Sliced Wasserstein Distance}, \citep{Bonnotte2013, kolouri2016sliced, liutkus19a}. There are also  hierarchical or multiscale strategies to compute an approximation of optimal transport \citep{HierarchicalOT, JMLRGerber}. Despite the good empirical performance of the Wasserstein distance on generative modeling \citep{genevay_2018, genevay19, arjovsky_2017}, it was recently proved that using the empirical Wasserstein distance as a loss function does not lead to the optimal solution. It is due to the estimator bias of the Wasserstein distance between continuous probability distributions \citep{Bellemare_cramerGAN, Genevay_phd}.

In order to train a neural network on large scale datasets with the Wasserstein distance, several works had the idea to rely on a minibatch computation of Optimal Transport distances and backpropagate the resulting gradient into the network \citep{genevay_2018, deepjdot}. This strategy leads to a complexity of $\mathcal{O}(k m^2)$, where $m$ is the batch size  and $k$ the number of considered batches. However, the price to pay when computing the average of several OT quantities between minibatches from inputs is a change in the original problem. Indeed, minimizing minibatch OT minimizes the expectation of optimal transport between minibatches of size $m$ and not the optimal transport between the original measures.
To control the approximation error, \citep{mbot_Sommerfeld} established a non-asymptotic deviation bound between the original optimal transport distance and its minibatch version.
Recently in the context of generative models, \citep{bernton2019parameter} showed the convergence of the minibatch minimizers to the true minimizers when the batch size $m$ increases. However, while the approximation with the Wasserstein distance has been well studied, many questions remain unsolved regarding the learning properties of the minibatch strategy. This includes in particular: non optimal connections between samples on transport plans; statistical estimation properties between the empirical counter part and the expectation; the optimization with stochastic gradient (SGD); and finally solutions to limit the bias of the minibatch OT losses. In a previous work \citep{fatras2019batchwass}, we partially answered the above questions for minibatch OT. After setting a rigorous formalism of minibatch OT losses for a sampling without replacement, \emph{i.e., when there are not repeated indices of data within a minibatch}, we studied their statistical and optimization properties in the case of \emph{uniform and bounded measures}. We also found that the minibatch OT losses do not respect the separation axiom, breaking the mathematical definition of a distance.

{ In this paper we propose to complete our previous work to a more general setting with relaxed hypothesis on probability distributions. We also consider a relatively larger number of optimal transport variants. We construct estimators with a general formalism for designing minibatch of data which respects the probability distribution constraints and we propose a new minibatch OT loss function. We show that our new estimators enjoy appealing statistical and optimisation properties. Finally, we study the performance of minibatch OT as losses for several machine learning applications.} 

The paper is structured as follows: in Section~\ref{sec:RW}, we do a brief review of the different optimal transport losses. In Section~\ref{sec:MBW}, we formalize minibatch OT losses, show basic properties and learning behaviors. Then we present our main results, a new loss function based on minibatch OT losses which respects the separability axioms. In Section \ref{sec:learning}, we give concentration bonds of minibatch OT losses in bounded and unbounded data scenario and study the use of SGD for minimizing minibatch OT. And finally, in Section~\ref{sec:exp}, we describe experiments using minibatch optimal transport. 

\section{Wasserstein distance and variants}

This section defines the classical OT problems and discuss their numerical complexity.

\paragraph{Wasserstein distance}\label{sec:RW}
Let $\mathcal{M}_{+}^1(\mathcal{X})$ denote the set of all probability distributions lying in the space $\mathcal{X}$. The Optimal Transport metric measures a distance between two probability distributions $(\alpha, \beta) \in \mathcal{M}_{+}^1(\mathcal{X}) \times \mathcal{M}_{+}^1(\mathcal{X})$ by considering a ground cost $c$ on the space $\mathcal{X}$ . 
The Optimal Transport problem $W_{c}$ between two distributions is defined as :
\begin{equation}
    W_{c}(\alpha, \beta) = \underset{\pi \in \UU(\alpha, \beta)}{\text{min}} \int_{\mathcal{X}\times \mathcal{Y}}
    c(\xx,\yy) 
    d\pi(\xx,\yy) ,
\label{eq:wasserstein_1_dist}
\end{equation}
where $\UU(\alpha, \beta)$ is the set of joint probability distribution with
marginals $\alpha$ and $\beta$ such that 
$
\boldsymbol U(\alpha, \beta) = \left \{ \pi \in \mathcal{M}_{+}^1(\mathcal{X}, \mathcal{Y}): \PP_{\mathcal{X}}\#\pi = \alpha, \PP_{\mathcal{Y}}\#\pi = \beta \right\}\nonumber
$. {$\PP_{\mathcal{X}}\#\pi$ (resp. $\PP_{\mathcal{Y}}\#\pi$) is the marginalization of $\pi$ over $\mathcal{X}$ (resp. $\mathcal{Y}$)}. Where $\PP_{\mathcal{X}}$ is the projection over the space $\mathcal{X}$, and $\#$ denotes the pushforward operator, which can be defined as follow: for a continuous map $T : \mathcal{X} \mapsto \mathcal{Y}$ and for any measurable set $B \subset \mathcal{Y}$, $\beta(B) = \alpha \big( \{   \xx \in \mathcal{X} : T(\xx) \in B\} \big)$.

When the ground cost is a metric, the optimal transport problem becomes a metric between distributions called the 1-Wasserstein distance. In this work, we consider the Euclidean distance on $\R^d$ as ground metric, \emph{i.e.,} $c(\xx,\yy) = \|\xx - \yy\|_2$ and we denote in this case the $p$-Wasserstein distance as $W_p$ with $p\geq 1$. Formally:

\begin{equation}
    W_{p}(\alpha, \beta) = \left( \underset{\pi \in \UU(\alpha, \beta)}{\text{min}} \int_{\mathcal{X}\times \mathcal{Y}} \|\xx - \yy\|_2^p d\pi(\xx,\yy) \right)^{1/p},
\label{eq:wasserstein_dist}
\end{equation}


Note that the optimization problem above is called the Kantorovitch formulation of OT and the optimal $\pi$ is called an optimal transport plan when it is a minimizer of problem \eqref{eq:wasserstein_dist}. When the distributions are discrete, the problem becomes a discrete linear program that can be solved with a cubic complexity in the size of the distributions support \citep{COT_Peyre}. Also the \emph{sample complexity}, i.e. the convergence in population of the Wasserstein distance, is known to be slow with a rate  $O(n^{-1/d})$ depending on the dimensionality $d$ of the space $\mathcal{X}$ and the size of the population $n$ \citep{dudley1969, weed2019}.
Other computation strategies can be used such as multi-scale strategy in order to compute a fast approximation of the Wasserstein distance \citep{JMLRGerber}. We can find also a hierarchical strategy which leverages clustered structures in data and has a quadratic complexity in the size of the biggest cluster \citep{HierarchicalOT}. 
Lastly, the Wasserstein distance has a closed form when data lie in 1D spaces. If the data are sorted, then the optimal transport plan is the identity. Hence solving the Wasserstein distance in 1D is equivalent to sort in $\mathcal{O}(n\log(n))$. This appealing rate has motivated many researchers to develop and use the Sliced Wasserstein distance \citep{Bonnotte2013, kolouri2016sliced, Kolouri2018SlicedWD, kolouri_GSW_2019, liutkus19a}. 

\paragraph{Entropic regularization}
Regularized entropic OT was proposed in \citep{CuturiSinkhorn} as a way to make the problem strictly convex and easier to solve. For the Euclidean distance, it is defined as:
\begin{align}
    & W^{\varepsilon}(\alpha, \beta) = \underset{\pi \in \UU(\alpha, \beta)}{\text{min}} \int\displaylimits_{\mathcal{X}\times\mathcal{Y}} \| \xx - \yy \|_2^p d\pi(\xx, \yy) + \varepsilon H(\pi|\xi), \\
    & \text{ with }
 H(\pi|\xi) = \int_{\mathcal{X}\times\mathcal{Y}} \log\left(\frac{d\pi(\xx, \yy)}{d\alpha(\xx) d\beta(\yy)}\right)d\pi(\xx, \yy), \label{EQ : def entropy}
\end{align}
where $\xi = \alpha \otimes \beta$ and $\varepsilon \geq 0$ is the
regularization coefficient. The power p is typically set to 1 or 2. We call this function, the entropic OT loss. Entropic regularization also makes the problem strongly convex and differentiable with respect to the cost or the input distributions, which is a key optimization property for using gradient-based algorithms. Other regularizations could be added to the original OT problem \citep{Dessein2018} for different purposes such as group-lasso or quadratic regularization  \citep{DACourty, blondel2018}.

It is well known that adding an entropic regularization leads to optimal transport plans that are dense \citep{blondel2018} and can be far from the original OT solutions. This leads to loose the metric property for the entropic OT loss $W^{\varepsilon}$, \emph{i.e., $W^{\varepsilon}(\beta, \beta) \ne 0$}. This motivated \citep{genevay_2018} to introduce an unbiased loss which uses entropic regularization and is called the Sinkhorn divergence. It is defined as:
\begin{equation}
    S^{\varepsilon}(\alpha, \beta) = W^{\varepsilon}(\alpha, \beta) - \frac{1}{2}(W^{\varepsilon}(\alpha, \alpha) + W^{\varepsilon}(\beta, \beta)).
\end{equation}
It can still be computed with the same order of computational complexity as the entropic OT loss and has been proven to be a divergence which interpolates between OT and Maximum Mean Discrepancy distance (MMD), with respect to the regularization coefficient \citep{feydy19a}. MMD are integral probability metrics over a reproducing kernel Hilbert space \citep{MMD_Gretton}. When $\varepsilon$ tends to 0, $S^{\varepsilon}(\alpha, \beta)$ recovers the OT solution and when $\varepsilon$ tends to $\infty$, $S^{\varepsilon}(\alpha, \beta)$ converges to the
MMD solution with a particular kernel. Second, as proved by \citeauthor{feydy19a}, if the cost $c$ is Lipschitz, then $S_{c}^{\varepsilon}$ is a convex, symmetric and smooth divergence.
The sample complexity of the Sinkhorn divergence was proven in \citep{genevay19} to be
$\centering
O\left(\frac{e^{\frac{\kappa}{\varepsilon}}}{\sqrt{n}}\left(1+\frac{1}{\varepsilon^{\lfloor d / 2\rfloor}}\right)\right)
$ where $d$ is the dimension of $\mathcal{X}$. It can be seen as an interpolation of sample complexities from MMD and OT sample complexity depending on $\varepsilon$.  So adding an entropic regularization lowers the dependence of the sample complexity to the dimensionality of the data space.

\paragraph{Gromov-Wasserstein distance}

Classical OT distances cannot be used when a relevant ground cost between the distributions cannot be defined. For instance, when $\alpha$ and $\beta$ are defined in Euclidean spaces of different dimensions. Learned deep learning features fall into this scheme as they can usually be arbitrarily rotated or permuted \citep{bunne19a}. 
A variant of the Wasserstein distance was designed to address this specific issue. The Gromov-Wasserstein (GW) distance \citep{memoli_GW} has been investigated in the past few years and relies on comparing intra-domain distances $c_{\mathcal{X}}$ and $c_{\mathcal{Y}}$. The general setting corresponds to computing couplings between metric measure spaces $(\mathcal{X},c_{\mathcal{X}},\alpha)$ and $(\mathcal{Y},c_{\mathcal{Y}},\beta)$, where ($c_{\mathcal{X}}$, $c_{\mathcal{Y}}$) are distances, while $\alpha$ and $\beta$ are measures on their respective spaces. One defines the Gromov-Wasserstein distance as:
\begin{equation}
\mathcal{GW}_p^p((\alpha, c_{\mathcal{X}}), (\beta, c_{\mathcal{Y}})) =  \underset{\pi \in \UU(\alpha, \beta)}{\text{min}} \int_{\mathcal{X}^2 \times \mathcal{Y}^2} |c_{\mathcal{X}}(x, x') - c_{\mathcal{Y}}(y, y')|^p d\pi(x, x') d\pi(y, y').
\end{equation}
We can interpret the $\mathcal{GW}$ distance as follows: the coupling tends to associate samples that share common relations with the other samples in their respective metric spaces.
Formally, $\mathcal{GW}$ defines a distance between metric measure spaces up to isometries, where one says that $(\mathcal{X},c_{\mathcal{X}},\alpha)$ and $(\mathcal{Y},c_{\mathcal{Y}},\beta)$ are isometric if there exists a bijection $\psi : \mathcal{X} \mapsto \mathcal{Y}$ such that the pushforward operator satisfies $\psi_{\#}\alpha = \beta$ and $c_{\mathcal{Y}}(\psi(x), \psi(x')) = c_{\mathcal{X}}(x, x')$. However, the Gromov-Wasserstein distance is challenging to compute, as a non convex quadratic program which is NP hard \citep{peyre16}. To address this issue from another perspective, one can realign the spaces $\mathcal{X}$ and $\mathcal{Y}$ using a global transformation before using the classical Wasserstein distance \citep{melis19a}. Furthermore, an entropic variant of Gromov-Wasserstein has been proposed to reduce its computational complexity \citep{peyre16} and recently, a sliced variant has been introduced in \citep{vayersgw} in the case of a particular cost. Finally, a tree variant was proposed to accelerate the computation of GW \citep{Le2019}.

\paragraph{Minibatch Wasserstein loss}
While the entropic OT loss has better computational complexity than the original Wasserstein distance, it is still challenging to compute it for a large dataset. To overcome this issue, several papers rely on a minibatch computation. Minibatches have been widely used in stochastic optimization for training ML models. For optimizing OT based criterion, minibatches have been used
for generative adversarial networks, they were associated with the Sinkhorn divergence as a loss in \citep{genevay_2018}, with an energy distance loss in \citep{salimans2018improving}, with the sliced Wasserstein distance variants \citep{Wu_2019_CVPR, liutkus19a, kolouri2016sliced} and a Gromov-Wasserstein loss in \citep{bunne19a}. We can also find this strategy in domain adaptation where the Wasserstein distance is optimized to learn a target joint distribution in \citep{deepjdot}. Instead of computing the OT problem between the full distributions, all those approaches compute an averaged of OT problems between batches of the source and the target domains.
Several works justifying the minibatch paradigm were recently published. \citep{bernton2019parameter} showed that for generative models, the minimizers of the minibatch loss converge to the true minimizer when the minibatch size increases. \citep{mbot_Sommerfeld} considered another approach, where they approximate OT with the minibatch strategy and exhibit a deviation bound between the two quantities. We followed  a different approach in \citep{fatras2019batchwass}, where we studied the behavior of using the minibatch OT losses as a loss function. We also studied the statistical and optimization properties of the minibatch Wasserstein loss functions on restricted cases, \textit{i.e.,} on bounded and uniform measures. We also highlighted the consequences of minibatch on the resulting transport plan and the behavior of such a loss for data fitting problems. 

The purpose of this work is to extend our results to a more general setting. We consider unbounded and non uniform probability distributions and a larger number of OT variants, such as the Gromov-Wasserstein distance. We then introduce a general framework to design minibatch OT in order to have meaningful estimators, including the sampling with replacement case where a given data appears several times in a minibatch. We state basic properties for the estimators. Then, we propose a new loss function to correct a downside of minibatch OT and we study its positiveness. After, we study concentration bounds for bounded and unbounded data scenarios. Regarding the optimisation properties of our loss function, we prove that minibatch OT can be optimised with a stochastic gradient strategy, in particular we considered weaker assumptions and the minibatch Wasserstein distance which was missing in our previous work. And finally we empirically demonstrate the reviewed properties and the performance of minibatch OT on applications such as generative modelling, gradient flows, map learning tasks, color transfer and meshes comparison.

\section{Minibatch Wasserstein }\label{sec:MBW}

The purpose of this section is to formally define and design the integration of the minibatch strategy with optimal transport. We start with a motivating example illustrating the different challenges that the minibatch strategy implies. Then we formalise the definitions of minibatch OT losses, after we present the basic properties, strengths and weaknesses of our minibatch OT losses. Then we introduce a new loss function which aims at correcting our minibatch OT losses.

\subsection{Motivating example : Generative Adversarial Networks}
In this subsection, we investigate an application where optimal transport has become a key tool. Generative adversarial networks have become a natural method to generate high quality images \citep{goodfellow_gan, Ledig2017}. The goal is to learn a generator $g_{\theta}$ of data from a random distribution, which lie in a latent space $\mathcal{Z}$, and to make generated data look like real data, which lie in a space $\mathcal{X}$. In this context, real data are empirical samples of an unknown distribution of interest, and as such form a discrete probability distribution $\balpha$, while the transformation from the latent space by the action of $g_\theta$ produces a continuous (possibly with density) distribution $\boldsymbol{\zeta}$. The generator is trained to minimize the distance between the real data and the generated distributions.

When the examples sampled by the generator have low variety, vanilla GANs suffer from gradient vanishing and mode collapse. To address these problems, \citep{arjovsky_2017} proposed to use the Wasserstein distance instead of KL-divergence for training GANs. As the target distribution is continuous, it corresponds for a finite dataset to a semi-continuous OT problem :
\begin{equation}
    W_1(\balpha, g_{\theta_\#}\boldsymbol\zeta)  = \min _{\gamma \in \mathcal{M}(\mathcal{X} \times \mathcal{Z})}\left\{\int_{\mathcal{X} \times \mathcal{Z}} c\left(\xx, g_{\theta}(\zz)\right) \mathrm{d} \gamma(\xx, \zz): \gamma \in \mathcal{U}(\balpha, g_{\theta_\#}\boldsymbol\zeta)\right\},
\end{equation}
where $g_{\theta_\#}$ is the generator pushforward operator. 
For the euclidean distance, the 1-Wasserstein distance can be rewritten with the Kantorovich-Rubinstein duality \citep{San15a} as follows:
\begin{equation}
    W_1(\balpha, g_{\theta_\#}\boldsymbol\zeta) = \underset{\operatorname{Lip}(f)\leq 1}{\operatorname{sup}} \Big\vert \frac1n \sum_{i=1}^n [f(\xx_i)] - \expect_{\zz \sim \boldsymbol{\boldsymbol{\zeta}}}[f(g_{\theta}(\zz))] \Big\vert .
\end{equation}
In practice, the Kantorovich potential $f$ is approximated with a neural network and optimized alternatively with the generator. However, this formulation requires the dual potential to be one Lipschitz, \emph{i.e., $\operatorname{Lip}(f) \leq 1$}. To enforce numerically this constraint, \cite{arjovsky_2017} manually constrained the neural network's weights to be less or equal to 0.01 in absolute value, while \cite{Gulrajani2017} added a gradient penalty in practice. Hence, those strategies do not calculate the true Wasserstein distance but an approximation. 

In their work, \cite{genevay_2018} relied on a minibatch computation of optimal transport to compute the primal formulation. This appealing strategy makes the problem tractable and it has been implemented as follows. After drawing $m$ samples from the latent domain and generating $m$ data, they pick $m$ training samples. Then they compute the Sinkhorn divergence between the training and generated samples. While we get a correct estimation of the Sinkhorn divergence between minibatches, it does not correspond to the true Sinkhorn divergence between measures. Finally the optimization problem is as follows:
\begin{equation}
    \underset{\theta}{\operatorname{min} }  \expect_{\widehat{\alpha}, \widehat{\zeta}} W_p(\widehat{\alpha}, g_{\theta_\#} \widehat{\zeta}),
\end{equation}

where $\widehat{\alpha}, \widehat{\zeta}$ represent respectively minibatches measures of source and target distributions.
This paper aims at bringing some light to this efficient strategy. 


\begin{figure}
    \centering
    \includegraphics[scale=0.5]{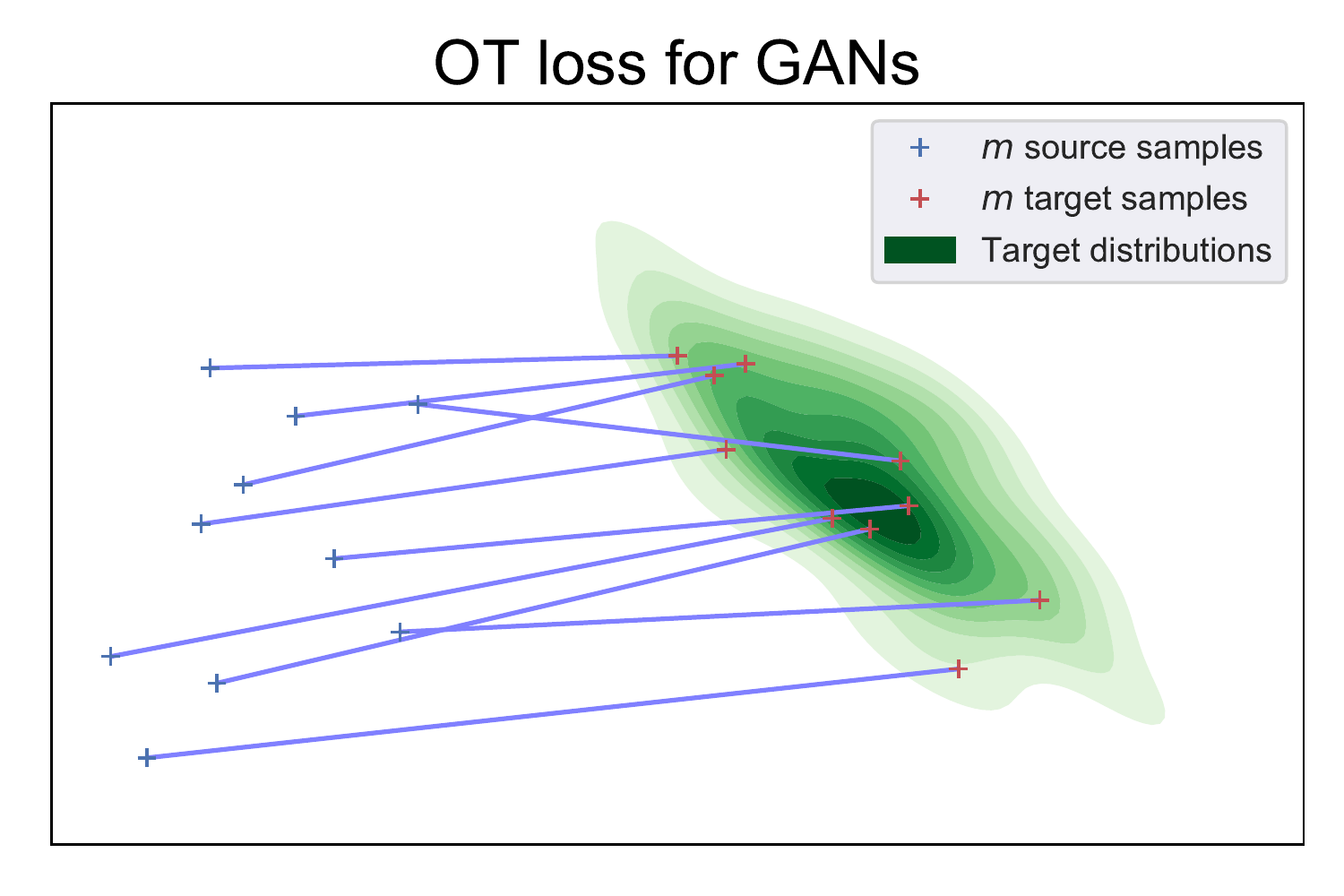}
    \caption{Illustration of optimal transport computation for GANs.}
    \label{fig:gan_ot}
\end{figure}

\subsection{Notations and Definitions}
\subsubsection{Notations}
In order to formalize the design of minibatches, we start by describing our notations. Vectors are denoted in roman boldfont, and data tuples (tuples of vectors) are denoted in capital boldfont. Suppose we have access to $n$ data $\xx \in \mathbb{R}^d$. We first assign a fixed index to each data and then get a $n$-tuple of data $\XX$, i.e., $\XX = (\xx_1, \xx_2, \cdots, \xx_n)$. This assignation allows us to draw minibatches of data, moreover, permutations of assigned labels would not change any result. 

As each data inside the data m-tuple has an index, it is then possible to characterize a m-tuple of data with a corresponding m-tuple of indices. A generic element of indices $I = (i_1,\ldots,i_m) \in \llbracket n \rrbracket^m$ is called an index $m$-tuple. 
For an index $m$-tuple $I = (i_1,\ldots, i_m)$,  $\XX(I) = (\xx_{i_1},\ldots,\xx_{i_m})$ is the corresponding data $m$-tuple, and vice-versa any data $m$-tuple can be written $\XX(I)$ for some index $m$-tuple $I$. After designing minibatches, we define the inputs of our problems.

Consider $ \alpha $ (resp. $ \beta $) a probability distribution on the source (resp. target) domain. In the case of discrete distributions, the distribution can be written as a sum of diracs, i.e., $\alpha = \an = \sum_{i=1}^n a_i \delta_{\xx_i}$, with a probability vector $\aa=(a_1, \cdots, a_n) \in \simplex_n$. We denote the product probability distribution $\alpha^{\otimes m }$ on $\mathcal{X}^{\otimes m}$ of $m$ i.i.d. random variables following $\alpha$. 
In a learning scenario, $\alpha$ and $\beta$ are unknown and instead, we have access to  $\XX=(\xx_1, \cdots, \xx_n)
$ (resp. $\boldsymbol{Y}=(\yy_1, \cdots, \yy_n) $), which corresponds to $n$ \emph{i.i.d.} random variables drawn from $\alpha$ (resp. from $\beta$), i.e. $\XX$ is drawn from $\alpha^{\otimes n}$ and $\boldsymbol{Y}$ is drawn from $\beta^{\otimes n}$. In our experiments, we associate to these (random) samples two uniform probability vectors, denoted
$\uu \in \S_n, (u_i)_{i\in \llbracket n \rrbracket} = \frac{1}{n}$, but the minibatch procedure can be defined for general probability vectors. We finish by defining extra notations. We consider the mapping: 
\begin{align}
C^{m,p}: (\XX,\YY) \in (\R^{m \times d})^2 \mapsto C^{m,p}(\XX, \YY) = \big( \| \xx_i - \yy_j \|_{2}^p   \big)_{1 \leq i,j \leq m} \in \mathcal{M}_m(\R),
\label{DEF : mb_matrix_map}
\end{align}
where $\mathcal{M}_n(\R)$ is the set of (real) square matrices of size $n$. 
The characteristic function of the set $A$, which is equal to 1 if $i \in A$ and 0 otherwise, is denoted $\mathbf{1}_{A}(i)$. With a slight abuse of notation we write $i \in I$ if the index $i$ appears in the $m$-tuple $I$. We also write $\mathbf{1}_I(i)$ for tuples of indices. Regarding the sum over the elements of $I = (i_1, \cdots, i_m)$, we denote it as $\sum_{i \in I} f(i)= \sum_{k=1}^{m}f(i_k)$ and similarly for the product over the elements $\Pi_{i \in I} f(i) = \Pi_{k=1}^m f(i_k)$. All the notations described above are summarized with simple examples in Table \ref{tab:summary_def_tab} and a longer version can be found in appendix \ref{app_sec:notations}.

\begin{table}[t]
        \begin{center}
            \begin{tabular}{ |c|c|c| } 
                 \hline
                 Notations & Description & Examples \\
                 \hline 
                 $\xx$ & vector $\in \mathbb{R}^d$ & $\xx=[1, 2, 3]$\\
                 $n$ & number of data & $n$=6\\
                 $m$ & minibatch size & $m=4 \leq n$\\
                 $I$ & Index $m$-tuple & $(1, 1, 2, 1)$ \\ 
                 $\llbracket n \rrbracket^m$ & Set of all index $m$-tuples & $\{ I_1, I_2 \cdots, \}$\\ 
                 $\XX(I)$ & data $m$-tuple & $(\xx_1, \xx_1, \xx_2, \xx_1)$ \\ 
                 $\XX$ & data $n$-tuple & $(\xx_1, \cdots, \xx_n)$ \\ 
                $\aa \in \simplex_n$ & probability vector & $\sum_{i=1}^n a_i = 1$ \\
                $\uu \in \simplex_n$ & uniform probability vector & $\sum_{i=1}^n \frac1n = 1$ \\
                $\alpha$ & probability distribution & $\mathcal{N}(0,1)$\\
                $\alpha^{\otimes m }$ & $m$-tuples drawn from $\alpha$ & $\XX \sim \alpha^{\otimes n}$\\
                $w$ & Reweighting function & $w([\frac{1}{2},\frac{1}{6},\frac{1}{3}], (1, 2)) = [\frac{2}{3}, \frac{1}{3}]$\\
                $P$ & Probability law to draw 
                index $m$-tuples & $P(I) = n^{-m}$\\ 
                $h$ & OT kernel & $h = W_p, W_\varepsilon, S_\varepsilon, \GW$\\
                $C^{m,p}$ & Ground cost matrix of size $n$ and $m$ & euclidean distance\\
                $\bar{h}_{w, P}$ & Minibatch kernel OT loss & $\overline{W}_{p,w,P}$\\
                $\widetilde{h}_{w,P}^k$ & Incomplete MBOT loss & $\widetilde{W}_{w^\mathtt{U},P^\mathtt{U}}^k$\\
                $\overline{\Pi}^h_{w,P}$ & MBOT plan & $\overline{\Pi}^h_{W,P}$ \\
                $\widetilde{\Pi}^{h,k}_{w,P}$ & Incomplete MBOT plan & $\widetilde{\Pi}^{W_{\varepsilon},k}_{w,P}$\\
                $\Lambda_{h, w, P}$ & Debiased minibatch loss & $\Lambda_{h=W_2^2}$\\
                $\widetilde{\Lambda}_{h,w,P, C^{n,p}(\XX, \YY)}^k$ & Incomplete debiased MBOT loss & $\widetilde{\Lambda}_{W,w,P}^k$\\
                 \hline
            \end{tabular}
        \end{center}
    \caption{Summary table of defined notations. (Left) notations, (middle) descriptions, (right) examples. }
    \label{tab:summary_def_tab}
\end{table}


\subsubsection{Minibatch Wasserstein definitions}\label{sec:definitions}
\newcommand\bJJ{\mathbb{J}}

To begin with, we define a generic mechanism based on minibatches to define a notion of "distance" between empirical measures. For this, we consider optimal transport losses. OT kernels were defined for continuous probability distributions. We suppose we have fixed data from now on and we consider that OT kernels take probability vectors as inputs instead of probability distributions.

\begin{definition}[OT Kernels]\label{DEF : OT_ker} An OT kernel is a function $h \in \{ W_p, W^{\varepsilon}, S^{\varepsilon},  \mathcal{GW} \}$. If $h \in \{ W_p, W^{\varepsilon}, S^{\varepsilon} \}$ it is a function of the form 
\begin{equation}
h : (\aa,\bb,C) \in (\S_m)^2 \times \mathcal{M}_m(\R)  \mapsto  \R_+
\label{EQ : ker}
\end{equation}
If $h = \mathcal{GW}$ it has the form
\begin{equation}
    \mathcal{GW} : (\aa,\bb,C_1, C_2) \in (\S_m)^2 \times \mathcal{M}_m(\R)^2 \mapsto  \R_+
\label{EQ : GW_ker}
\end{equation}
\end{definition}

Discrete probability distributions can either be represented as a sum of diracs or with a probability vector and the support of measures. We chose the latter as it is easier to define formal mathematical objects when we consider discrete probability distributions and for consistency with \citep{COT_Peyre}. {Indeed sum of diracs are equal for different indices assignations, \emph{i.e., $\sum_{k=1}^n a_{i_k} \delta_{x_{i_k}} = \sum_{k=1}^n a_{k} \delta_{x_k}$}, the result of selecting an element with a given index from the minibatch would depend on the order of diracs. One can define a discrete probability distribution from a probability vector $\aa \in \simplex_n$ and locations $\XX$ in a canonical way by $\a_n = \sum_{i=1}^n a_i \delta_{x_i} $ (see remark 2.1 \cite{COT_Peyre}), and we will often implicitly use this assignment throughout the rest of the article (see remark 2.11 \cite{COT_Peyre}).}

{To define minibatch OT losses, a first ingredient is a "reweighting function" 
$w$
that takes as inputs a discrete probability $\mathbf{a} \in \simplex_n$ and an $m$-tuple of indices $I = (i_1,\ldots,i_m)$ and outputs a 
discrete probability vector
$\mathbf{b} = w(\mathbf{a},I) \in \simplex_m$. 
A second ingredient is  a parametric family of distributions $\{P_{\mathbf{a}} : \mathbf{a} \in \Sigma_n\}$ such that for each $\mathbf{a} \in \Sigma_n$, $P_{\mathbf{a}}$ is a probability distribution over $m$-tuples $I$ of indices. {The law on probability tuples assures that we have a weighted average of OT kernels and its combination with a suited reweighting function assures all samples are transported.} Those ingredients are needed to get unbiased estimator of minibatch OT. Formally, we need:

{

\begin{definition}[Reweighting and probability functions] A reweighting function is a map $w$ of the form :
\begin{equation}
w : (\aa, I) \in \S_n \times \llbracket n \rrbracket^m \mapsto \S_m
\label{EQ : rew_fun}
\end{equation}
A probability function is a map $P$ of the form :
\begin{equation}
P : \aa \in \S_n \mapsto P_{\aa} \in \mathcal{P} \big( \llbracket n \rrbracket^m \big),
\label{EQ : prob_fun}
\end{equation}
where $\mathcal{P} \big( \llbracket n \rrbracket^m \big)$ is the set of probability distributions over the set of $m$-tuples $I$ of indices $\llbracket n \rrbracket^m$.
\end{definition}

We are now ready to give a formal definition of minibatch OT losses. The idea is to compute the expectation of the OT kernels over minibatches $I$, furthermore we need the reweighting functions to assure that the OT kernels has probability vectors as inputs.

\begin{definition}[Minibatch Wasserstein] \label{def:mbw}

Let $C = C^{n,p}(\XX, \YY)$ be a matrix of size $n \times n$. Given a kernel $h \in \{ W_p, W_p^p, W^{\varepsilon}, S^{\varepsilon}\}$ as in \eqref{EQ : ker}, two reweighting functions $w_1, w_2 $ and two probability functions $P^1,P^2$ as in \eqref{EQ : rew_fun} and \eqref{EQ : prob_fun} respectively, we define the minibatch OT loss $\overline{h}_{w_1, w_2, P^1, P^2}$ for any $\aa, \bb \in \simplex_n$  by :
\begin{equation}
 \overline{h}_{w_1, w_2, P^1, P^2, C}(\mathbf{a},\mathbf{b})
 := \expect_{I \sim P^1_{\mathbf{a}}} \expect_{J \sim P^2_{\mathbf{b}}} 
 h \Big( w_1(\mathbf{a},I), w_2(\mathbf{b},J), C_{(I,J)}\Big),
 \label{def:minibatch_wasserstein}
\end{equation}
where for $I,J$ two $m$-tuples, $C_{(I,J)}$ is the matrix extracted from $C$ by keeping the rows and columns corresponding to $I$ and $J$ respectively.
Moreover, we also define for two ground costs $C^1 = C^{n,p}(\XX, \XX)$ and $C^2 = C^{n,p}(\YY, \YY)$ the loss:
\begin{equation}
    \overline{\GW}_{w_1,w_2, P^1, P^2, C^1, C^2} (\aa, \bb ) := \expect_{I \sim P^1_{\mathbf{a}}} \expect_{J \sim P^2_{\mathbf{b}}} 
 \GW \Big( w_1(\mathbf{a},I), w_2(\mathbf{b},J), C^1_{I,I}, C^2_{J,J}\Big),
\label{def:minibatch_GW}
\end{equation}
where $C^1_{(I,I)}$ (resp. $C^2_{(J,J)}$) is the matrix extracted from $C^1$ (resp. $C^2$) by keeping the rows and columns corresponding to $I$ and $I$ (resp. $J$ and $J$).\\
\end{definition}}
While it is easier to get statistical results with the ground cost $C^{n,p}(\XX, \YY)$, which is a square matrix of size $n$, in practice we only need to compute $C_{(I,J)}$ as it is equal to $C^{m,p}(\XX(I), \YY(J))$. In what follows, the dependence of minibatch OT in the ground cost $C$ will often be omitted when there is no possible confusion. When the reweighting functions and the probability laws on tuples are the same (equal to $w$ and $P$ respectively), we use the following shorthand notations for the abover losses : $\overline{h}_{w, P}$ 
With a slight abuse of notation, we also use the notation $\overline{h}_{w, P}$ for the $\GW$ loss.
\begin{remark} In \eqref{def:minibatch_wasserstein} and \eqref{def:minibatch_GW} the dependence in the distribution supports, $\XX$ and $\YY$, is implicit through the euclidean ground cost $C$.
\end{remark}

The loss $\overline{h}(\mathbf{a}, \mathbf{b})$ corresponds to an averaged optimal transport distance between sub-probability distributions of input probability distributions $\mathbf{a}$ and $\mathbf{b}$. The minibatch OT losses define weighted U-statistics and V-statistics \citep{LeeUstats} where the weights depend on the input probability vectors $\mathbf{a}, \mathbf{b}$ and on the laws over index m-tuple $P_{\mathbf{a}}, P_{\mathbf{b}}$. This connection turns out to be central to get quantitative statistical results. Concrete versions of these minibatch OT losses are obtained by specifying its ingredients $h$, $w$, and $P$. We now give a few examples of some \textit{reweighting functions} and \textit{families of distributions}.
\begin{example}[Uniform reweighting function]\label{def:reweight_function_U}
The \emph{uniform reweighting function} $w^\mathtt{U}$ is independent of the input discrete probability $\mathbf{a}$. It is defined coordinatewise for any $m$-tuple $I = (i_1,\ldots,i_m)$ by $w_k^\mathtt{U}(\mathbf{a},I) = \frac{1}{m}$, $1 \leq k \leq m$ and yields to a uniform probability vector in $\simplex_m$.
\end{example}

\begin{example}[Normalized reweighting function]\label{def:reweight_function_N}
The \emph{normalized reweighting function} $w^\mathtt{W}$ normalizes the restriction of the input discrete probability $\mathbf{a}$ to the support of $I$, to ensure it remains a discrete probability. It is defined coordinatewise for any $m$-tuple $I = (i_1,\ldots,i_m)$ by $w_k^\mathtt{W}(\mathbf{a},I) = \frac{a_{i_k}}{\sum_{p=1}^m a_{i_p}}$, $1 \leq k \leq m$, 
which is again a probability vector even if entries in $I$ are repeated.
When $I$ is such that $\sum_{i \in I} a_i=0$, we define $w_k^\mathtt{W}(\mathbf{a},I) = w_k^\mathtt{U}(\mathbf{a},I) = \frac{1}{m}$.

\end{example}

For instance, consider four $\XX=(\xx_1, \xx_2, \xx_3, \xx_4)$ data with weights $\aa = (\frac14, \frac18, \frac18, \frac12 )$, if one picks the batch $I=(2, 4)$, the reweighting functions give $w^\mathtt{W}(\mathbf{a},I) = [\frac15, \frac45]$ and $w^\mathtt{U}(\mathbf{a},I) = [\frac12, \frac12]$. In the case of a uniform discrete probability $\uu \in \simplex_n$, the two reweighting functions are identical. 

Regarding the parametric law on indices, which gives the probability to pick a given batch of samples, we focus on two constructions depending whether sampling is done with or without replacement. Indeed in practice, most of work use a sampling without replacement, and it is easy to design this case with our formalism. We first consider sampling with replacement.

\begin{example}[Drawing indices with replacement]
\label{def:law_indices_rep} 
Drawing $i_\ell \in \llbracket n \rrbracket$, $1 \leq \ell \leq m$ i.i.d. (with replacement) from the discrete probability distribution $\mathbf{a} \in \simplex_n$ yields the law on indices 
\begin{equation}
P_\mathbf{a}^{\mathtt{U}}(I) =\Pi_{i \in I} a_i.
\label{EQ : samp_no_rep}
\end{equation}
\end{example}

Now we give an example of drawing without replacement. The idea is to give a zero probability to pick a batch with repeated indices.

\begin{example}[Drawing indices ``without replacement'']\label{def:law_indices_rep2} 
Given a discrete probability distribution $\mathbf{a} \in \simplex_n$, it is also possible to draw distinct indices $i_\ell \in \llbracket n \rrbracket$, $1 \leq \ell \leq m$, by defining $P_\mathbf{a}^\mathtt{W}(I)=0$ if the $m$-tuple $I$ \emph{has} repeated indices, otherwise
\begin{equation}
P_{\mathbf{a}}^\mathtt{W}(I) = 
\frac{1}{m}
\frac{(n-m)!}{(n-1)!}
\sum_{i \in I} a_i.
\label{EQ : samp_rep}
\end{equation}

With a uniform discrete probability, $a_i= \frac1n$, $1 \leq i \leq n$, this law corresponds to drawing the $m$-tuples without repeated indices $I$ uniformly at random among all possible $m$-tuples without repeated indices, i.e., drawing the $m$ indices $i_p$ without replacement. By abuse of language, we will sometime refer to this law as a draw "without replacement" even for non uniform $\mathbf{a}$.
\end{example}





This formalism is a generalization of minibatch OT losses previously defined in \citep{fatras2019batchwass}. Indeed, for a sampling without replacement, associated to a uniform probability distribution $\uu \in \simplex_n$ and reweighting function $w^\mathtt{U}$, we have:

\begin{proposition}[Minibatch OT loss \citep{fatras2019batchwass}] 
\label{def:law_indices_rep_aistats} 
Denote $\mathcal{P}^{o,m}$ the set of all ordered $m$-tuples without repeated indices. Given a discrete uniform probability distribution $\uu$, the reweighting function $w^\mathtt{U}$ and the probability law on m-tuples $P_{\uu}^\mathtt{W}$, we have that our minibatch OT losses is equal to the minibatch OT losses previously defined in \citep{fatras2019batchwass}. 
Formally,

\begin{equation}
\overline{h}_{w^\mathtt{U},P^\mathtt{W}}(\uu,\uu) ={n \choose m}^{-2}  \sum_{I^o  \in \mathcal{P}^{o,m} } \sum_{J^o \in \mathcal{P}^{o,m}}   h \Big( w_1(\uu,A), w_2(\uu,B), C_{(I,J)}\Big),
\end{equation}

where $C_{(I^o,J^o)}$ is the ground cost matrix between elements in $I^o$ and $J^o$.
\end{proposition}

\begin{proof}
We prove it for the Wasserstein distance losses and the proof for the $\GW$ loss follows the same steps. In the case of a uniform distribution $\uu$, we have $P_{\uu}^\mathtt{W}(I) = \frac{(n-m)!}{n!} $.
\newline 
We denote $\Pim$ the set of all $m$-tuples without repeated indices and we define $\mathcal{P}^{o,m}$ the set of all ordered $m$-tuples without repeated indices. For each element $I^o$ in $\mathcal{P}^{o,m}$, there are $m!$ permutations of $m$-tuples without replacement $I$. For each $I$, let us denote its corresponding element in $\mathcal{P}^{o,m}$ as $E(I)$. Denote  $I$ (resp. $J$) the $m$-tuples such as $E(I) = I^o$ (resp. $E(J) = J^o$). We can then show that our estimator is equal to the one defined in \citep{fatras2019batchwass}. Let us gather the permutations of $I^o$ and $J^o$ as: 
\begin{equation*}
    h \Big( w^\mathtt{U}(\uu,I^o), w^\mathtt{U}(\uu,J^o), C_{(I^o,J^o)}\Big) = [m!]^{-2} \sum_{I: E(I)=I^o} \sum_{J: E(J)=J^o}  h \Big( w^\mathtt{U}(\uu,I), w^\mathtt{U}(\uu,J), C_{(I,J)}\Big)\text{, }
\end{equation*}
then:
\begin{align*}
&{n \choose m}^{-2}  \sum_{I^o \in \mathcal{P}^{o,m}} \sum_{J^o \in  \mathcal{P}^{o,m}}  h \Big(  w^\mathtt{U}(\uu,I^o), w^\mathtt{U}(\uu,J^o), C_{(I^o,J^o)}\Big) \\
&= {n \choose m}^{-2}  \sum_{I^o \in \mathcal{P}^{o,m}} \sum_{J^o \in \mathcal{P}^{o,m}} [m!]^{-2}  \sum_{I: E(I)=I^o} \sum_{J: E(J) = J^o}  h \Big( w^\mathtt{U}(\uu,I), w^\mathtt{U}(\uu,J), C_{(I,J)}\Big)\\
&= (\frac{(n-m)!}{n!})^{2} \sum_{I \in \Pim} \sum_{J \in \Pim}  h \Big( w^\mathtt{U}(\uu,I), w^\mathtt{U}(\uu,J), C_{(I,J)}\Big) =  \overline{h}_{w^\mathtt{U},P_{\uu}^\mathtt{W}}(\uu,\uu)
\end{align*}
\end{proof}

The general proposed formalism allows one to recover the GAN formalism. We can define a sampling without replacement in the source domain with $P_{\mathbf{a}}^\mathtt{W}$ and a sampling with replacement in the target domain with $P_{\mathbf{b}}^\mathtt{U}$, to get the loss $\overline{h}_{w^\mathtt{W}, w^\mathtt{U}, P_{\mathbf{a}}^\mathtt{W}, P_{\mathbf{b}}^\mathtt{U}}$. After setting a rigorous formalism for minibatch optimal transport, we study its transport plan counter part.


 


\subsubsection{Minibatch transport plan}

Classical OT losses such as the $p$-Wasserstein distance $W_p$ or its entropic variant $W^\varepsilon$ are directly associated with a transport plan between distributions. We now propose to similarly define a transport plan associated to the proposed minibatch losses. The main idea is that for each pair of samples $\xx_i$ and $\yy_j$, one can average the connections provided by all possible ``minibatch transport plans'', with the following definition.

\begin{definition}[minibatch transport plan] \label{def:MBTP}
We will denote by $\Pi(h,C, \aa, \bb)$ the set of all optimal transport plans for a given OT kernel $h$, cost matrix $C$ and marginals $\aa, \bb$. Let $C = C^{n,p}(\XX, \YY)$ be a matrix of size $n \times n$ and let $\mathbf{a}_{I}, \mathbf{b}_{I} \in \simplex_m$ be discrete probability vectors indexed by $m$-tuples $I$ of $\llbracket 1,n\rrbracket$.
For each pair of index $m$-tuples $I = (i_1,\ldots,i_m)$ and $J = (j_1,\ldots,j_m)$ from $\llbracket 1,n\rrbracket^m$, consider $C' := C_{I, J}$ the $m \times m$ matrix with entries $C'_{k\ell} = C_{i_k,j_\ell}$ (repeated entries in $I$ or $J$ imply repeated lines or columns) and denote by
$\Pi^{m} _{ I, J }$ an arbitrary element of $\Pi(h, C_{I,J}, \mathbf{a}_{I}, \mathbf{b}_{J})$. This optimal transport plan is an $m \times m$ matrix satisfying $\Pi^{m}_{I,J} \in U( \mathbf{a}_{I},\mathbf{b}_{J} )$, that is to say
\begin{equation}\label{eq:minibatchTPadmissible}
\Pi^{m}_{I,J} \mathbf{1}_m = \mathbf{a}_I\quad \text{and}\quad
\mathbf{1}_m^\top \Pi^{m}_{I,J} = \mathbf{b}_J^\top.
\end{equation}
It can be lifted to an $n \times n$ matrix where all entries are zero except those indexed in $I \times J$: 
\begin{align}
    \label{eq:minibatchTPlifted}
    \Pi_{I,J} &= Q_I^\top \Pi^{m}_{I,J} Q_J\\
\intertext{where $Q_I$ and $Q_J$ are $m \times n$ matrices defined entrywise as}
(Q_I)_{k i} &= \delta_{i_k,i}, 1 \leq k \leq m, 1 \leq i \leq n\\
(Q_J)_{\ell j} &= \delta_{j_\ell,j}, 1 \leq \ell \leq m, 1 \leq j \leq n.
\end{align}
Each row of these matrices is a Dirac vector, hence they satisfy $Q_I \mathbf{1}_n = \mathbf{1}_m$ and $Q_J \mathbf{1}_n = \mathbf{1}_m$.


\end{definition}
We also define the \textit{averaged minibatch transport matrix} which takes into account all possible minibatch couples.

\begin{definition}[Averaged minibatch transport matrix]
\label{def:AVG_OT_plan}

Consider as in Definition~\ref{def:mbw} an OT kernel $h$, two reweighting functions $w_1, w_2$ and a family of probability distributions  $\{ P_{\mathbf{a}} : \mathbf{a} \in \Sigma_n \}$ over index $m$-tuples from $\llbracket 1,n\rrbracket$, where $1 \leq m \leq n$.
Given discrete probabilities $\mathbf{a},\mathbf{b} \in \simplex_n$ and data tuples $\XX,\YY$, 
 consider for each pair of $m$-tuples $I$, $J$ the discrete probabilities $\mathbf{a}_I = w_1(\mathbf{a},I) \in \simplex_m$, $\mathbf{b}_J = w_2(\mathbf{b},J) \in \simplex_m$, and let $\Pi_{I,J}$ be defined as in Definition~\ref{def:MBTP}.
The averaged minibatch transport matrix is
  \begin{equation}
    \overline{\Pi}^h_{w_1, w_2, P_\aa, P_\bb}(\mathbf{a},\mathbf{b}) \defas \expect_{I \sim P_\mathbf{a}, J \sim P_\mathbf{b}} \Pi_{I, J}
  \label{def:omega_pi_m}
  \end{equation}
  For brevity this is simply denoted $\overline{\Pi}_{w,P}^h$ when $w$ and $P$ are  clear from context.
\end{definition}
The average in the above definition can be expressed as a finite weighted sum of $\Pi_{I,J}$. It is therefore well defined for an arbitrary choice of optimal transport plans $\Pi_{I,J}$, and we do not need to concern ourselves with the measurability of selection of optimal transport plans. The same will be true whenever an average of optimal transport plans will be taken in the rest of this paper, since all results concerning such averages will be nonasymptotic. We will therefore avoid further mentioning this issue, for the sake of brevity.

Note that the Sinkhorn divergence involves three terms, hence three transport plans, which explains why we do not attempt to define an associated averaged minibatch transport matrix. While the $n \times n$ matrix defined in~\eqref{def:omega_pi_m} is candidate to be transport plan between $\mathbf{a}$ and $\mathbf{b}$, we need to check if it is indeed admissible, i.e., if it has the right marginals. This is why it is a priori only called an averaged minibatch transport {\em matrix}.


\begin{proposition}{\label{prop:admi_plan}\label{prop:upper_bound}}
If the reweighting function $w$ and the parametric distribution on $m$-tuples $P_\mathbf{c}$ satisfy the following admissibility condition 
   \begin{equation}
        \expect_{I \sim P_{\mathbf{c}}}  Q_I^\top w(\mathbf{c},I) = \mathbf{c},\qquad \forall \mathbf{c} \in \simplex_n\label{eq:wPadmissible}
    \end{equation}
Then with the notations of Definition~\ref{def:AVG_OT_plan}, the averaged minibatch transport matrix $\overline{\Pi}^h_{w,P}$ is an admissible transport plan between the discrete probabilities $\mathbf{a},\mathbf{b} \in \simplex_n$ in the sense that $\overline{\Pi}^h_{w,P} \mathbf{1}_n = \mathbf{a}$ and $\mathbf{1}_n^\top \overline{\Pi}^h_{w,P} = \mathbf{b}^\top$.
 Considering the Wasserstein kernel $h = W_p^p$, the minibatch loss defined in~\eqref{def:minibatch_wasserstein}, as the associated coupling $\overline{\Pi}^h_{w,P}$ is not the optimal coupling of the full OT problem, it satisfies
\begin{equation}\label{eq:mblowerbound}
    \overline{h}_{w,P}(\mathbf{a}, \mathbf{b}) = \langle \overline{\Pi}^{h}_{w,P}, C \rangle_F \geq h(\mathbf{a}, \mathbf{b}).
\end{equation}
\end{proposition}

Under assumption~\eqref{eq:wPadmissible} one can safely call $\overline{\Pi}^h_{w,P}(\mathbf{a}, \mathbf{b})$ an averaged minibatch transport \emph{plan}.

Our main examples of reweighting functions and parametric probability distributions indeed satisfy the admissibility condition~\eqref{eq:wPadmissible}.
\begin{lemma}[Admissibility]\label{lem:admissibilitywPusual}
The uniform reweighting function $w^\mathtt{U}$ and the parametric law "with replacement" $P^\mathtt{U}$ satisfy the admissibility condition~\eqref{eq:wPadmissible}.
The admissibility condition also holds for the parametric law without replacement $P^\mathtt{W}$ with the normalized reweighting function $w^\mathtt{W}$.\\
In contrast for $w^\mathtt{U},P^\mathtt{W}$ when $\mathbf{a}$ is not uniform, the resulting OT matrix is not a transportation plan.
\end{lemma}

}

\subsubsection{Minibatch subsampling}
In practical settings, since $\overline{h}(\aa,\bb)$ is an expectation over the combinatorial number of all possible pairs of $m$-tuples $I,J$ according to the considered parametric probability law, it is often estimated by drawing only $k$ such pairs of $m$-tuples according to $P$, called  subsample quantity.

\begin{definition}[Minibatch subsampling]\label{def:sub_estimator} Consider the notations from Definition~\ref{def:mbw}. Pick two integers $ k > 0$ and $0 < m \geq n$. Then, we define the incomplete estimator:


\begin{equation}
    \widetilde{h}_{w,P}^k(\mathbf{a}, \mathbf{b}) := \frac{1}{k} \sum_{  (I, J)  \in \mathbb{D}_k  }  h \Big( w(\mathbf{a},I), w(\mathbf{b},J), C_{(I,J)}\Big)
\end{equation}
where $\mathbb{D}_k$ is a set of $k$ pairs of $m$-tuples drawn independently from the joint distribution $P_{\mathbf{a}} \otimes P_{\mathbf{b}}$.
\end{definition}

Incomplete estimators have been widely studied in the U-statistics literature. They can be seen as weighted estimators where the weighted coefficient is equal to 1 if the batch couple has been picked or 0 otherwise. Their variance is always higher than the complete U-statistic (see Theorem 1, section 4.3, \citep{LeeUstats}). It is clear that the incomplete estimator $\widetilde{h}_{w,P}^k(\mathbf{a}, \mathbf{b})$ is closely related to $\overline{h}_{w,P}(\mathbf{a},\mathbf{b})$, it differs with a lack or extra minibatches optimal transport terms. A similar construction holds for minibatch transport plan estimators:

\begin{definition}[Incomplete minibatch transport plan]\label{def:incomplete_avg_OT_plan}
We consider the same definition as above in Definition~\ref{def:sub_estimator} and we define incomplete transport plan estimator. Let two integers $k \geq 1$ and $m \leq n$:



\begin{equation}
    \widetilde{\Pi}^{h,k}_{w,P}(\mathbf{a},\mathbf{b})   := \frac{1}{k} \sum_{  (I, J)  \in \mathbb{D}_k  } \Pi_{I, J},
\end{equation}
where $\Pi_{I, J}$ is the lifted $n \times n$ OT plan between minibatches.
\end{definition}

The next reformulation of the above definitions is useful to prove deviation bounds between the complete and the incomplete estimators. See Lemma \ref{app:inc_U_to_U} and Theorem \ref{thm:dist_marg} below.
\begin{remark}\label{inc_ber} Let $n$, $m \geq n$ and $k$ be positive integers. Let $\big( \mathfrak{b}^{\mathbf{a},\mathbf{b}}_{\ell}(I,J)) \big)_{I,J \in \llbracket n \rrbracket^m, 1 \leq \ell \leq k}$ be a sequence of mutually independent Bernoulli variables of parameter $P_{\mathbf{a}}(I)  P_{\mathbf{b}}(J)$ such that 

\noindent
\begin{align*}
 \mathfrak{b}^{P_\mathbf{a},P_\mathbf{b}}_{\ell}(I,J) =  \begin{cases} 1 \text{ if $(I,J)$ has been selected in the $\ell$-th draw} \\ 0 \text{ otherwise}.  \end{cases}
\end{align*}

\noindent
We then can write 

\noindent
\begin{align*}
     \widetilde{h}_{w,P}^k(\mathbf{a}, \mathbf{b}) & = \frac1k \sum_{\ell = 1}^k \sum_{I,J  \in \llbracket n \rrbracket^m}  \mathfrak{b}^{P_\mathbf{a},P_\mathbf{b}}_{\ell}(I,J) h \Big( w(\mathbf{a},I), w(\mathbf{b},J), C_{(I,J)}\Big)   \\
      \widetilde{\Pi}^{h,k}_{w,P}(\mathbf{a},\mathbf{b})  & = \frac1k\sum_{\ell = 1}^k \sum_{I,J  \in \llbracket n \rrbracket^m}  \mathfrak{b}^{P_\mathbf{a},P_\mathbf{b}}_{\ell}(I,J) \Pi_{I,J}.
\end{align*}
\end{remark}


Note that because incomplete U-statistics are not U-statistics in general, the incomplete minibatch transport plan estimator do not define a transport plan between the full distributions in general, i.e., their marginals are not equal to probability vectors $\mathbf{a}$ and $\mathbf{b}$.
In the following section, we discuss more closely the difference between drawing with or without replacement.

\subsubsection{Drawing data with or without replacement}

Our general flexible formalism allows us to define several minibatch strategies by playing with the probability law on tuples. The laws can be also different between the source and the target distributions. In particular, as given in examples, the cases of drawing with or without replacement. An estimator based on sampling without replacement is the most common practice when we have access to $n$ samples. While this drawing has been investigated for minibatch OT losses, the case with replacement, which appears in the GANs formalism, remains an open question that we aim at answering. The minibatch OT losses represent a weighted sum of Wasserstein distance over batches of size $m$. In the case of sampling without replacement, they are generalized unbiased U-statistics while with a sampling with replacement, we get generalized biased V-statistics. Precisely, they are two sample U-statistics or V-statistics of order $2m$ (see \citep{LeeUstats}) and $h$ is a U-statistic kernel. Interestingly, similar biased and unbiased estimators have been designed to estimate MMD \citep{MMD_Gretton}. 

Finally an important parameter is the value of the minibatch size $m$. In the case of sampling without replacement, we remark that the minibatch procedure allows us to interpolate between OT, when $m=n$ and averaged pairwise distance, when $m=1$. This property is not shared by the sampling with replacement. Indeed when $m=n$, it does not correspond to original OT due to the repetition of data. It only converges to the true OT when $n \rightarrow \infty$. This effect will be illustrated later on the averaged transport plan and on toy examples in the following section.

\subsection{Illustration on simple examples}

To illustrate the effect of the minibatch paradigm on the transport plan and the connections between source and target samples, we compute the minibatch OT plans for several values of $m$ on two simple examples. Furthermore, we also compare the minibatch OT plans of the different $P$ laws we defined in example \ref{def:law_indices_rep} and \ref{def:law_indices_rep2}. For experiments, we define two estimators. $\overline{h}^\mathtt{W}$ (resp. $\overline{\Pi}^h_\mathtt{W}$) with law $P_\mathbf{a}^\mathtt{W}$ and reweighting function $w^\mathtt{W}$ stands for the minibatch Wasserstein loss (resp. minibatch OT plan) over the m-tuples without repetitions. And $\overline{h}^\mathtt{U}$ (resp. $\overline{\Pi}^h_\mathtt{U}$) with law $P_\mathbf{a}^\mathtt{U}$ and reweighting function $w^\mathtt{U}$ stands for the minibatch Wasserstein loss (resp. minibatch OT plan) over the m-tuples $I$.

\paragraph{Distributions in 1D} The 1D case is an interesting problem because we have access to a closed-form of the optimal transport solution which allows us to calculate the closed-form of a minibatch paradigm. Indeed, the solution can be computed with a sorting algorithm which gives an appealing $\mathcal{O}(nlog(n))$ complexity compare to the initial $\mathcal{O}(n^3log(n))$. 

We suppose that we have a probability vector $\uu$ and we recall that $\overline{h}^\mathtt{W}$ corresponds to the minibatch OT losses defined in \citep{fatras2019batchwass}.
We assume (without loss of generality) that the points are ordered in their own distribution. In such a case, we can compute the 1D Wasserstein 1 distance with cost $c(\xx,\yy)=|\xx-\yy|$ as:
$ W(\uu, \uu) = \frac{1}{n} \sum_{i=1}^n \vert \xx_i - \yy_j \vert$ and the OT matrix is simply an identity matrix scaled by $\frac{1}{n}$ (see remark 2.28 \citep{COT_Peyre} for more details).
 After a short combinatorial calculus (given in appendix), the 1D minibatch transport matrix coefficient between sorted samples $(\overline{\Pi}^W_\mathtt{W})_{j,k}$ can be computed as:
\begin{align*}
(\overline{\Pi}^W_\mathtt{W})_{j,k} = \frac{1}{m} \dbinom{n}{m}^{-2} \sum_{i=i_{\text{min}}}^{i_{\text{max}}} \dbinom{j-1}{i-1} \dbinom{k-1}{i-1} \dbinom{n-j}{m-i} \dbinom{n-k}{m-i}
\end{align*}
where $i_{\text{min}} = \text{max}(0, m-n+j, m-n+k)$ and $i_{\text{max}} = \text{min}(j, k)$. $i_{\text{min}}$ and $i_{\text{max}}$ represent the sorting constraints.

We show on the first row of Figure \ref{fig:1D_unif} the minibatch OT plans $\overline{\Pi}^W_\mathtt{W}$ with $n=20$ samples for different values of the minibatch size $m$. On the second row of the figure a plot of the distributions in several rows of $\overline{\Pi}^W_\mathtt{W}$, to illustrate the number of connections. We give the OT plans for entropic and quadratic regularized OT between full distributions for comparison purpose. It is clear from the figure that the OT matrix densifies when $m$ decreases, which is a similar effect to entropic regularization. Note the more localized spread of mass of quadratic regularization that preserves sparsity as discussed in \citep{blondel2018}.

While the entropic regularization spreads the mass in a similar manner for all samples, minibatch OT concentrates the mass at the extremities. Note that the minibatch OT matrices solution is for ordered samples and do not depend on the position of the samples once ordered, as opposed to the regularized OT methods. This will be better illustrated in the next example.

Finally, a close form is also available in the case of drawing with replacement. We provide it in appendix. Unfortunately, its computational complexity makes it hard to use in practice.

\begin{figure*}[t]
    \centering
    \includegraphics[scale=0.4]{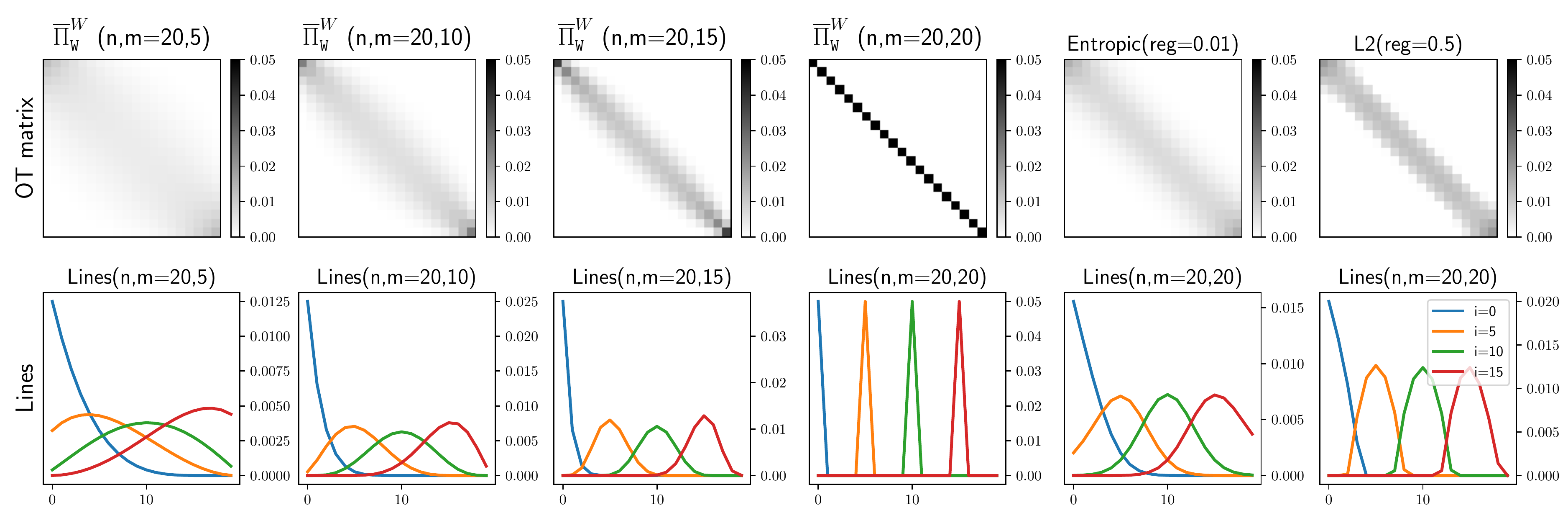}
    \caption{Several OT matrices between distributions with $n=20$ samples in 1D. The first row shows the minibatch OT matrices $\overline{\Pi}^W_\mathtt{W}(\uu, \uu)$ for different values of $m$, the second row provides the shape of the distributions on the rows of $\overline{\Pi}^W_\mathtt{W}(\uu, \uu)$. The two last columns correspond to classical entropic and quadratic regularized OT.}
    \label{fig:1D_unif}
\end{figure*}

\paragraph{Minibatch Wasserstein in 2D} We illustrate several OT matrices between two empirical distributions of 10 2D-samples each in Figure \ref{fig:2D_gauss}. We consider the MBOT transport plan for several batch sizes, the entropic and quadratic regularized OT between full distributions. We use two 2D empirical distributions (point cloud) where the samples have a cluster structure and the samples are sorted \emph{w.r.t.} their cluster. We first discuss the sampling without replacement. We can see from the OT matrices in the first row of the figure that the cluster structure is more or less recovered with the regularization effect of the minibatches (and also regularized OT). On the second row one can see the effect of the geometry of the samples on the spread of mass. Similarly to 1D, for Minibatch OT, samples on the border of the simplex cannot spread as much mass as those in the center and have darker rows. This effect is less visible on regularized OT.

\begin{figure*}[t]
    \centering
    \includegraphics[scale=0.4]{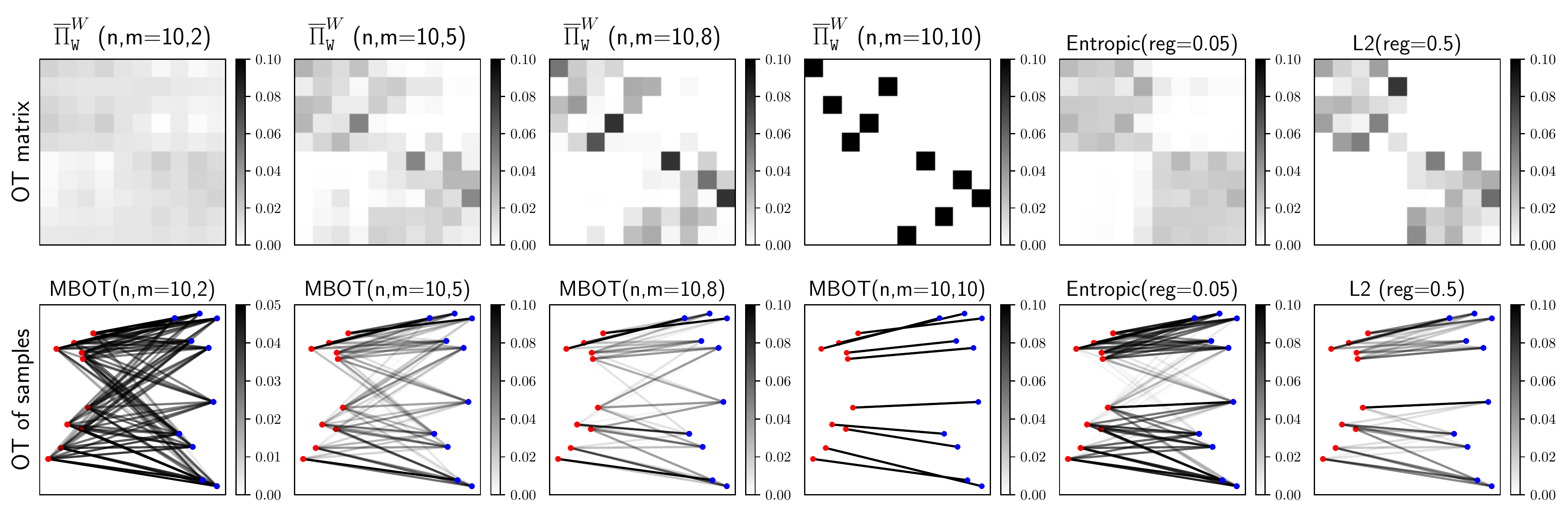}
    \caption{ Several OT matrices between 2D distributions with $n=10$ samples. The first row shows the minibatch OT matrices $\overline{\Pi}^W_\mathtt{W}(\uu, \uu) $ for different values of $m$. The second row provide a 2D visualization of where the mass is transported between the 2D positions of the sample.}
     \label{fig:2D_gauss}
\end{figure*}
 
We also illustrate the difference of transport plans between sampling with or without replacement. We consider the same setting as above but with 5 empirical data. On each column we show the transport plan and the shape of connection between samples. We can see that the estimator $\overline{\Pi}^W_\mathtt{U}(\uu, \uu)$ has always a denser plan, i.e. a bigger number of connections, than the estimator $\overline{\Pi}^W_\mathtt{W}(\uu, \uu)$. In particular, when $m=n=5$, we get the optimal transport plan with $\overline{\Pi}^W_\mathtt{W}(\uu, \uu)$ while we do not recover it with $\overline{\Pi}^W_\mathtt{U}(\uu, \uu)$ due to the fact that samples can be repeated. Now that we have rigorously defined how we can build minibatch Wasserstein losses between empirical measures, we study its loss properties.

\begin{figure*}[t]
    \centering
    \includegraphics[scale=0.4]{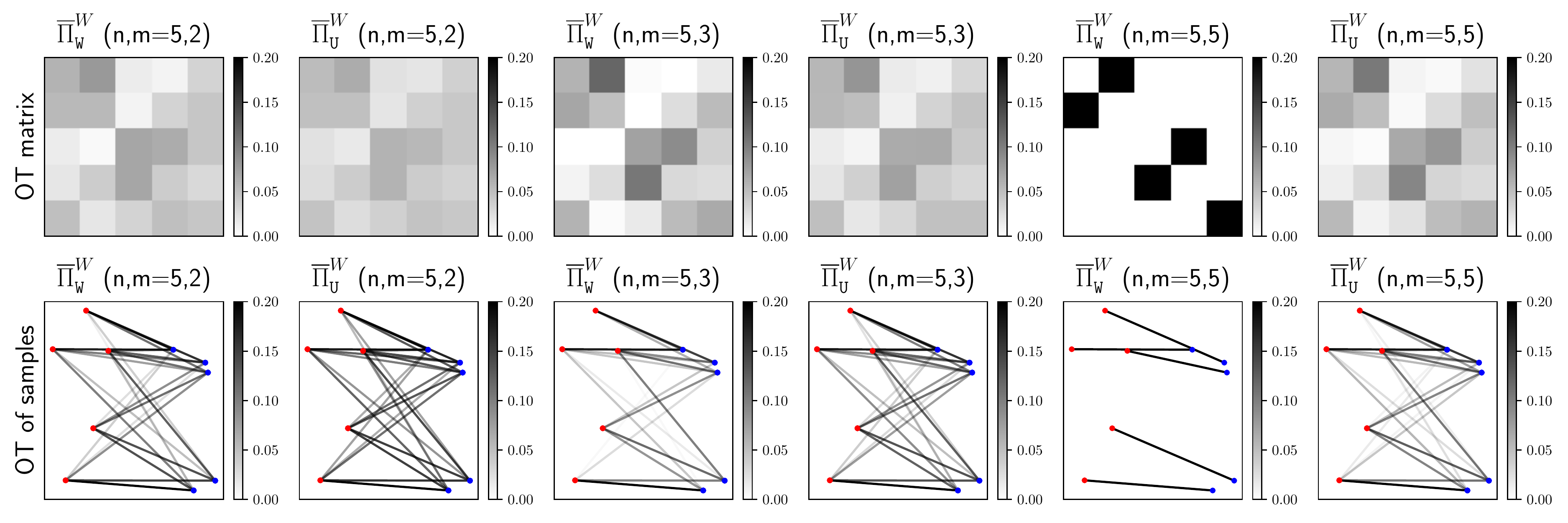}
    \caption{ Difference between transport plan estimators with 2D distributions and $n=5$ samples. Each column gives the OT plan $\overline{\Pi}^W_\mathtt{W}(\uu, \uu)$ or $\overline{\Pi}^W_\mathtt{U}(\uu, \uu)$ (top) and the shape of the distributions on the rows of the OT matrices (bottom). }
    
\end{figure*}

\subsection{Loss properties}\label{subsec:basic_prop}
We now review basic properties for our general minibatch OT losses formalism.

\begin{proposition}[Estimator properties]\label{prop:convexity}
The minibatch OT losses enjoy the following properties:
\begin{itemize}
    \item The losses are not distances
    \item The losses are symmetric
\end{itemize}

\begin{proof}
We give the proof that minibatch OT losses are not distances. Consider a uniform probability vector  and random $3$-data tuple $\XX=(\xx_1, \xx_2, \xx_3 )$ with distinct vectors. As $\overline{h}_{w,P}$ is a weighted sum of positive terms, it is equal to 0 if and only if each of its term is 0. But consider the minibatch term $I_1 = ( i_1, i_2 )$ and $ I_2 = ( i_1, i_3 )$, then obviously $h(w(\uu, I_1),w(\uu, I_2), C(\XX(I_1), \XX(I_2))) \ne 0$ as $\xx_2 \ne \xx_3$, finishing the proof.
\end{proof}
\end{proposition}


The symmetry of the losses is inherited from the optimal transport problem which is itself symmetric. The loss of the separability distance axiom means that for data fitting problems, the final solution will not match the target distribution. The axiom is recovered for minibatches without replacement when $m = n$ as we recover the original OT formulation. 

We defined minibatch OT losses and reviewed their basic properties. In what follows, we propose an elegant formulation which fixes this loss.

\subsection{Debiasing minibatch Wasserstein losses}\label{sec:DMBW}

As we have shown before, the minibatch OT losses are not distances, for general probability vectors and data n-tuple $\XX$, $\overline{h}_{w,P}(\mathbf{b}, \mathbf{b}) > 0$. This leads to an undesirable situation when one uses it for learning purposes as the final solution is not the target distribution but a shrunk version of it. Hence, we would like to debias the losses to get $\overline{h}_{w,P}(\mathbf{b}, \mathbf{b})=0$. We debias the minibatch OT losses by following the same idea as the Sinkhorn divergence, we remove half of each self term $\overline{h}_{w,P}(\mathbf{a}, \mathbf{a})$ and $\overline{h}_{w,P}(\mathbf{b}, \mathbf{b})$.

\begin{definition}[Debiased Minibatch Wasserstein estimators] Let $C = C^{n,p}$. Consider $1 \leq m \leq n$ be an integer and $h$ be the Wasserstein distance $W$, the entropic loss $W_{\varepsilon}$, the Sinkhorn divergence $S_{\varepsilon}$, or the Gromov-Wasserstein distance $GW$ for some ground cost $c(x,y)$, we define the following quantities:
\begin{equation}
\Lambda_{h,w,P, C(\XX, \YY)}(\mathbf{a},\mathbf{b}) :=  \overline{h}_{w,P, C(\XX, \YY)}(\mathbf{a},\mathbf{b}) - \frac12  \big( \overline{h}_{w,P, C(\XX, \XX)}(\mathbf{a},\mathbf{a}) + \overline{h}_{w,P, C(\YY, \YY)}(\mathbf{b},\mathbf{b}) \big),
\label{def:DebiasedMinibatchEstimators}
\end{equation}

That we note when it is clear of context $\Lambda_{h,w,P}$ and its incomplete counter part:
\begin{equation}
\widetilde{\Lambda}_{h,w,P, C(\XX, \YY)}^k(\mathbf{a},\mathbf{b}) :=  \widetilde{h}_{w,P, C(\XX, \YY)}^k(\mathbf{a}, \mathbf{b}) - \frac12 \big(  \widetilde{h}_{w,P, C(\XX, \XX)}^k(\mathbf{a}, \mathbf{a}) + \widetilde{h}_{w,P, C(\YY, \YY)}^k(\mathbf{b}, \mathbf{b}) \big).
\label{def:DebiasedMinibatchIncompleteGW}
\end{equation}
\end{definition}

\begin{remark} We keep making the slight abuse of notation to consider all OT kernels with $\overline{h}$, but we explicit the loss $\Lambda$ for a Gromov-Wasserstein loss. We note the ground cost $C^{n,p, 1}$ and $C^{n,p, 2}$ as $C^1$ and $C^2$ for sake of readability. With the $\GW$ kernel, the loss $\Lambda$ is equal to:
\begin{align}
&\Lambda_{h,w,P, C^1(\XX, \XX), C^2(\YY, \YY)}(\mathbf{a},\mathbf{b}):=\overline{h}_{w,P, C^1(\XX, \XX), C^2(\YY, \YY)}(\mathbf{a},\mathbf{b})  \nonumber\\
&-\frac12  \big( \overline{h}_{w,P, C^1(\XX, \XX), C^1(\XX, \XX)}(\mathbf{a},\mathbf{a}) + \overline{h}_{w,P, C^2(\YY, \YY), C^2(\YY, \YY)}(\mathbf{b},\mathbf{b}) \big),
\end{align}
\end{remark}

It is straight forward to see that $\Lambda_{h,w,P}(\mathbf{b}, \mathbf{b}) = 0$. A similar loss has been proposed in \citep{salimans2018improving} as a generalized energy distance using the entropic Wasserstein distance as metric. While their loss debiased the minibatch bias, it still had a bias from the entropic regularization. They then relied on the energy distance properties to argue positiveness. The downside of this loss is that it needs to rely on a metric to be positive, however the entropic regularized Optimal Transport is not a metric between probability distributions as $W^\varepsilon(\mathbf{a}, \mathbf{a}) \neq 0$.

We bring insights to this debiased loss and compare its differences to the minibatch OT losses both mathematically and empirically. We use our loss $\Lambda_{h,w,P}$ with the Wasserstein distance and Sinkhorn divergence because they respect the distance separability axiom. Unfortunately, we prove that even if we consider the Wasserstein distance, this loss is not positive and we will give counter examples.

\paragraph{Positivity}

The loss function $\Lambda_{h,w,P}$ is composed of three terms of the form of $\overline{h}_{w,P}$, then it is possible to estimate it with the different estimators $\overline{h}^\mathtt{W}$ and $\overline{h}^\mathtt{U}$ we defined in section \ref{sec:definitions}. When estimated with $\overline{h}^\mathtt{W}$ (resp $\overline{h}^\mathtt{U}$), we denote $\Lambda_{\overline{h}^\mathtt{W}}$ (resp. $\Lambda_{\overline{h}^\mathtt{U}}$). Let us consider 8 points on the unit circle equally distributed. Then let us add a perturbation as a rotation to each point position, where the rotation vary from 0 to $\pi$. The perturbed distribution becomes our target distribution. When computing the quantity $\Lambda_{W_p}(\uu,\uu)$, with $p \geq 2$ and an euclidean ground cost, it can return a negative value. We give the variations of the debiased minibatch OT losses in function of the pertubartion in figure \ref{fig:comparison_debias_estimator} for both the estimators $\overline{h}^\mathtt{W}$ and $\overline{h}^\mathtt{U}$. 
\begin{figure}[t]
    \centering
    \includegraphics[scale=0.425]{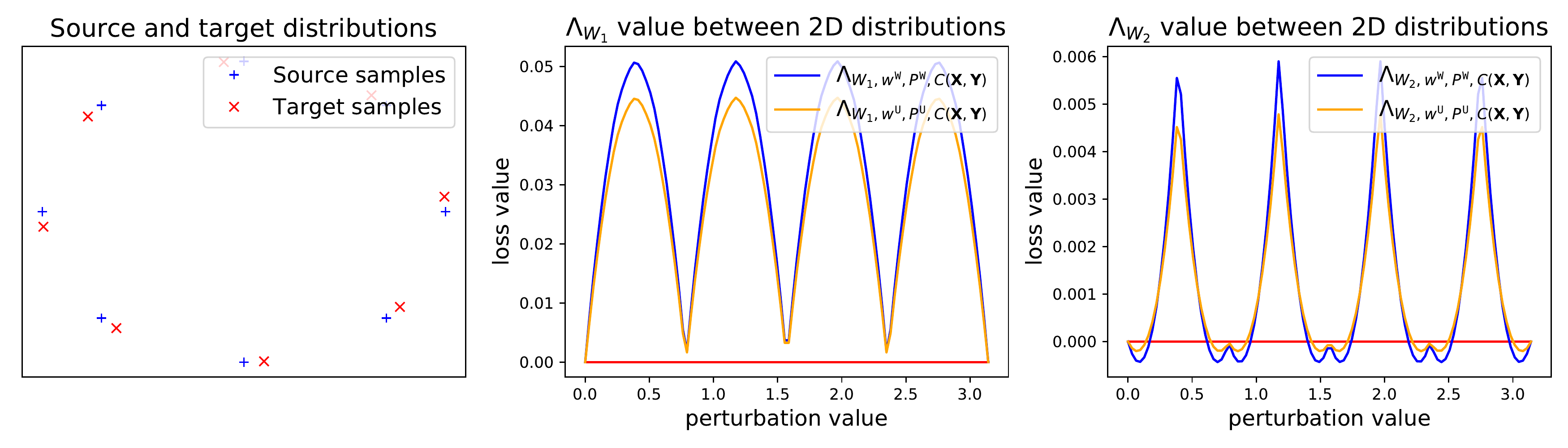}
    \caption{Positivity counter example. (Left) source and target distribution for a given perturbation. (Middle and right) Comparison of different estimator values for $\Lambda_{W_1}$ and $\Lambda_{W_2}$ with an euclidean ground cost between the distributions. The red line is the y-axis equal to 0.}
    \label{fig:comparison_debias_estimator}
\end{figure}
\newline

The loss function might not be always positive for particular case but in practice, we always had a positive loss and it performed better than the biased minibatch OT losses. Furthermore, while we have been able to find counter examples for $p\geq 2$, we have not found any counter example for $W_1$. Hence $\Lambda_{W_1}$ might be a positive loss function, and the proof is left as future work.

We have defined a loss which is based on the minibatch Wasserstein distance and which respect the separability axiom of distance. These desirable properties come with downsides as the loss function might not be always positive in practice. We now take a statistical point of view and we will carry concentration inequalities and optimization properties.

\section{Learning with minibatch OT: statistical and optimization properties}\label{sec:learning}

In this section, we aim at developing quantitative statistical and optimization results and we start with statistical bounds.
\subsection{Concentration bound}

In the case of sampling without replacement bounded and uniform measures, we were able to show concentration bounds between our estimator and its expectation (Theorem 1, \cite{fatras2019batchwass}). We first do a parallel between the losses defined in our previous paper and the losses $\overline{h}_{w,P}$.

\begin{remark}
As proven in the case of the reweighting function $w^\mathtt{U}$, the probability distributions over $m$-tuples without replacement $P^\mathtt{W}$ and the uniform probability vectors $\mathbf{a}$ and $\mathbf{b}$, our estimator corresponds to the discrete-discrete loss of our previous paper \citep{fatras2019batchwass}. Furthermore, taking its expectation over minibatches gives the continuous-continuous loss.
\end{remark}

We are now interested to find similar and more general results for the asymptotic behavior of our estimator $\widetilde{h}_{w,P}^k(\mathbf{a}, \mathbf{b})$ and its deviation to its expectation $\mathbb{E}\widetilde{h}_{w,P}^k(\mathbf{a}, \mathbf{b})$. We will give a bound for several scenarios. For bounded measures, we will prove that we have a Hoeffding inequality such as in \citep{fatras2019batchwass}. Then we relax the boundness condition to give a concentration bound for subgaussian measures.

In this context, the probability vectors $(\aa^{(n)})$ and $(\bb^{(n)})$ are sequences which depend on the number of data $n$. More precisely $(\aa^{(n)})$ and $(\bb^{(n)})$ are sequences of vectors such that for each $n \in \N$, $\aa^{(n)}, \bb^{(n)} \in \Sigma_n$, we denote the space of these sequences as $(\aa^{(n)}), (\bb^{(n)}) \in \Sigma$. The sequences of probability vectors $(\aa^{(n)})$ and $(\bb^{(n)})$ can not be taken arbitrarily if we want to guarantee convergence. Hence we rely on local constraints that the probability vectors $\aa$ and $\bb$ must verify. 

\begin{definition}[Local averages conditions]\label{def:loc_cond}  Let $(\aa^{(n)}) \in \Sigma$ and two integers $n,m \in \mathbb{N}^*$ such as $n \geq m$. We say that $(\aa^{(n)})$ verifies the \textit{local mean condition} if there exists a constant $D>0$ and $\gamma \in (0,1]$ such that for any $n \in \mathbb{N}^*$ and $I \subset \llbracket n\rrbracket $ with $|I| = m$ we have:
\begin{equation}\label{DEF : loc_mean}
    \frac{1}{m} \sum_{i \in I} \aa^{(n)}_i \leq \frac{D}{n^{\gamma}}. 
\end{equation}
We write that $(\aa^{(n)})$ satisfies $\mathtt{Loc_A}(m,\gamma,D))$ (or $\mathtt{Loc_A}(m,\gamma)$) when the constant $D$ is implicit).\\
\textup{(ii)} Analogously, $(\aa^{(n)})$ is said to verify the local geometric mean condition if there exists a constant $D>0$ and $\gamma > 0$ such that for any $n \in \mathbb{N}^*$ and $I  \in  \llbracket n\rrbracket^m $ we have
\begin{equation}\label{DEF : loc_prod}
   \Big( \Pi_{i \in I} \aa^{(n)}_i \Big)    ^{\frac{1}{m}} \leq \frac{D}{n^{\gamma}}. 
\end{equation}

\noindent 
We write that $(\aa^{(n)})$ verifies $\mathtt{Loc_G}(m,\gamma,D))$ (or $\mathtt{Loc_G}(m,\gamma)$) when the constant $D$ is implicit).
\end{definition}
A straight forward example is the uniform vector $\uu^{(n)}$ which respects $\mathtt{Loc_A}(m,1,1)$ local mean condition and $\mathtt{Loc_G}(m,1,1)$ for the local product condition. Thus Eq.\eqref{DEF : loc_mean} naturally extends and quantifies the fact that a sequence has uniformly controlled $m$-averages. We also observe that for any generic sequence $(\aa^{(n)})$ in $\Sigma$ verifies $\mathtt{Loc_A}(m,0, \frac{1}{m})$. Regarding the local product condition, Eq.\eqref{DEF : loc_prod} extends the fact that a sequence has uniformly controlled $m$-products. We illustrate the local constraints on the simplex in figure \ref{fig:local_constraints} with python ternary \citep{pythonternary}. We have the following result about the local constraints:

\begin{lemma}\label{lemma:local_constraints} Let $m \in \N^*$, $\gamma >0$ and $D > 0$. Let $(\aa^{(n)}) \in \Sigma$ be a sequence of probability vectors. The following statements hold:\\
\textup{(i)} If $(\aa^{(n)})$ verifies $\mathtt{Loc_A}(m,\gamma,D)$ or $\mathtt{Loc_G}(m,\gamma,D)$ then $\gamma \leq 1$.\\
\textup{(ii)} If $(\aa^{(n)})$ is $\mathtt{Loc_A}(m,\gamma,D)$ then $(\aa^{(n)})$ is $\mathtt{Loc_G}(m,\gamma,D)$.
\end{lemma}

\begin{figure}
    \centering
    \includegraphics[scale=0.6]{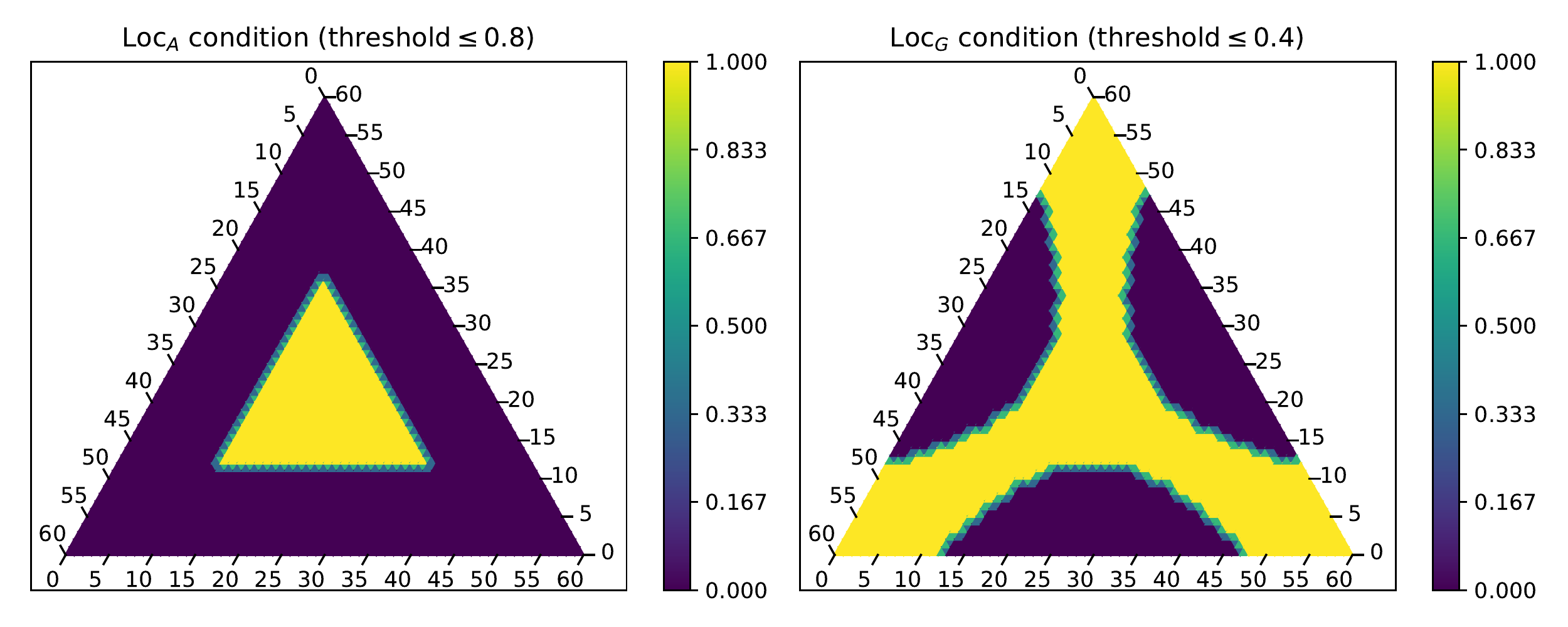}
    \caption{$\mathtt{Loc_A}$ and $\mathtt{Loc_G}$ local constraints illustrations on the simplex with $m=2$ and $n=3$.}
    \label{fig:local_constraints}
\end{figure}

\paragraph{Bounded data.}
For bounded data, we show that in order to obtain reasonable convergence properties of the estimators $\overline{h}_{w,P}(\aa^{(n)}, \bb^{(n)})$ we need to ensure that the sequences $(\aa^{(n)})$ and $(\bb^{(n)})$ verify the \textit{local condition} with enough decay, e.g. $(\aa^{(n)}), (\bb^{(n)})$ are $\mathtt{Loc_A}(m,\gamma)$ or $\mathtt{Loc_G}(m,\gamma)$ for a $\gamma$ sufficiently close to $1$.

\begin{theorem}[Maximal deviation bound for compactly supported distributions]\label{thm:inc_U_to_mean} Let $\delta \in (0,1) $, $k \geqslant 1$ an integer and $m \geqslant 1$ be a fixed integer. Let $C=C^{n,p}$ be as in \eqref{DEF : mb_matrix_map}. Consider two distributions $\alpha,\beta$, two n-tuples of empirical data $\XX \sim \alpha^{\otimes n}, \YY \sim \beta^{\otimes n}$ and a kernel $h \in \{W_p, W_p^p, W_\epsilon, S_\epsilon, \GW$\}. Let the reweighting function $w$ and the probability law over $m$-tuple $P$ be as in examples \ref{def:reweight_function_U}, \ref{def:reweight_function_N}, \ref{def:law_indices_rep}, \ref{def:law_indices_rep2}. Let the sequences of probability vectors $(\aa^{(n)}) \in \Sigma$ and $(\bb^{(n)}) \in \Sigma$ satisfy $\mathtt{Loc_A}(m,\gamma,D)$ and let $D>0$ and $\gamma \in (\frac34 ,1]$. We have a deviation bound for the sampling without replacement between $\widetilde{h}_{w,P}^k(\aa^{(n)}, \bb^{(n)})$ and $\E \overline{h}_{w,P}(\aa^{(n)}, \bb^{(n)})$ depending on the number of empirical data $n$ and the number of batches $k$:

\begin{equation}
\P \left( \vert \widetilde{h}_{w^\mathtt{W},P^\mathtt{W}}^k(\aa^{(n)}, \bb^{(n)}) - \E \overline{h}_{w^\mathtt{W},P^\mathtt{W}}(\aa^{(n)}, \bb^{(n)})\vert \geq M \Big(2\frac{ D^2 m^{ \frac12 }}{n^{2( \gamma - \frac{3}{4}) }} \sqrt{ 2 \log(\frac{2}{\delta})} + \sqrt{\frac{2 \log(\frac{2}{\delta})}{k}} \Big) \right) \leq \delta,
\label{eq:hoeffding_bounded_case_without_rep}
\end{equation}

where $M$ is a constant depending on the diameters of distribution supports. And for the sampling with replacement, let the sequences of probability vectors $(\aa^{(n)}), (\bb^{(n)})$ verify $\mathtt{Loc_G}(m,\gamma,D)$ for some $\gamma \in (1-\frac{1}{4m}, 1]$ and $D>0$, we have:
\begin{equation}
\P \left( \vert \widetilde{h}_{w^\mathtt{U},P^\mathtt{U}}^k(\aa^{(n)}, \bb^{(n)}) - \E \overline{h}_{w^\mathtt{U},P^\mathtt{U}}(\aa^{(n)}, \bb^{(n)})\vert \geq M \Big(  2\frac{ D^{2m} m^{ \frac12 }}{n^{2m(\frac{1}{4m} - 1 + \gamma)}} \sqrt{ 2 \log(\frac{2}{\delta})} + \sqrt{\frac{2 \log(\frac{2}{\delta})}{k}} \Big) \right) \leq \delta.
\label{eq:hoeffding_bounded_case_with_rep}
\end{equation}
\end{theorem}

\begin{remark}

In the case of uniform measures, we recover the sampling without replacement bounds of \citep{fatras2019batchwass} for both sampling with or without replacement:
\begin{align}
&\P \left( \vert \widetilde{h}_{w^\mathtt{W},P^\mathtt{W}}^k(\uu^{(n)}, \uu^{(n)}) - \E \overline{h}_{w^\mathtt{W},P^\mathtt{W}}(\uu^{(n)}, \uu^{(n)})\vert \geq M \Big( 2\sqrt{ 2 \frac{m}{n}\log(2/\delta)} + \sqrt{\frac{2 \log(2/\delta)}{k}} \Big) \right) \leq \delta, \nonumber\\
&\P \left( \vert \widetilde{h}_{w^\mathtt{U},P^\mathtt{U}}^k(\uu^{(n)}, \uu^{(n)}) - \E \overline{h}_{w^\mathtt{U},P^\mathtt{U}}(\uu^{(n)}, \uu^{(n)})\vert \geq M \Big( 2 \sqrt{2 \frac{m}{n}\log(2/\delta)} + \sqrt{\frac{2 \log(2/\delta)}{k}} \Big) \right) \leq \delta. \nonumber
\label{eq:hoeffding_unif_bounded}
\end{align}

\end{remark}

The proof is based on the U-statistics concentration bound proof \citep{Hoeffding1963} and can be found in appendix \ref{app_sec:concentration_bounded} with the proof of constant $M$. The proof idea is to rewrite our minibatch OT losses as a sum of independent terms and then to apply Hoeffding's lemma to the rewritten sum. The local constraints were necessary for a generalization of the concentration bounds to non uniform probability vectors $\aa^{(n)}$ and $\bb^{(n)}$.
These concentration bounds are also valid for our debiased minibatch OT loss as it is composed of three terms of the form $\overline{h}$. Furthermore, it is possible to extend this concentration inequality with an expectation over the batch couples and empirical data.

\begin{corollary}
    With the same hypothesis and notations as in Theorem \ref{thm:inc_U_to_mean}. The following inequality holds:
    \begin{align}
      &\mathbb{E}[ \vert   \widetilde{h}_{w^\mathtt{W},P^\mathtt{W}}^k(\aa^{(n)}, \bb^{(n)}) - \E \overline{h}_{w^\mathtt{W},P^\mathtt{W}}(\aa^{(n)}, \bb^{(n)}) \vert  ] \leqslant 20 \cdot M \max \Big(2\sqrt{2} D^2 \frac{m^{ \frac12 }}{n^{2( \gamma - \frac{3}{4}) }}, \sqrt{\frac{2}{k}} \Big),\\
      &\mathbb{E}[ \vert   \widetilde{h}_{w^\mathtt{U},P^\mathtt{U}}^k(\aa^{(n)}, \bb^{(n)}) - \E \overline{h}_{w^\mathtt{U},P^\mathtt{U}}(\aa^{(n)}, \bb^{(n)}) \vert  ] \leqslant 20 \cdot M \max \Big(2\sqrt{2} D^{2m} \frac{m^{ \frac12 }}{n^{1/2 - 2m + 2m\gamma}}, \sqrt{\frac{2}{k}} \Big).
    \label{thm:expectation_concentration_inequality}
    \end{align}
    And for our debiased minibatch OT loss:
    \begin{align}
      &\mathbb{E}[\vert \widetilde{\Lambda}_{h,w^\mathtt{W},P^\mathtt{W}}^k(\mathbf{a}^{(n)}, \mathbf{b}^{(n)}) - \mathbb{E}\widetilde{\Lambda}_{h,w^\mathtt{W},P^\mathtt{W}}^k(\mathbf{a}^{(n)}, \mathbf{b}^{(n)}) \vert] \leqslant 40 \cdot M \max\left(2\sqrt{2} D^2 \frac{m^{ \frac12 }}{n^{2( \gamma - \frac{3}{4}) }}, \sqrt{\frac{2}{k}}\right),\\
      &\mathbb{E}[\vert \widetilde{\Lambda}_{h,w^\mathtt{U},P^\mathtt{U}}^k(\mathbf{a}^{(n)}, \mathbf{b}^{(n)}) - \mathbb{E}\widetilde{\Lambda}_{h,w^\mathtt{U},P^\mathtt{U}}^k(\mathbf{a}^{(n)}, \mathbf{b}^{(n)}) \vert] \leqslant 40 \cdot M \max\left(2\sqrt{2} D^{2m} \frac{m^{ \frac12 }}{n^{1/2 - 2m + 2m\gamma}}, \sqrt{\frac{2}{k}}\right).
    \end{align}
\end{corollary} 

This deviation bound shows that if we increase the number of data $n$ and batches $k$ while keeping the minibatch size fixed, we get closer to the expectation. Remarkably for all OT kernel $h$, the bound does not depend on the dimension of $\mathcal{X}$, which is an appealing property when data lie in high dimensional space. A similar property was proven but only for $W_p^p$ (see proposition 20, \cite{weed2019}). Another nice property of the bounds above is that for a fixed minibatch size $m$, if one chooses $k$ proportional to the number of samples, the convergence of $\widetilde{h}_{w,P}^k$ to its mean is in $O(n^{-1/2})$ for a $O(n)$ computational complexity. 

Now let us consider a small experiments. To illustrate the dependence to the dimension, we consider 2 empirical data $n$-tuple, $\XX$ and $\YY \sim \alpha^{\otimes n}$, where $\alpha$ is the uniform distribution on the unit cube $[0,1]^{d}$, and compute $\Lambda_h(\uu,\uu)$ as a function of $n$. For a first experiment, we fix the batch size $m=128$ and we consider several values of dimension $d$. For a second experiment, we now fix the dimension $d$ and consider several batch sizes. Both experiments highlight no dependence of $\Lambda_h(\uu,\uu)$ to the dimension. To the best of our knowledge, it is the first time that a loss using the exact Wasserstein distance has no dependence on the dimension, making it a good candidate for learning problems.

We gave concentration bounds in the bounded data case and now we extend these results to the unbounded data case.

\begin{figure}[t]
    \centering
    \includegraphics[scale=0.5]{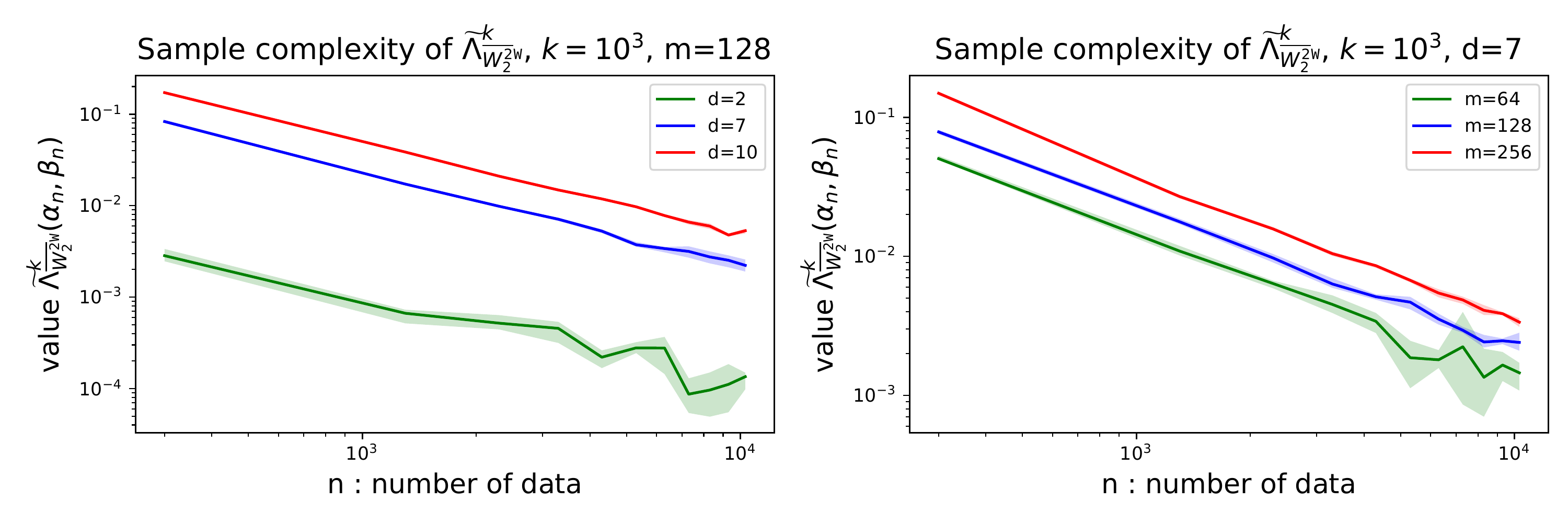}
    \caption{$\widetilde{\Lambda}_{\overline{h}^\mathtt{W}}^k(\aa,\bb)$ as a function of $n$ in log-log space. Here $\left(\aa, \bb\right)$ are two probability vectors associated to $\XX$ and $\YY \sim \alpha^{\otimes n}$, where $\alpha$ is the uniform distribution on the unit cube $[0,1]^{d}$. (Left) $\widetilde{\Lambda}_{\overline{h}^\mathtt{W}}^k(\aa,\bb)$ is tested for several values of $d \in\{2,7,10\}$ and a fix $m=128$ or (right) $\widetilde{\Lambda}_{\overline{h}^\mathtt{W}}^k(\aa,\bb)$ is tested for several values of $m \in\{64, 128, 256\}$ and a fix $d=7$. The experiments were run 5 times and the shaded bar corresponds to the 20\% and 80\% percentiles.}
    \label{fig:bary_map_toy_2D}
\end{figure}

\paragraph{Unbounded data.}

We supposed in the previous results that the distributions have a bounded support. We can relax this condition by supposing they have light tails, i.e., they are subgaussian. We consider the euclidean norm ($\|.\|_2$) and give a formal definition:

\begin{definition}[Subgaussian random vectors] A random vector $\xx \in \mathbb{R}^{d}$ is subGaussian, if there exists $\sigma \in \R$ so that:
$$
\mathbb{E} e^{\langle\mathbf{v}, \xx-\mathbb{E} \xx\rangle} \leq e^{\frac{\|\yy \|^{2} \sigma^{2}}{2}}, \quad \forall \yy  \in \mathbb{R}^{d}
$$
\end{definition}

The proof uses a related class of subgaussian random vectors and a discussion of the difference is available in appendix. In the case of subgaussian data, we can not rely on the Hoeffding inequality anymore as the data are not bounded. However we are able to get a similar concentration inequality. Hereafter we write $A \ll_{\gamma} B$ for $A,B >0$ if there exists a large constant $\tau = \tau( \gamma) >0$ such that $A \leq \tau B$.

\begin{theorem}[Concentration inequality sub-Gaussian data]\label{thm:sugaussian_concentration}
Let the cost $C=C^{n,p}$ be defined as in \eqref{DEF : mb_matrix_map}. Let $(\xx_i)_{1 \leq i \leq n} $ and $(\yy_i)_{1 \leq i \leq n} $ be two \textit{i.i.d.} sequences of random vectors such that $\xx_1 \in \operatorname{normSG}(\rho_\xx, \sigma^2 _\xx)$ and $ \yy_1 \in \operatorname{normSG}(\rho_\yy, \sigma^2 _\yy)$ with $\s_\xx, \s_\yy > 0$ and $\rho_\xx, \rho_\yy \in \R
^d$. Let us introduce 
\begin{align*}\label{thm_sub1}
    \s & := \min(\s_\xx, \s_\yy) \\
    \rho &  := \|\rho_\xx - \rho_\yy\|_2
\end{align*}
Let the sequence probability vectors $(\aa^{(n)}), (\bb^{(n)})$ verify $\mathtt{Loc_A}(m,\gamma,D)$ for some $\gamma \in (\frac{3}{4}, 1]$ and $D>0$. We assume that $n$ verifies the following condition:
\begin{equation}\label{thm_sub2}
    n \geq \tau(m, \s, \rho, D,p).
\end{equation}

\noindent
Consider $m \geqslant 1$ be a fixed integer and a kernel $h \in \{W_p, W_\epsilon, S_\epsilon$\}.  Let the reweighting function $w^\mathtt{W}$ and the probability law over $m$-tuple $P^\mathtt{W}$ be as in examples \ref{def:reweight_function_N} and \ref{def:law_indices_rep2}. Then we have the following concentration bound for the sampling without replacement:
\begin{equation}\label{thm_sub3}
 \P \left( \Big| \bar{h}_{w^\mathtt{W}, P^\mathtt{W}}(\aa^{(n)}, \bb^{(n)}) - \E \overline{h}_{w^\mathtt{W}, P^\mathtt{W}}(\aa^{(n)} , \bb^{(n)}) \Big| \geq  ( 2^{3p+4} m)^{\frac 12} \s^{p} D^2 \cdot \frac{\log(4n)^{\frac{p+1}{2}}}{n^{2 ( \gamma - \frac 34 )}}  \right) \leq 4 n^{-\frac{1}{2^p}},
\end{equation}
\end{theorem}
See the supplementary for a proof.

The proof uses a truncation argument where we split data which lie in a compact and data which do not. We can remark the following facts about our bounds in the unbounded case:

\begin{itemize}
    \item The decay and the constants in \eqref{thm_sub3} are artificial consequences of the constants chosen in the proof.
    
    \item The condition \eqref{thm_sub2} seems to be necessary. In the bounded case of theorem \ref{thm:inc_U_to_mean} such a condition was not needed.
    
    \item We loose a $\sqrt{\log(n)}$ factor between \eqref{thm_sub3} and the deviation bound between the complete estimator $\bar{h}$ and its mean from the compactly supported data case. We report to the appendix for a longer discussion on this comparison.
    

\end{itemize}

After studying concentration bounds for the minibatch OT loss in the bounded and unbounded data cases, we give similar bounds for the minibatch OT plan.

\paragraph{Minibatch Transport plan.}
As discussed before an interesting output of Minibatch Wasserstein is the minibatch OT plan $\overline{\Pi}^h_{w,P}$, but since it is hard to compute in practice we instead use $\widetilde{\Pi}^{h,k}_{w,P}$ and we investigate the error on the marginal constraints.  In our previous work, we were able to show a deviation between the marginals of our incomplete estimator on uniform measures and its expectation, we now aim at extending our previous result in a more general case.
In what follows, we denote $\Pi_{(i)}$ the $i$-th row of matrix $\Pi$ and $ \mathbf{1} \in \R^n $ the vector whose entries are all equal to $1$.

\begin{theorem}[Distance to marginals]\label{thm:dist_marg} Let $ \delta \in (0,1) $, two integers $m \leq n $ and consider two sequences of probability vectors $(\aa^{(n)}),(\bb^{(n)}) \in \Sigma$. Let a ground cost $C = C^{m,p}$ for some $p \geq 1$. Consider an OT kernel $h \in \{ W_p, W_p^p, W^{\varepsilon}, S^{\varepsilon}, \GW\}$. Let $C=C^{n,p}$ be as in \eqref{DEF : mb_matrix_map}. Suppose now that the probability law over $m$-tuples $P$ and the reweighting function $w$, as defined in \eqref{EQ : rew_fun} and \eqref{EQ : prob_fun}, satisfy the admissibility condition \eqref{eq:wPadmissible}. For all integers $ k \geqslant 1 $ and all integers $ 1 \leqslant i \leqslant n  $, we have:
\begin{equation}
\P \left(
\Big|  \widetilde{\Pi}^{h,k}_{w,P}(\aa^{(n)}, \bb^{(n)})_{(i)} \mathbf{1} - a_i^{(n)} \Big| \geq  \sqrt{\frac{2 \log(2/\delta)}{k}} \right) \leq \delta
\end{equation}
\end{theorem}

The proof uses the convergence of $\widetilde{\Pi}^{h,k}_{w,P}$ to $\overline{\Pi}^h_{w,P}$ and the fact that 
$\overline{\Pi}^h_{w,P}$ is a transport plan and respects the marginals. It is far easier to get this result as we always have bounded transport plan.
Let us now study the practical differences with the minibatch Wasserstein distance.
Thanks to the statistical properties of our estimator, we now know that minibatch OT losses can be used to measure similarities between distributions. We now study their behaviour with modern optimization techniques.

\subsection{Gradient and optimization}

Consider a standard parametric data fitting problem in the space of probability measures. Given discrete samples \(\left(\xx_{i}\right)_{i=1}^{n} \in \R^d\) from an unknown distribution \(\pi\), we want to fit a parametric model \(\theta \mapsto \beta_{\theta} \in \mathcal{M}_{1}^{+}(\R^d) \) to \(\pi\) using a contrast function $\rho : \mathcal{M}_{1}^{+}(\R^d) \times \mathcal{M}_{1}^{+}(\R^d) \rightarrow \R_{+}$. We thus look for the solution of
\begin{equation}\label{eq:opt_problem}
\hat{\theta} = \text{arg}\min _{\theta \in \Theta}\quad  \rho(\alpha, \beta_{\theta}) 
\end{equation}
where $\alpha$ is the empirical distribution of sample $(x_i)_{i=1}^n$.  When the contrast function is chosen to be the Wasserstein distance, the above optimization problem is known as Minimal Wasserstein estimation \citep{bernton2019parameter}.  Learning many generative models can also be framed as solving \eqref{eq:opt_problem} with the contrast function being equal to some (possibly regularized) Optimal Transport cost \citep{genevay_2018}. 

One way to compute the estimator $\hat{\theta}$ given by \eqref{eq:opt_problem}, is to use a stochastic solver for semi-discrete optimal transport (chapter 5 \citep{COT_Peyre}). 
This strategy is unfortunately not efficient in practice \citep{Genevay_2016, seguy2018large}. A common alternative approach is to use stochastic gradient descent with stochastic gradients computed based on minibatches sampled from $\alpha, \beta_{\theta}$ as was done for example in \citep{genevay_2018, salimans2018improving}. It was noted in \citep{fatras2019batchwass}, that such stochastic gradients are biased, but they can nevertheless be treated as unbiased stochastic gradients of a Minibatch Wasserstein loss. The following theorems combined are a generalization of that result, that is applicable also when the OT kernel is not regularized by an entropic term, and the cost matrix is not necessarily differentiable. The full formal statement statement and proof of those theorems can be found in Appendix \ref{app_sec:optimization}.  
\begin{theorem}  \label{thm:derivative_OT_cost}
Let $\aa, \bb \in \Sigma_m$. Let $\XX$ be a $\R^{dm}$-valued random variable, and $\{\YY_{\theta} \}$ a family of $\R^{dm}$-valued random variables defined on the same probability space, indexed by $\theta \in \Theta$, where $\Theta \subset \R^{q}$ is open. Assume that $\theta \mapsto \YY_{\theta}$ is $C^1$. Denote $C = C^{m,p}$ for some $p \geq 1$ and let $h \in \{W, W^{\epsilon}\}$. Then the function $\theta \mapsto -h(\aa,\bb,C(\XX, \YY_{\theta}))$ is Clarke regular and for all $1 \leq i \leq q$ we have:
\begin{align} \label{eq:exchange_theorem_eq1}
\partial_{\theta_i} h(\aa,\bb,C(\XX, \YY_{\theta})) = \{ -\text{tr}(P \cdot D^{T})\cdot (\nabla_{\theta_i} Y): & P \in \Pi(h, C(\XX, \YY_{\theta}), \aa, \bb), \\
D \in \R^{m, m}, \hspace{2pt} & D_{j,k} \in \partial_{Y} C_{j,k}(\XX, \YY_{\theta}) \} \nonumber
\end{align}
where $\partial_{\theta_i}$ is the Clare subdifferential with respect to $\theta_i$, $\partial_Y C_{j,k}$ is the subdifferential of the cell $C_{j,k}$ of the  cost matrix with respect to $Y$ and $\Pi(h, C,\aa, \bb)$ is defined in definition \ref{def:MBTP}. For $h = GW$, when the cost matrix is differentiable (that is $p > 1$), the function $-h(\aa, \bb, C(\XX, \XX), C(\YY_{\theta},\YY_{\theta}))$ is also regular, and an analogous formula holds. 
\end{theorem}
\begin{theorem} \label{thm:exchange_grad_exp_sm}
Let $\aa, \bb, \XX, \YY$ be as in theorem \ref{thm:derivative_OT_cost}, and assume in addition that the random variables $\XX, \{Y_{\theta}\}_{\theta \in \Theta}$ have finite $p$-moments. For $h \in \{W, W^{\epsilon} \}$, under an additional integrability assumption, we have:
\begin{align} \label{eq:exchange_theorem_eq2}
\partial_{\mathbf{\theta}} \expect \left[ h(\aa, \bb, C(\XX, \YY_{\theta})) \right] =  \expect \left[ \partial_{\mathbf{\theta}} h(\aa, \bb, C(\XX, \YY_{\theta})) \right].
\end{align}
with both expectation being finite. Furthermore the function $\theta \mapsto - \expect \left[ h(\aa, \bb, C(\XX, \YY_{\theta})) \right]$ is also Clarke regular. An analogous results holds for $h = GW$,  given that the cost is differentiable (that is $p > 1$) and random variables $\XX, \{Y_{\theta}\}$ have finite $2p$-moments.
\end{theorem}
\begin{remark}
If the cost matrix in Theorem \ref{thm:derivative_OT_cost} is differentiable with respect to $\YY$ (that is for $p > 1$) and $h = W^{\epsilon}$, then all the Clarke derivatives in \eqref{eq:exchange_theorem_eq1}, \eqref{eq:exchange_theorem_eq2} are sets consisting of one element, which is the gradient of respective functions. In that case we may deduce for $h = S^{\epsilon}$ a formula for the gradient from Theorem \ref{thm:derivative_OT_cost} and an interchange of expectation and integration from Theorem \ref{thm:exchange_grad_exp_sm}. The above results are therefore a strict generalization of Theorem 3 in \citep{fatras2019batchwass}.
\end{remark}
Suppose that in the above theorem the random variable $\XX$ is distributed according to $\alpha^{\otimes m}$, each random variable $\YY_{\theta}$ is distributed according to $\beta^{\otimes m}$ and $\XX$ is independent of family of variables $\{\YY_{\theta} \}_{\theta \in \Theta}$. Then theorem \ref{thm:derivative_OT_cost} implies that it is easy to compute unbiased stochastic gradients of a Minibatch Wasserstein loss, defined as follows:

\begin{definition}[Minibatch Wasserstein] Let $\alpha, \beta \in \mathcal{P}_p(\R^n)$ be two measures on an Euclidean space with finite $p$-moments, for $p \geq 1$. Chose an integer $m \in \mathbb{N}$ and let $h \in \{W, W^{\epsilon}, S^{\epsilon} \}$. Given the ground cost the ground cost $C^{m,p}$ defined in Eq.\eqref{DEF : mb_matrix_map}, we define the following quantity:
\begin{equation}
U_h^m(\aa, \bb, \alpha, \beta) :=  \expect_{(\XX, \YY) \sim  \alpha^{\otimes m } \otimes \beta ^{\otimes m}} \big[ \overline{h}_{w,P, C^{m,p}(\XX,\YY)}(\aa, \bb) \big]
\label{def:expectationMinibatch}
\end{equation}
for any $\aa,\bb \in \R^n$. We define an analogous quantity For $h = GW$. Assuming that $\alpha, \beta \in \mathcal{P}_{2p}(\R^n)$, we denote
\begin{equation}
U_h^m(\aa, \bb, \alpha, \beta) :=  \expect_{(\XX, \YY) \sim  \alpha^{\otimes m } \otimes \beta ^{\otimes m}} \big[ \overline{h}_{w,P, C^{m,p}(\XX,\XX), C^{m,p}(\YY,\YY)}(\aa, \bb) \big]
\label{def:expectationMinibatchGW}
\end{equation}
for any $\aa, \bb \in \R^n$.
\end{definition}
The fact that the above is well follows trivially from the assumption that measures $\alpha, \beta$ have finite $p$-moments (or finite $2p$-moments for $h=GW$) and a standard bound \eqref{eq:p_moment_bound} used in the proof of Theorem \ref{thm:exchange_grad_exp_sm_app}. In fact, the finiteness of \eqref{def:expectationMinibatch} and \eqref{def:expectationMinibatchGW} is show in that proof. We finish this section by noting, that Theorem \ref{thm:exchange_grad_exp_sm} implies that if we use the Minibatch Wasserstein loss with $h \in \{W, W^{\epsilon}, GW\}$  (or $h=S^{\epsilon}$ for $p>1$) as a contrast function in \eqref{eq:opt_problem}, then the objective function is minus Clarke regular. In this case, it is known that SGD with decreasing step sizes converges almost surely to the set of critical points of Clarke generalized derivative \citep{davis2020stochastic}, \citep{majewski2018analysis}. Finally note that on contrary to \citep{fatras2019batchwass}, we were able to relax the assumptions on the compactness support of distributions to exchange gradients and expectations with instead supposing finite moments.
\section{Numerical experiments}\label{sec:exp}
After presenting the formalism of minibatch Wasserstein, studied its statistical and optimization properties and defining a new unbiased  loss function, we now explore different applications of our methods. To compare the minibatch OT losses and their debiased counter parts, we set two qualitative experiments and a quantitative one. The first experiment is a gradient flow between male and female images and the second is a Monge map estimation between male and female images. The quantitative experiment consists in learning a GAN where we investigate the inception score of several minibatch OT losses. Our fourth experiment is a color transfer experiment that we introduced in \citep{fatras2019batchwass}, we complete it by investigating the sparsity degree of the resulting minibatch OT plan. Finally, our two last experiments are dedicated to the minibatch Gromov-Wasserstein loss where we investigate the inherited properties from the Gromov-Wasserstein distance. As our experiments are learning scenarios, we have uniform measures and consider the reweighting function $w^\mathtt{U}$, regarding the probability laws on tuples, we investigate both $P^\mathtt{W}$ and $P^\mathtt{U}$. Note that $w^\mathtt{U}$ and $P^\mathtt{W}$ check the admissibility condition \eqref{prop:admi_plan} Finally, experiments were computed on a single GTX Titan GPU.

\subsection{Gradient Flow between human faces}
The first experiment we conducted is a gradient flow of a source distribution towards a target distribution. It corresponds to the nonparametric setting of a data fitting experiments such as GANs. For two given probability vectors $\aa$ and $\bb$, and support $\YY$ associated to $\bb$, the goal of gradient flows is to model a support $\xx_t$  which at each iteration follows the loss gradient $\xx_t \mapsto h(\aa, \bb, C(\XX_t, \YY))$. This experiment has been investigated in \citep{liutkus19a, Peyre2015}. 
In this non parametric setting, $\aa$ is parametrized by a vector position $\boldsymbol{x}(t)$ which encodes its
support. We apply it between male and female images from the celebA dataset \citep{liu2015faceattributes} where we seek a natural evolution along iterations. CelebA is a large-scale face attribute dataset with 202,599 face images, 5 landmark locations, and 40 binary attribute annotations per image. We only considered 5000 male images and 5000 female images. We build the training dataset by cropping and scaling the aligned images to 64 x 64 pixels.

Following the procedure in \citep{feydy19a, fatras2019batchwass}, the gradient flow algorithm uses an Euler 
scheme and we start from an initial distribution at time $t=0$. At each iteration we numerically integrate the ordinary differential equation:
$$
\dot{\boldsymbol{\XX}}(t)=- \nabla_{\boldsymbol{x}} \widetilde{\Lambda}_{h,w,P, C(\XX(t), \YY)}^k(\aa,\bb).
$$
As our losses take probability vectors as inputs, we need to correct an inherent scaling when we calculate the gradient.
The scaling comes from the sample weights $a_i$, which is equal to $1/m$ in our case.
To correct the scaling, we apply a re-scaling to the gradient equal to $m$. Finally, for a n-tuples of
data $\XX$ we integrate:
\begin{equation}
\dot{\boldsymbol{\XX}}(t)=-m \nabla_{\boldsymbol{x}}\left[\widetilde{h}_{P,w, C(\XX(t), \YY)}^k(\aa, \bb) - \frac12 \Big( \widetilde{h}_{P,w, C(\XX(t), \YY)}^k(\aa, \aa) + \widetilde{h}_{P,w, C(\XX(t), \YY)}^k(\bb, \bb) \Big)\right].
\end{equation}

We conducted gradient flow experiments for both minibatch OT loss and debiased 
minibatch OT loss with the Wasserstein distance as OT kernel and probability laws on $m$-tuples $P^\mathtt{W}$ and $P^\mathtt{U}$. However, as our images lie in high dimension, the euclidean ground 
cost is not meaningful anymore, that is why we followed the experiments of \citep{liutkus19a}
where they considered gradient flows in the latent space of a pre-trained AutoEncoder.
We considered a pre-trained DFC-VAE \citep{hou2017deep} with $64 \times 64$ image and perform gradient flow in the encoder's latent space. In addition of the typical \citep{KingmaW13} loss, DFC-VAE considers the difference between features of the input image and the reconstructed image through a pre-trained neural network. In our case, we considered a pre-trained VGG-19 network and the layers 1-2-3. We trained the DFC-VAE with a batch size of 64 for 5 epochs over the training dataset and use Adam method for optimization with initial learning rate of 0.0005, see \citep{hou2017deep} for more details. With the feature extraction, we are able to improve the quality of final distribution's images. 

\begin{figure}[t]
    \centering
    \includegraphics[scale=0.4]{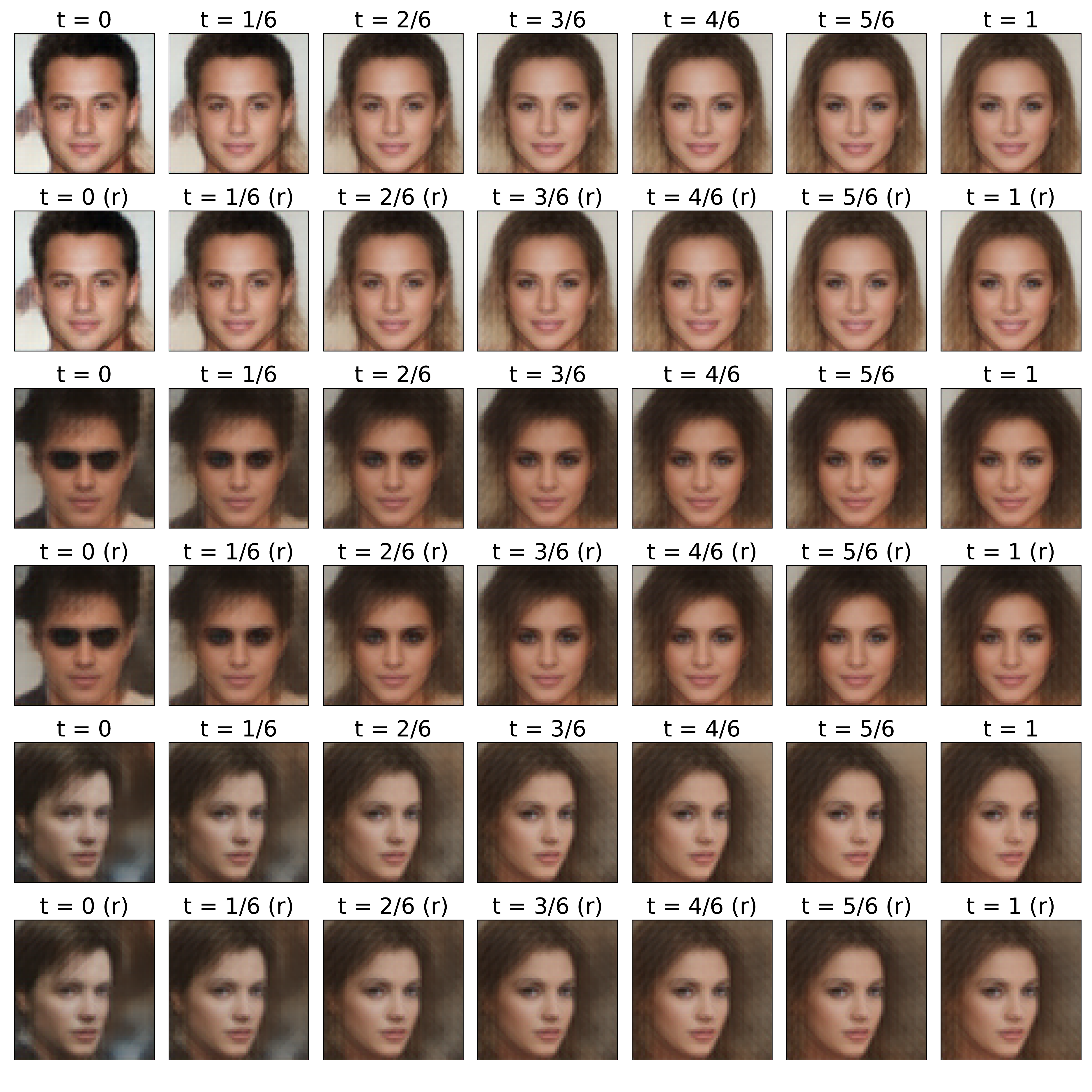}
    \caption{minibatch Wasserstein gradient flow on the CelebA dataset in a DFC-VAE latent space. Source data are 5000 male images while target data are 5000 female images. The batch size m is set to 200 and the number of minibatch k is set to 10. (r) means that the probability law on $m$-tuple is $P^\mathtt{U}$ otherwise it is $P^\mathtt{W}$.}
    \label{fig:GF_VAE_BW}
\end{figure}

The minibatch Wasserstein loss produces blurred images at the end of the flow as shown in Figure \ref{fig:GF_VAE_BW}, especially at the back of the image where all details are lost. This is due to the fact that the minibatch Wasserstein shrinks the distribution. On the contrary,
the debiased minibatch Wasserstein reported in Figure \ref{fig:GF_VAE_UnbiasedBW} produced images with high background details, quality and coherence with respect to the original background. Moreover, the evolution seems more natural between the source and the target distribution.

\begin{figure}[t]
    \centering
    \includegraphics[scale=0.4]{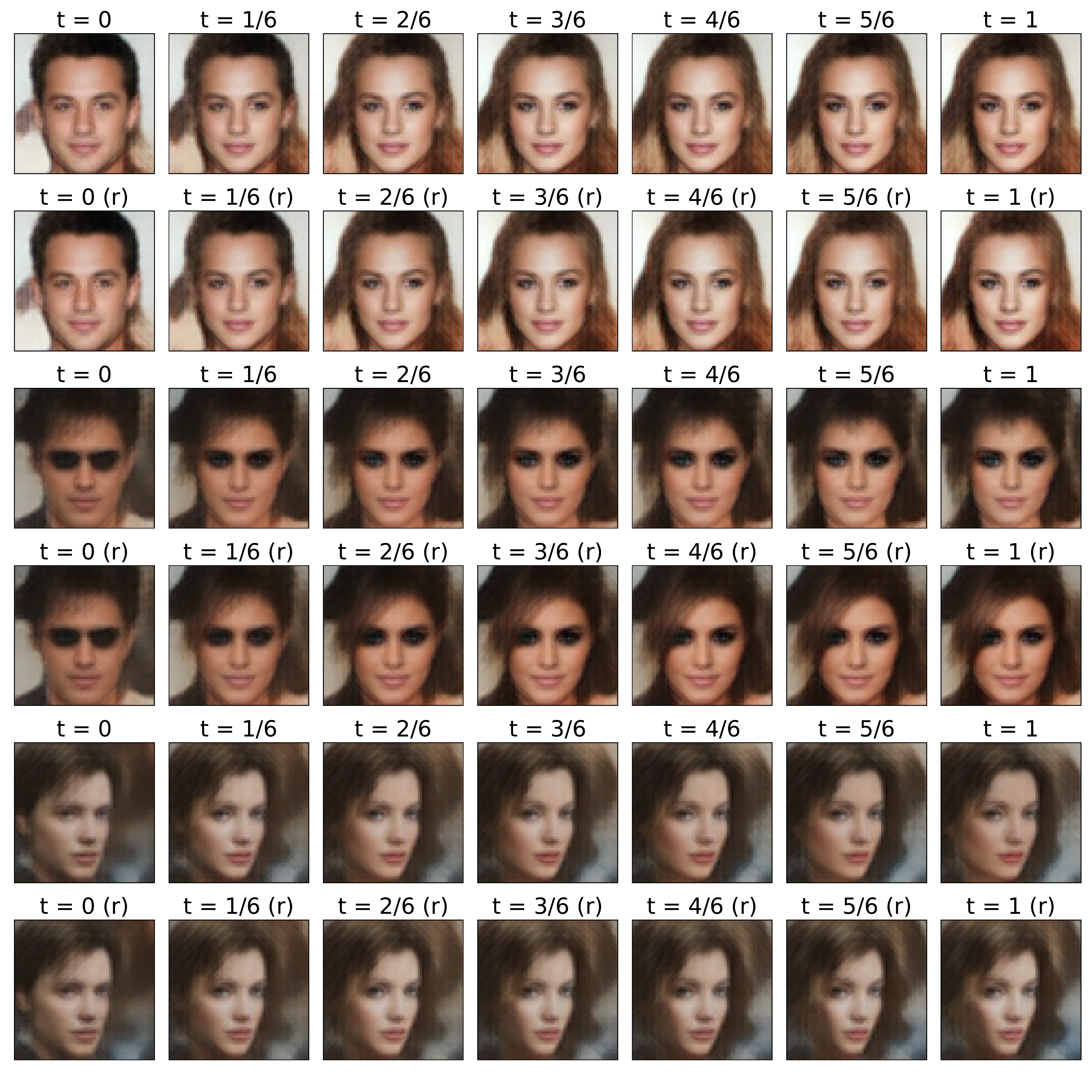}
    \caption{Unbiased minibatch Wasserstein gradient flow on the CelebA dataset in a DFC-VAE latent space. Source data are 5000 male images while target data are 5000 female images. The batch size m is set to 200 and the number of minibatch k is set to 10. (r) means that the probability law on $m$-tuple is $P^\mathtt{U}$ otherwise it is $P^\mathtt{W}$.}
    \label{fig:GF_VAE_UnbiasedBW}
\end{figure}

\subsection{Mapping estimation}
While the previous application focused on updating samples, the second application is a continuous mapping estimation between source and target distributions that will allow transforming new samples that are not in the original training data. The map is parametrized by a neural network $f_{\varphi}: \mathcal{X} \rightarrow \mathbb{R}^{p}$ between the source data and the target data. The objective is to minimize the loss:
\begin{equation}
    \underset{\varphi}{\text{min }}\Lambda_{h,w,P, C(f_{\varphi}(\XX), \YY)}(\uu, \uu),
\end{equation}
Where $f_{\varphi}(\XX) = \{f_{\varphi}(\xx_1), \cdots, f_{\varphi}(\xx_n) \}$. We apply this problem on the celebA dataset \citep{liu2015faceattributes}. 
We considered 5000 male and 5000 female images. The image size is $64 \times 64$. The goal is to learn how to transform a male 
image into a female one. Unfortunately, 
in order to avoid blurry images, we once again relied on
the latent space, of dimension 100, of a pre-trained DFC-VAE \citep{hou2017deep}. We used the same setting as described in the Gradient Flow section. We performed the training in the latent space and then 
we decoded the transform samples. We consider a 4 dense layer neural network with relu activation function ($100\rightarrow 1024 \rightarrow 1024 \rightarrow 512 \rightarrow 100$). The minibatch size $m$ is set to $128$, and we used the Adam optimizer \citep{Kingmaadam14} with a step size of $1e^{-4}$ and the coefficients $\beta_1 = 0$ and $\beta_2 = 0.9$.

We conducted the experiments for 
minibatch Wasserstein loss and for the debiased loss $\Lambda_h$. We spotted once again that the 
transformed samples with the minibatch Wasserstein losses are blurred (figure \ref{fig:monge_vae_BW}). 
However, the results with the unbiased minibatch Wasserstein loss are more diverse and more realistic. It shows the effectiveness of the unbiased loss to debiased the minibatch Wasserstein losses 
(figure \ref{fig:monge_vae_BW}). It is interesting to note that the estimated mapping are quite different on some images between losses which use a sampling with or without replacement.

\begin{figure}[t]
    \centering
    \includegraphics[scale=0.65]{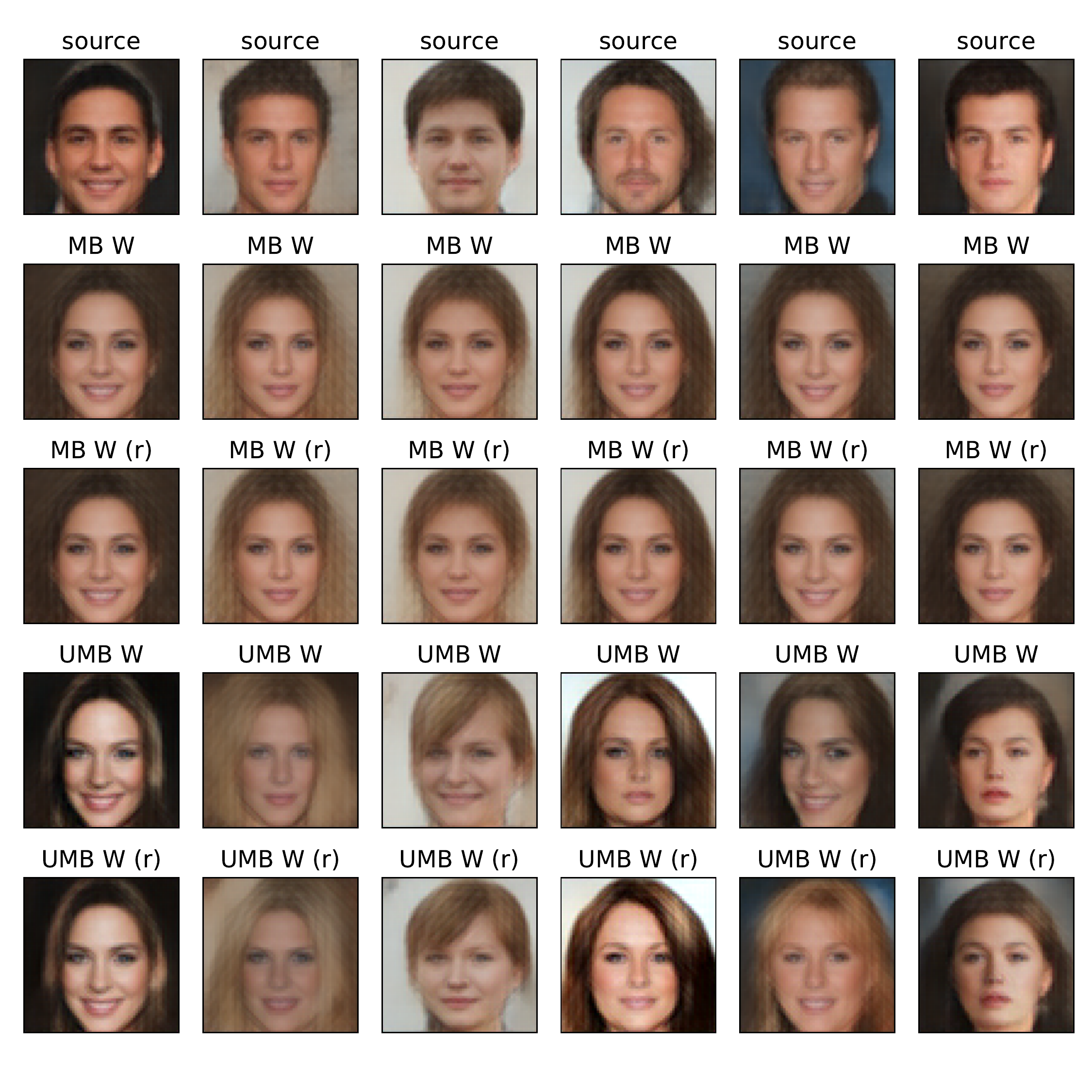}
    \caption{Map learning between 5000 male source images and 5000 female target images. The batch size $m$ is set to 128 and the number of batch couple k is set to 1. (First row) Source data. (Second and third rows) Respectively minibatch Wasserstein without replacement and with replacement (r) mapping on the CelebA dataset in a DFC-VAE latent space. (Fourth and fifth rows) Respectively unbiased minibatch Wasserstein without and with replacement (r) mapping on the CelebA dataset in a DFC-VAE latent space.}
    \label{fig:monge_vae_BW}
\end{figure}
\subsection{Generative Adversarial Networks (GANs) on Cifar10}

Image generation have become a popular machine learning applications with Generative Adversarial 
Networks (GANs) \citep{goodfellow_gan, arjovsky_2017, li2017mmd, genevay_2018, bunne19a} or AutoEncoders (AEs) \citep{KingmaW13, Patrini2019SinkhornA, kolouri2018sliced, tolstikhin2018wasserstein}.
Some state of the art image generation methods have successfully used the Wasserstein distance and its variants as loss functions \citep{arjovsky_2017, Gulrajani2017}.
Recently, \citep{genevay_2018, salimans2018improving} developed Sinkhorn GAN, a GAN variant which uses minibatch Sinkhorn
divergence as a loss function and performed well in practice. Hence, we want to learn a GAN
using our loss function which is a debiased version of minibatch SD. 

The objective of a GAN is to train a neural network $G_{\theta}$ that can generate realistic data which are close to real data $\XX$.
To generate data, the generator $G_{\theta}$ takes a random input $\zz$ in a latent space from $\ZZ \in \mathcal{Z}$.
We want to measure and minimize the distance between the generated data and the real data. 
For the ground cost of the Wasserstein distance, we could rely on an euclidean cost between images.
Unfortunately, using an euclidean cost on high dimensional images generates blurred versions of the real images \citep{Aggarwal01onthe, Wang2005, Kulis2013}. Hence, we will learn adversarially a critic networks $f_{\varphi}$ which extracts meaningful feature vectors for input images. Then we will apply the euclidean distance between the encoded generated data and encoded real data. Other methods relied on a feature extractor such as MMD GAN \citep{li2017mmd} of Sinkhorn GAN \citep{genevay_2018}. We can summarize our learning problem as the following:
\begin{align}
    &\underset{\theta}{\text{min }} \underset{\varphi}{\text{max }} \quad \mathbb{E}\Lambda_{h,w,P, C_{\varphi}(\XX, G_{\theta}(\ZZ))}(\uu, \uu),\\
    \text{where } \quad &C_{\varphi}(\xx_i, \yy_j) \stackrel{\text { def. }}{=}\left\|f_{\varphi}(\xx_i)-f_{\varphi}(\yy_j)\right\|_2 \quad \text { and } \quad f_{\varphi}: \mathcal{X} \rightarrow \mathbb{R}^{p}.\nonumber
\end{align}
Where $G_{\theta}(\ZZ) = \{  G_{\theta}(\zz_1), \cdots, G_{\theta}(\zz_m) \}$ and $\uu \in \mathbb{R}^m$. We train GAN for image generation of CIFAR-10 data \citep{cifar}. The number of data is 50K of size $32 \times 32$. Regarding the implementation detail, we consider the same setting as \citep{li2017mmd, genevay_2018}. The input noise is of dimension 100. The generator and the critic have 4 convolution layers (full detail in tab \ref{tab:gan_arch}). We clip the parameters of the critic in order to have a lipschitz constant bounded by 1 as done in \citep{li2017mmd, genevay_2018}. The batch size $m$ we considered is 64 and we set the number of batch couple to $k=1$ for each SGD update. The optimizer we used is RMSProp \citep{rmsprop} with a learning rate of $5.10^{-4}$. Regarding the entropic regularization parameter for the Sinkhorn divergence, we set it in $\{10, 100, 1000\}$. We update the discriminator 5 times before one update of the generator. 

We compare our method to 4 different methods: WGAN-GP \citep{Gulrajani2017}, Sinkhorn GAN, OT-GAN \citep{salimans2018improving} and MMD GAN. Regarding Sinkhorn GAN we use a batch size of 256 as done in their work. We also compared our GAN to the effective WGAN-GP, we considered the same architecture as above but we used the hyperparameters described in their paper \citep{Gulrajani2017}. Finally for OT-GAN \citep{salimans2018improving}, we used an entropic regularization parameter set to 500 and for fair comparison with other methods, we set the batch size to 256. In their paper authors used batch size of 8000 images to get a more stable training, however this method is not reproducible in our setting with a single GPU. We report Inception scores in Table \ref{tab:inception_cifar}. As we can see, the debiased minibatch Sinkhorn divergence gives the best Inception score showing the relevance of this new loss function. Comparing to the typical Sinkhorn GAN, the debiased strategy increases the inception score by 1 point. Furthermore, it seems that regularizing the problem with the entropic regularization helps to get better performance as already suggested in previous work \cite{genevay_2018}. We also report in Figure \ref{fig:gan_examples} some generated examples from MBSD, UMBSD and WGAN models and we can see that UMBSD lead to slightly more detailed samples than MBSD and more realistic than WGAN-GP.

\begin{table}
    \centering
    \begin{tabular}{|c|c|}
     \hline
        Generator & Critic \\
        \hline
        INPUTS: 100  & INPUTS: $3 \times 32 \times 32$ \\
                Conv2D$^\intercal$ nc=256 k=4 stride=1, BN, ReLU  & Conv2D nc=64 k=4 stride=2, LReLU(slope=0.2) \\
                Conv2D$^\intercal$ nc=128 k=4 stride=2, BN, ReLU  & Conv2D nc=128 k=4 stride=2, LReLU(slope=0.2) \\
                Conv2D$^\intercal$ nc=64 k=4 stride=2, BN, ReLU  & Conv2D nc=256 k=4 stride=2, LReLU(slope=0.2) \\
                Conv2D$^\intercal$ nc=3 k=4 stride=2, TanH  & Conv2D nc=100 k=4 stride=1, LReLU(slope=0.2) \\
            \hline
    \end{tabular}
    
    \caption{Generator (left) and critic (right) 4 convolutional layer architectures used in our experiments to generate CIFAR10 data.}
    \label{tab:gan_arch}
\end{table}

\begin{table}[t!]
    \centering
    \begin{tabular}{|c|c|}
        \hline
        Methods & Inception score\\
        \hline
        WGAN-GP & $4.59 \pm 0.07$\\
        \hline
        MBSD ($\varepsilon=10$) & $3.57 \pm 0.03$\\
        \hline
        MBSD ($\varepsilon=100$) & $3.61 \pm 0.05$\\
        \hline
        MBSD ($\varepsilon=1000$) & $3.83 \pm 0.05$\\
        \hline
        OT-GAN ($\varepsilon=500$) & $4.13 \pm 0.08$\\
        \hline
        MMD & $4.29 \pm 0.06$\\
        \hline
        \hline
        UMBW (ours) & $4.38 \pm 0.08$\\
        \hline
        UMBSD ($\varepsilon=10$) (ours) & $4.72 \pm 0.07$\\
        \hline
        UMBSD ($\varepsilon=100$) (ours) & $\textbf{4.76} \pm 0.08$\\
        \hline
        UMBSD ($\varepsilon=1000$) (ours) & $4.67 \pm 0.08$\\
        \hline
      \end{tabular}
    \caption{Inception Scores on CIFAR10 for several GAN variants trained with a batch size of 64. Biggest score is in bold.}
    \label{tab:inception_cifar}
\end{table}

\begin{figure}[!t]
    \centering
    \includegraphics[scale=0.3]{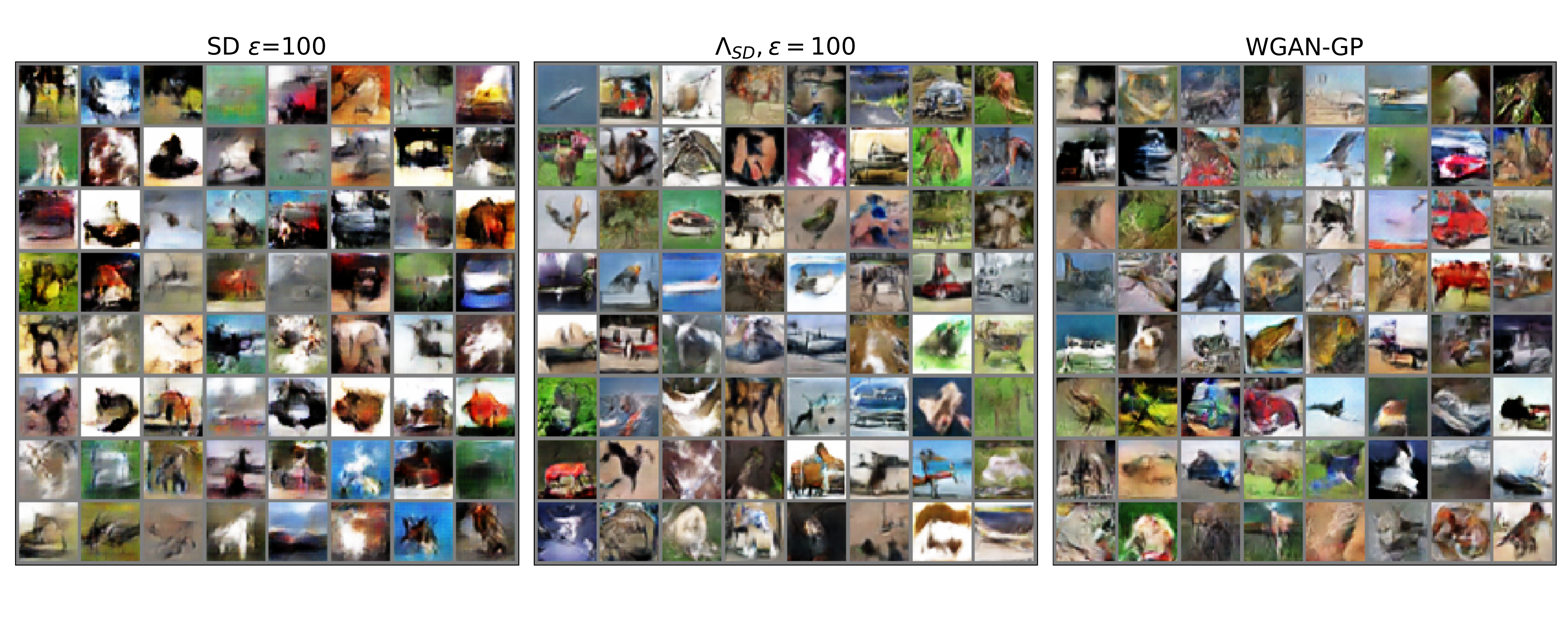}
    \caption{Generated samples with the same latent vectors from different GANs.}
    \label{fig:gan_examples}
\end{figure}

\subsection{Large scale barycentric mapping for color transfer}
\begin{figure*}[!t]
    \centering
    \includegraphics[scale=0.315]{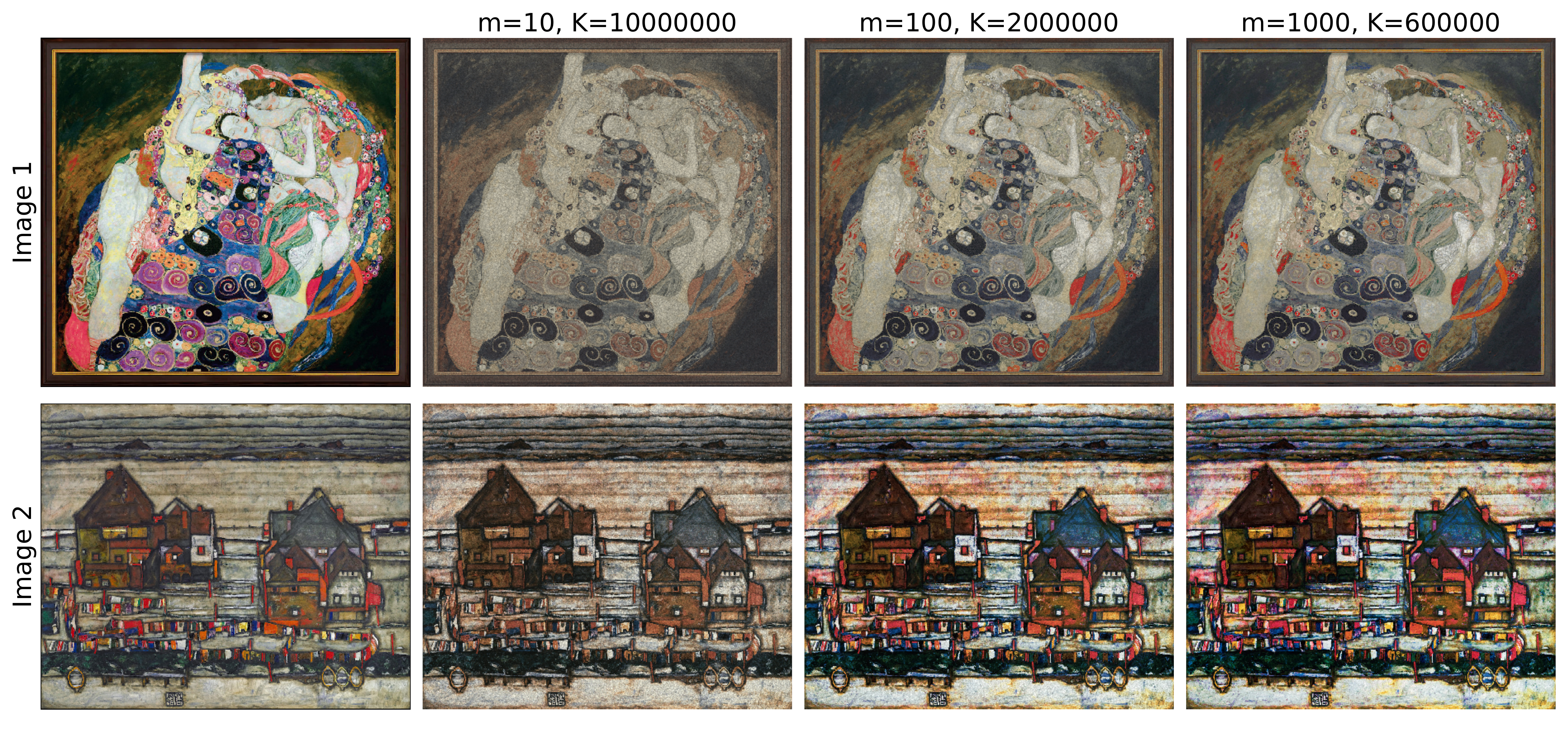}\vspace{-5mm}
    \caption{Color transfer between full images for different batch size and number of batches. (Top) color transfert from image 1 to image 2. (Bottom) color transfer from image 2 to image 1.}
    \label{fig:CT_full_img}
\end{figure*}

The purpose of color transfer is to transform the color of a source image so that it follows the color of a target image. Optimal Transport is a well known method to solve this problem and has been studied before in (\cite{Ferradans2013, blondel2018}). Images are represented by point clouds in the RGB color space. Then by calculating the transport plan between the two point clouds, we get a transfer color mapping by using a barycentric projection. As the number of pixels might be huge, previous work selected a subset of pixels using k-means clusters for each point cloud. This strategy allows to make the problem memory tractable but looses some information to the quantification. With MB optimal transport, we can compute a barycentric mapping for all pixels in the image by incrementally updating the full transported vector at each minibtach. 
When one selects a source indices m-tuple $I_1$ and a target m-tuple $I_2$, she just needs to update the transported vector between the considered minibatches as $Q_{I_1} Y_s = \sum_{I_2}  \Pi_{  I_1, I_2 }^m Q_{I_2} X_t$, with matrix $Q_{I_1}$ and $Q_{I_2}$ defined as in definition \ref{def:MBTP}. Indeed, the incremental computation can be rewritten as:
\begin{equation}
    Y_s = n_s \widetilde{\Pi}^{W_2^2,k}_\mathtt{W}(\uu, \uu) X_t,
\end{equation}
when we use the incomplete MBOT plan $\widetilde{\Pi}^{W_2^2,k}_\mathtt{W}$. To the best of our knowledge, it is the first time that a barycentric mapping algorithm has been scaled up to 1M pixel images. About the required memory for experiments, the memory cost to store data is $O(n)$. The minibatch OT calculus requires $O(m^2)$ because we need to store the ground cost and the OT plan. The marginal experiment requires $O(n)$, as we just need to average the marginals of the plan. Finally, the memory cost is $O(n)$ while exact OT would be $O(n^2)$.

The source image has (943000, 3) RGB dimension and the target image has RGB dimension (933314, 3). For this experiments, we compare the results between the minibatch framework with the Wasserstein distance for several m and k. We used batch of size 10, 100 and 1000. We selected $k$ so as to obtain  a good visual quality and observed that a smaller $k$ was needed when using large minibatches. Also note that performing MB optimal transport can be done in parallel and can be greatly speed-up on multi-CPU architectures.
 One can see in Figure \ref{fig:CT_full_img} the color transfer (in both directions) provided with our method. We can see that the diversity of colors falls when the batch size is too small as the entropic solver would do for a large regularization parameter. However, even for 1M pixels, a batch size of 1000 is enough to keep a good diversity of colors. 
\newline

From now on for speed constraints, we consider a selected subset of 1000 pixels using k-means clusters for each point cloud. We reproduced empirically the results of Theorem \ref{thm:dist_marg} about the marginal errors, as shown in Figure \ref{fig:marg_sparsity} we recover the $O(k^{-1/2})$ convergence rate on the marginal with a constant depending on the batch size $m$.  

\begin{figure}[t]
    \centering
    \includegraphics[scale=0.5]{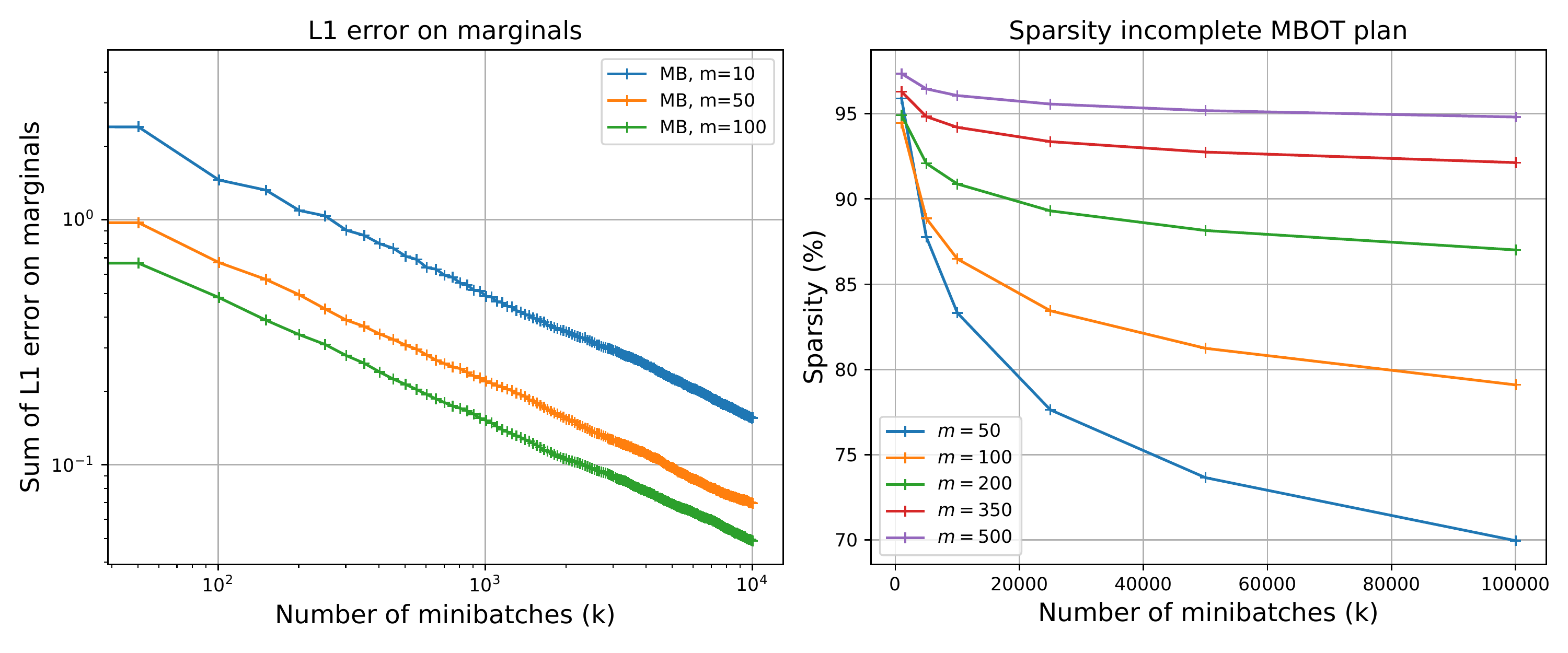}
    \caption{(left) L1 error on both marginals (loglog scale). We selected 1000 points from original images and computed the error on marginals for several m and k (loglog scale). (Right) Sparsity of incomplete minibatch OT plan $\widetilde{\Pi}^{W_2^2,k}_\mathtt{W}(\uu, \uu)$. We selected 1000 points from original images and computed the sparsity of $\widetilde{\Pi}^{W_2^2,k}_\mathtt{W}(\uu, \uu)$ for several $k$ and $m$.}
    \label{fig:marg_sparsity}
\end{figure}
As we stated above, minibatch Wasserstein loss increases the number of connection similarly to regularized OT variants. Hence, we want to conduct a sparsity experiment of the minibatch Wasserstein transport plan and we report it for several settings. We considered batch sizes of  50, 100, 200, 350 and 500 and computed the sparsity of the incomplete minibatch OT plan with respect to several number of minibatches $k$. The results are gathered in figure \ref{fig:marg_sparsity}. We see that as $m$ gets smaller, the degree of sparsity decreases and that the sparsity reaches a limit as the number of minibatches increases. Intuitively, it is expected as when $m$ gets smaller, the number of connections increases. The results can be justified with the following facts. When the minibatch size between the source and target batches is the same and with uniform weights, then $m$ coefficients of the transport matrix will be non null for the exact Wasserstein distance. 
As we draw $k$ batch couples, such as $k.m < n$ and if we suppose that the batches define a disjoint union of the samples $\XX$ and $\YY$, then we have at most $km$ coefficients of $\widetilde{\Pi}^{W_2^2,k}_\mathtt{W}(\uu, \uu)$ non zero. In the case of non uniform weights $\aa, \bb$, the positive linear program has a solution with at most $2m-1$ non zero coefficients. Then we have at most $k.(2m-1)$ coefficients of $\widetilde{\Pi}^{W_2^2,k}_\mathtt{W}(\aa, \bb)$ non zero.

\subsection{Minibatch Gromov-Wasserstein rotation and translation invariance}


The Gromov-Wasserstein distance has the nice properties to be rotational and translation invariant, so in this section we study if the minibatch Gromov-Wasserstein loss (MBGW) shares the same properties. To the best of our knowledge, it is the first time that minibatch Gromov-Wasserstein loss properties have been investigated theoretically and empirically. As shown in the previous section, our statistical results can be extended to the Gromov-Wasserstein distance. We start with a spiral experiment where we compute the value of the MBGW loss for several rotations of the spirals. Then, we aim at checking if the MBGW loss is able to recover the motion of a galloping horse on a dataset containing a sequence of shapes \citep{solomon2016}. 
\newline

{\bfseries Rotational invariance.} 
Our first result shows the stability of rotation and translation invariances with minibatches. We have the following results:
\begin{proposition}[Invariance]
  The minibatch Gromov-Wasserstein is rotation and translation invariant.
  \label{prop:invariance_rot_gw}

\begin{proof}
Let $\aa$ and $\bb$ be two probability vectors with support $\XX$ and $\YY$ respectively. Consider now the support $\YY^\prime$ which is a rotation and a translation of $\YY$. Consider three ground costs $C^1 = C(\XX, \XX)$, $C^2 = C(\YY, \YY)$ and $C^3 = C(\YY^\prime, \YY^\prime)$. For fixed minibatches $I$ and $J$, as $\YY^\prime$ is a translation and rotations of $\YY$, we have: $$\GW \Big( w_1(\mathbf{a},I), w_2(\mathbf{b},J), C^1_{I,I}, C^2_{J,J}\Big) = \GW \Big( w_1(\mathbf{a},I), w_2(\mathbf{b},J), C^1_{I,I}, C^3_{J,J}\Big),$$ summing over all minibatch couples finishes the proof.

\end{proof}
\end{proposition}
Empirically, distances which are rotation invariant return a constant when comparing rotated distributions. To support the proposition, we consider a small spiral experiment for different rotations of the target distribution. We follow the procedure in \citep{vayersgw}. The source and the target distributions are spirals taken from the scikit-learn spiral dataset \citep{pedregosa2011scikit}. We compute Gromov-Wasserstein distance and the MBGW loss on $n = 300$ samples. We report in Figure \ref{fig:spiral} the average values of the GW and MBGW losses for a varying angle and we can see that it is in practice invariant to rotation. From the figure \ref{fig:spiral}, one recovers that the MBGW loss returns a constant, depending on the minibatch size $m$, and hence is rotation invariant.
\newline

\begin{figure*}[t]
    \centering
    \includegraphics[scale=0.6]{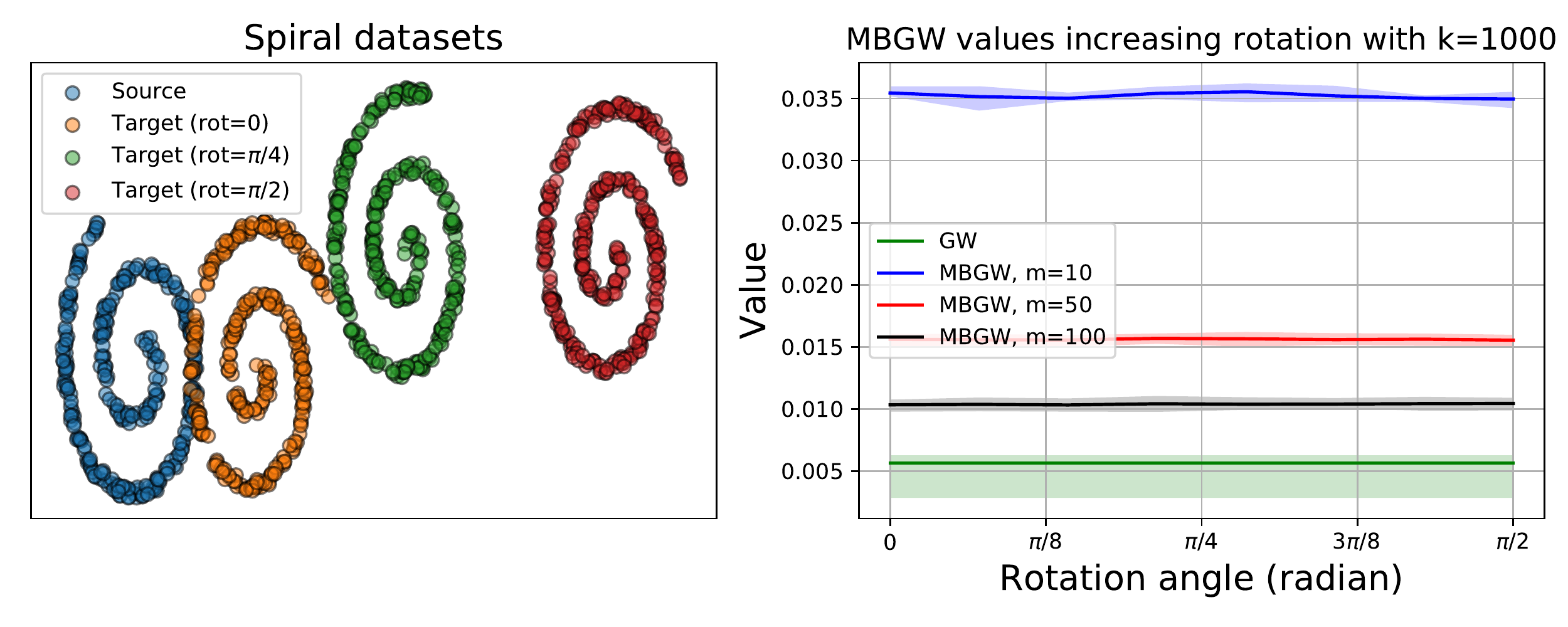}
    \caption{Average value of MBGW and GW losses as a function of rotation angle on 2D spirals. Colored areas correspond to the 20\% and 80\% percentiles. Experiments were run 10 times.}
    \label{fig:spiral}
\end{figure*}

{\bfseries Meshes comparison} In the context of computer graphics, Gromov-Wasserstein distance can be used to measure similarities between two meshes \citep{peyre16, Solomon2015, vayersgw}. It can also be used for shape matching, search, exploration or organization of databases. As minibatch GW loss and its debiased counter parts are not distances, we want to know if they are meaningful for use in a context of meshes comparison. From a time series of 45 meshes representing the motion of a galloping horse, we compute a multidimensional scaling (MDS) of the pairwise distances with minibatch GW losses, that allows plotting each mesh as a 2D point. Each horse mesh is composed of approximately 9, 000 vertices. The results can be found in figure \ref{fig:horse_mds}. As one can observe in figure \ref{fig:horse_mds}, the cyclical nature of this motion is successfully recovered in this 2D plot for both MBGW loss and its debiased counter parts.
\newline

\begin{figure*}[t]
    \centering
    \includegraphics[scale=0.5]{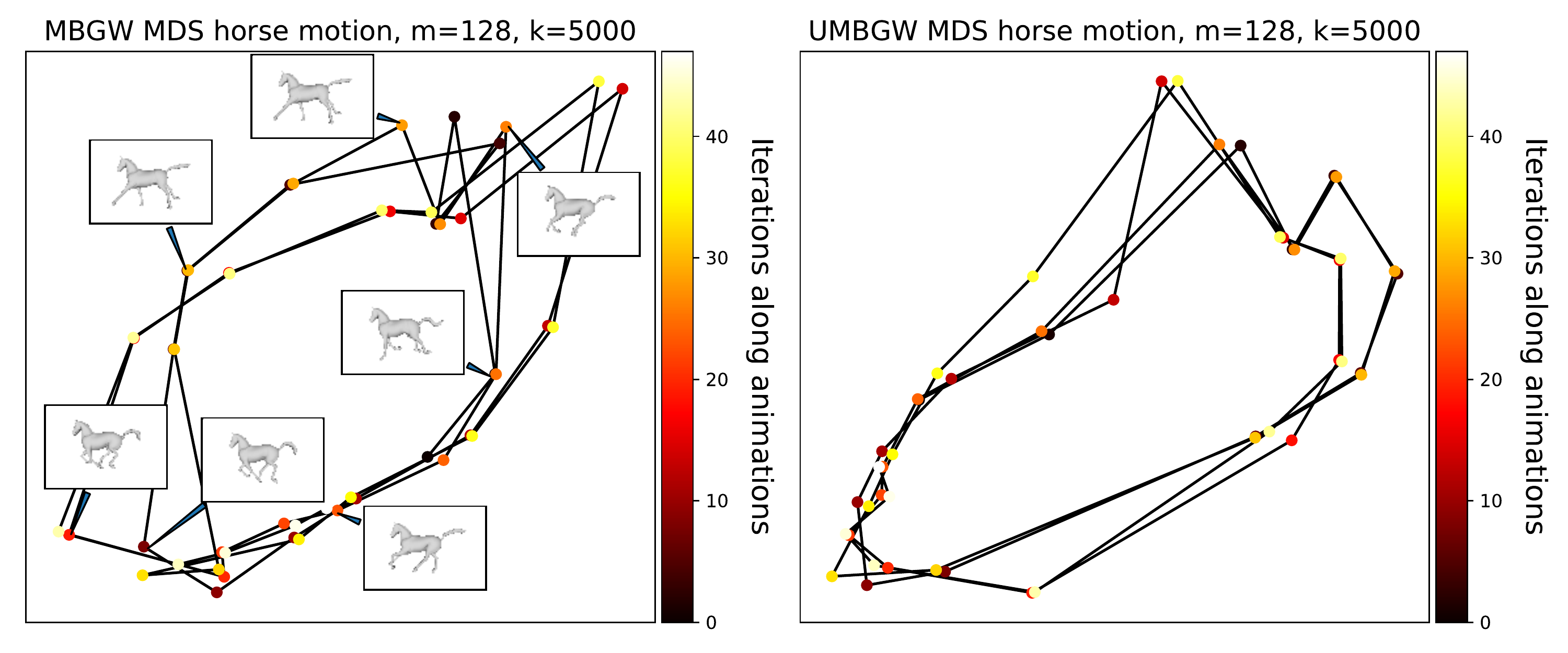}
    \caption{MDS on the galloping horse animation with MB Gromov-Wasserstein loss and its debiased variant. Each sample in this Figure corresponds to a mesh and is colored by the corresponding time iteration. One can see that the cyclical nature of the motion is recovered. }
    \label{fig:horse_mds}
\end{figure*}

{\bfseries Running time comparison} Our last experiment is the time computation of minibatch Gromov-Wasserstein. We compare it to Gromov-Wasserstein distance, the entropic regularized Gromov-Wasserstein, the Sliced Gromov-Wasserstein and its rotational invariant variant \citep{vayersgw}. Unfortunately, the Sliced variant can only be computed for square euclidean ground cost unlike the MBGW and is not rotational invariant. 
We calculate these distances between two 100-D random measures of $n \in {10^2, ..., 10^4}$ points. For the minibatch Gromov-Wasserstein we consider two settings. The first setting is with a fixed number of minibatch couples ($k=5000$) and the second is linear setting where $k$ grows linearly according to $n$ ($k=\frac{n}{10}$). The latter is due to our concentrations bounds which decreases linearly in the number of samples if we consider a number of minibatch couples proportional to the number of samples (see Theorem \ref{thm:inc_U_to_mean}).
We use the Python Optimal Transport (POT) toolbox to compute GW distance on CPU. For entropic-GW we use the POT implementation with a regularization parameter of $\varepsilon = 0,01$. We were not able to get converged transport plan for a bigger number of data than $10^4$ for both GW and its entropic variant.

\begin{figure*}[t]
  \centering
  \includegraphics[scale=0.4]{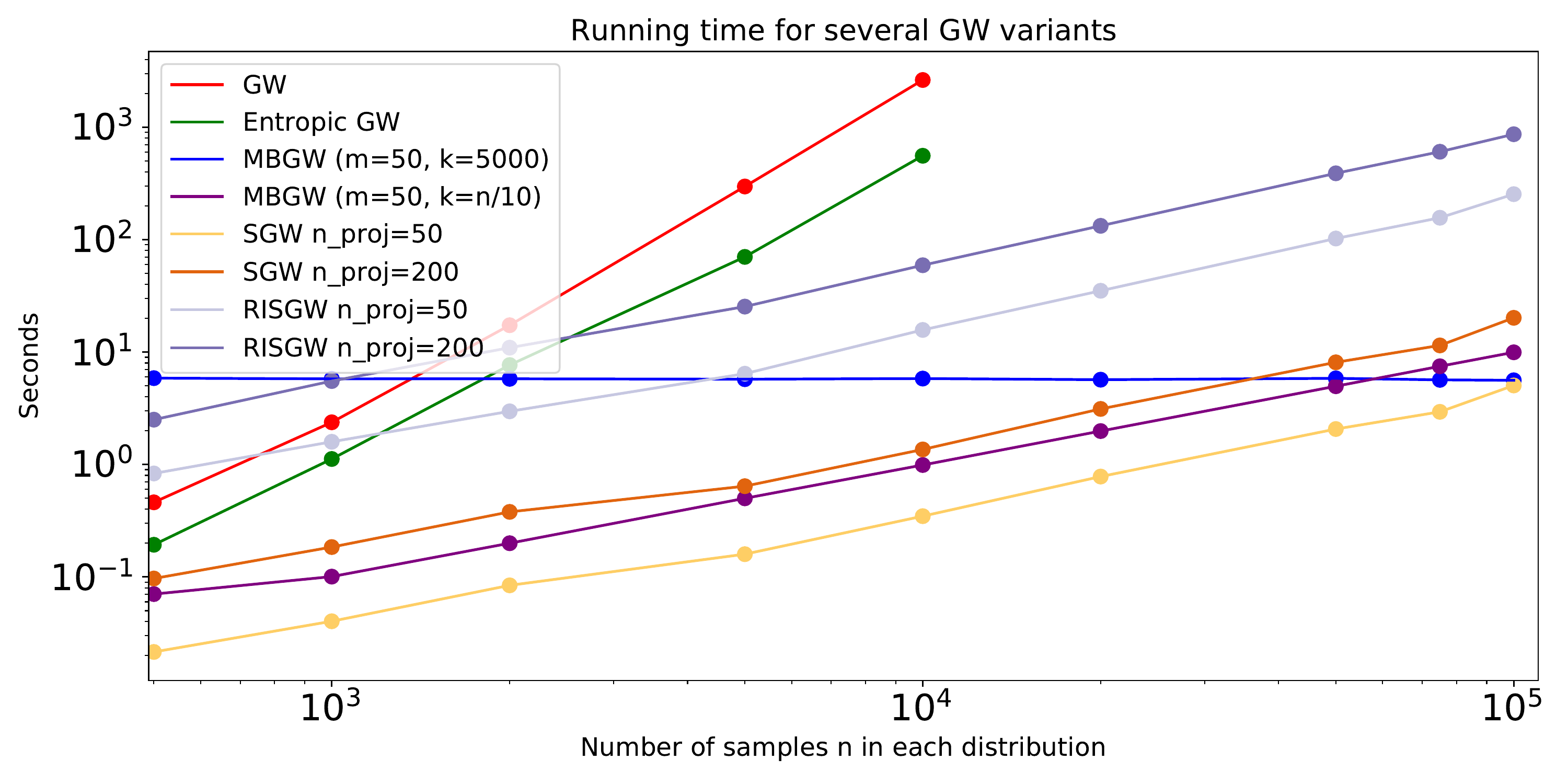}
  \caption{Runtimes comparison between SGW, GW, entropic-GW between two 100-D random distributions with varying number of points from 0 to $10^4$ in log-log scale. The time includes the calculation of the pair-to-pair distances.}
  \label{fig:running_times}
\end{figure*}

We see that MBGW enjoys a constant time computation. The sliced Gromov-Wasserstein and its rotational variant grow in $\mathcal{O}(n log(n))$ making it slower than the minibatch GW for large scale dataset. Regarding GW and its entropic counter part, we see that for $10e^4$ points, the MBGW is 100 time faster than GW.

\section{Conclusion}

In this paper, we extended the impact study of using a minibatch strategy with a Wasserstein distance (\cite{fatras2019batchwass}). We alleviate the hypothesis by considering unbounded and non uniform probability distributions. We defined several estimators based on different sampling strategies, reviewed their basic properties, proposed a new loss function which fixes the loss of the separability distance axiom, and studied the asymptotic behavior of our estimators. We showed a deviation bound between our subsampled estimators and their expectation.
Furthermore, we studied the optimization procedure of our estimator and proved that it enjoys unbiased gradients for all OT kernels unlike the Wasserstein distance. Finally, we demonstrated the effect of minibatch strategy with gradient flow experiments, color transfer, GAN, map learning and Gromov-Wasserstein experiments.

\section*{Acknowledgements}

Authors would like to thank Thibault Séjourné for fruitful discussions. This work is partially funded through the projects OATMIL ANR-17-CE23-0012 and 3IA Côte d'Azur Investments ANR-19-P3IA-0002 of the French National Research Agency (ANR). Y.Z. was supported by the European Research Council (grant no. 864138 “SingStochDispDyn”). Y.Z. would like to thank the School of Mathematics at the University of Edinburgh for its hospitality during the preparation of this manuscript.

\newpage

\nocite{*}
\bibliography{egbib}

\newpage

\begin{table}[t!]
        \begin{center}
            \begin{tabular}{ |c|c|c| } 
                 \hline
                 Notations & Description & Example \\
                 \hline 
                 $\xx$ & vector $\in \mathbb{R}^d$ & $\xx=[1, 2, 3]$\\
                 $n$ & number of data & $n$=6\\
                 $m$ & minibatch size & $m=4 \leq n$\\
                 $I$ & Index $m$-tuple & $(1, 1, 2, 1)$ \\ 
                 $\llbracket n \rrbracket^m$ & Set of all index $m$-tuples & $\{ I_1, I_2, \cdots, \}$\\ 
                 $\Pim$ & Set of all index $m$-tuples without replacement & $\{ I_1, I_2, \cdots, \}$ \\
                 $\mathcal{P}^{m,o}$ & Set of all ordered index $m$-tuples without replacement & $\{ I_1, I_2, \cdots, \}$ \\
                 $\XX(I)$ & data $m$-tuple & $(\xx_1, \xx_1, \xx_2, \xx_1)$ \\ 
                 $\XX$ & data $n$-tuple & $(\xx_1, \cdots, \xx_n)$ \\ 
                $\sum_{i \in I} f(i)$ & Sum over all elements of tuple $I$ & $\sum_{k=1}^{m}f(i_k)$\\
                $\Pi_{i \in I} f(i)$ & Product over all elements of tuple $I$ & $\Pi_{k=1}^{m}f(i_k)$\\
                $\simplex_{n}$ & Simplex of size $n$ &$\{\mathbf{a} \in \R_+^n, \sum_{i=1}^n a_i = 1\}$\\
                $\aa \in \simplex_n$ & probability vector & $\sum_{i=1}^n a_i = 1$ \\
                $\uu \in \simplex_n$ & uniform probability vector & $\sum_{i=1}^n \frac1n = 1$ \\
                $\simplex$ & Set of all sequences of probability vectors & $(\aa^{(n)})$\\
                $\alpha$ & probability distribution & $\mathcal{N}(0,1)$\\
                $\alpha^{\otimes m }$ & $m$-tuples drawn from $\alpha$ & $\XX \sim \alpha^{\otimes n}$\\
                $w$ & Reweighting function & $w([\frac{1}{2},\frac{1}{6},\frac{1}{3}], (1, 2)) = [\frac{2}{3}, \frac{1}{3}]$\\
                $P$ & Probability law to draw 
                index $m$-tuples & $P(I) = n^{-m}$\\ 
                $p$ & Power of the Wasserstein distance & $W_{p=2}$\\
                $\varepsilon$ & entropic regularization coefficient & $W_{\varepsilon=0.1}$\\
                $h$ & OT kernel & $h = W, W_\varepsilon, S_\varepsilon, \GW$\\
                $C^m$ & Ground cost matrix of size $n$ and $m$ & euclidean distance\\
                $\mathcal{M}_m(\R) $ & Set of square (real) matrices of size $m$ & $\mathcal{M}_4(\R) $\\
                $\overline{h}_{w, P}$ & Minibatch kernel OT loss & $\overline{W}_{p,w,P}$\\
                $\widetilde{h}_{w,P}^k$ & Incomplete MBOT loss & $\widetilde{W}_{w^\mathtt{U},P^\mathtt{U}}^k$\\
                $\Lambda_{h, w, P}$ & Debiased minibatch loss & $\Lambda_{h=W_2^2}$\\
                $\widetilde{\Lambda}_{h,w,P, C(\XX, \YY)}^k$ & Incomplete debiased MBOT loss & $\widetilde{\Lambda}_{W,w,P}^k$\\
                $\overline{\Pi}^h_{w,P}$ & MBOT plan & $\overline{\Pi}^h_{W,P}$ \\
                $\widetilde{\Pi}^{h,k}_{w,P}$ & Incomplete MBOT plan & $\widetilde{\Pi}^{W_{\varepsilon},k}_{w,P}$\\
                $\overline{h}^\mathtt{U}$ & MBOT loss (sampling without replacement) & $\overline{W}^\mathtt{U}$\\
                $\overline{h}^\mathtt{W}$ & MBOT loss (sampling with replacement) & $\overline{W}^\mathtt{W}$\\
                $\overline{\Pi}^h_\mathtt{U}$ & OT plan (sampling without replacement) & $\overline{\Pi}^{S_\varepsilon}_\mathtt{W}$\\
                $\overline{\Pi}^h_\mathtt{W}$ & OT plan (sampling with replacement) & $\overline{\Pi}^W_\mathtt{U}$ \\
                $\mathtt{Loc_A}$ & local product constraint & $\mathtt{Loc_A}(m, \gamma, D)$\\
                $\mathtt{Loc_G}$ & local product constraint & $\mathtt{Loc_G}(m, 1, 1)$\\
                $D, \gamma$ & minibatch local constraints & $D=1, \gamma=1$\\
                $\operatorname{nSG}(\rho, \sigma^2 )$ & space of subgaussian random variables & $X \in \operatorname{nSG}(\rho, \sigma^2 )$\\
                 \hline
            \end{tabular}
        \end{center}
    \caption{Table of used notations. (Left) notations, (middle) descriptions, (right) examples.}
    \label{app:summary_def_tab}
\end{table}

\section{Appendix}

{\bfseries Outline.}
The supplementary material of this paper is organized as follows:
\begin{itemize}
    \item Appendix \ref{app_sec:notations} provides a table of all used notations.
    \item Appendix \ref{app_sec:formalism} gives the proofs of our general minibatch OT distances formalism. In particular, it proves under what conditions the minibatch OT matrix is a minibatch OT plan.
    \item Appendix \ref{app_sec:concentration_bounded} provides the concentration bounds proofs for compactly supported distributions. We generalize the U-statistic proof to know under what conditions our estimator is close to its mean.
    \item Appendix \ref{app_sec:subgauss_concentration_bounded} provides the concentration bound proofs for subgaussian distributions. Based on the compactly supported case, we use a truncation argument to provide a more general concentration bound for unbounded distributions.
    \item Appendix \ref{app_sec:distance_marginals} gives the concentration bound proofs for the minibatch OT plan. We provide a concentration bound of our incomplete minibatch OT plan around the input marginals.
    \item Appendix \ref{app_sec:optimization} details the optimization proofs. We prove that we can exchange (sub-)gradient over parameters and expectations, which justifies the use of SGD for optimization. Notably, this proof includes the Wasserstein and the Gromov-Wasserstein distances.
    \item Appendix \ref{app_sec:1D_case} discusses the 1D case. We detail the calculus of the 1D minibatch Wasserstein distance close form.
    \item Appendix \ref{app_sec:contributions} discusses the authors contributions.
\end{itemize}

\subsection{Notations} \label{app_sec:notations}

We gather all our notations in the table \ref{app:summary_def_tab}.

\subsection{Formalism}\label{app_sec:formalism}

In this appendix, we show results which justify our formalism and then, we show how we can upper bounded our minibatch OT loss. We first show that Example \ref{def:law_indices_rep2} defines a probability law on $m$-tuples without replacement.
\begin{customexample}{\ref{def:law_indices_rep2}}[Drawing indices ``without replacement'']
 
Given a discrete probability distribution $\mathbf{a} \in \simplex_n$, it is also possible to draw distinct indices $i_\ell \in \llbracket n \rrbracket$, $1 \leq \ell \leq m$, by defining $P_\mathbf{a}^\mathtt{W}(I)=0$ if the $m$-tuple $I$ \emph{has} repeated indices, otherwise
\begin{equation}\label{app : samp_rep}
P_{\mathbf{a}}^\mathtt{W}(I) = 
\frac{1}{m}
\frac{(n-m)!}{(n-1)!}
\sum_{i \in I} a_i.
\end{equation}
Denote by $\Pim$ the set of all $m$-tuples without repeated elements. Let us check, that equation \eqref{app : samp_rep} defines a probability distribution on $\Pim$.
Observe that $\sum_{i=1}^n a_i=1$ and that for each $1 \leq i \leq n$
\begin{align}
\sharp\{I \in \Pim: i \in I\}&=\sharp\{I \in \Pim: n \in I\}\nonumber \\
& = \sharp \{ I=(i_{1},\ldots,i_{m}) \in \Pim: i_{1}=n\} + \ldots + \sharp \{ I = (i_{1},\ldots,i_{m}) \in \Pim: i_{m}=n\}\nonumber \\ 
&= m \cdot \sharp \{ I = (i_{1},\ldots,i_{m}) \in \Pim: i_{m}=n\}.\label{app:num_i_all_tup}
\end{align}
Since $\sharp \{ I = (i_{1},\ldots,i_{m}) \in \Pim: i_{m}=n\}$ is the number of $(m-1)$-tuples without repeated indices of $\llbracket 1,n-1\rrbracket$, $(n-1)!/(n-m)!$, it follows that
\begin{equation}
\frac{m (n-1)!}{(n-m)!} \cdot \sum_{I \in \Pim} P_{\mathbf{a}}^{\mathtt{W}}(I)
=
\sum_{I \in \Pim} \sum_{i \in I} a_i 
=
\sum_{i=1}^n a_i \cdot \sharp\{I \in \Pim: i \in I\} 
= m \cdot \frac{(n-1)!}{(n-m)!}.
\end{equation}
This shows that $\sum_{I \in \Pim} P_{\mathbf{a}}^{\mathtt{W}}(I)=1$.
\end{customexample}

We now prove that if Equation \eqref{eq:wPadmissible} is respected, then the minibatch OT matrix defines a transport plan. We also prove that for a $W_p^p$ kernel, minibatch $W_p^p$, \emph{i.e.,} $\overline{W_p^p}$, is an upper bound of $W_p^p$.
\begin{customprop}{\ref{prop:admi_plan}}
If the reweighting function $w$ and the parametric distribution on $m$-tuples $P_\mathbf{c}$ satisfy the following admissibility condition 
   \begin{equation}
        \expect_{I \sim P_{\mathbf{c}}}  Q_I^\top w(\mathbf{c},I) = \mathbf{c},\qquad \forall \mathbf{c} \in \simplex_n\label{app:wPadmissible}
    \end{equation}
Then with the notations of Definition~\ref{def:AVG_OT_plan}, the averaged minibatch transport matrix $\overline{\Pi}^h_{w,P}$ is an admissible transport plan between the discrete probabilities $\mathbf{a},\mathbf{b} \in \simplex_n$ in the sense that $\overline{\Pi}^h_{w,P} \mathbf{1}_n = \mathbf{a}$ and $\mathbf{1}_n^\top \overline{\Pi}^h_{w,P} = \mathbf{b}^\top$.
 Considering the Wasserstein kernel $h = W_p^p$, the minibatch loss defined in~\eqref{def:minibatch_wasserstein}, as the associated coupling $\overline{\Pi}^h_{w,P}$ is not the optimal coupling of the full OT problem, it satisfies
\begin{equation}\label{app:mblowerbound}
    \overline{h}_{w,P}(\mathbf{a}, \mathbf{b}) = \langle \overline{\Pi}^{h}_{w,P}, C \rangle_F \geq h(\mathbf{a}, \mathbf{b}).
\end{equation}
\end{customprop}

Under assumption~\eqref{app:wPadmissible} one can safely call $\overline{\Pi}^h_{w,P}(\mathbf{a}, \mathbf{b})$ an averaged minibatch transport \emph{plan}.

\begin{proof}
By the definition of $\Pi^m_h$ and the properties $Q_J \mathbf{1}_n = \mathbf{1}_m$ and $\Pi^{m}_{I,J}\mathbf{1}_m = \mathbf{a}_I = w(\mathbf{a},I)$ we have
\begin{align*}
    \overline{\Pi}^h_{w,P} \mathbf{1}_n
    &= \expect_{I, J} Q_I^\top \Pi_{I,J}^{m} Q_J \mathbf{1}_n
    = \expect_{I, J} Q_I^\top \Pi_{I,J}^{m} \mathbf{1}_m
    = \expect_{I, J} Q_I^\top \mathbf{a}_I
    = \expect_{I} Q_I^\top \mathbf{a}_I
    = \expect_{I \sim P_\mathbf{a}} Q_I^\top w(\mathbf{a},I) = \mathbf{a}.
\end{align*}
The proof that $\mathbf{1}_n^\top \overline{\Pi}^h_{w,P} = \mathbf{b}^\top$ is similar.
This establishes that $\overline{\Pi}^h_{w,P} \in U(\mathbf{a},\mathbf{b})$ is an admissible transport plan between the discrete probabilities $\mathbf{a}$ and $\mathbf{b}$. 

We now prove~\eqref{app:mblowerbound} for the Wasserstein distance $h = W_p^p$ ($\epsilon=0$, $1 \leq p < \infty$). Since $\overline{\Pi}^h_{w,P}$ is an admissible transport plan we have:
\begin{align*}
        h(\mathbf{a}, \mathbf{b}) = \underset{ \Pi \in U( \mathbf{a}, \mathbf{b} ) }{\operatorname{min} } \langle \Pi, C \rangle 
        &\leq \langle \overline{\Pi}^h_{w,P}, C \rangle
\end{align*}
Further, by definition of the average minibatch transport plan $\overline{\Pi}^h_{w,P}$, and observing that the matrices $Q_I,Q_J$ from Definition~\ref{def:MBTP} are such that $C_{I,J} = Q_I C Q_J^\top$,
we obtain
\begin{align*}
      \langle \overline{\Pi}^h_{w,P}, C \rangle
    &= \langle \expect_{I,J} \Pi_{I,J}, C \rangle
    =  \expect_{I,J} \langle  \Pi_{I,J}, C \rangle
    =  \expect_{I,J} \langle  Q_I^\top\Pi^m_{I,J}Q_J, C \rangle
    =  \expect_{I,J} \langle  \Pi^m_{I,J}, C_{I,J} \rangle
\end{align*}
Now observe that by definition of the minibatch transport plans $\Pi^m_{I,J}$ (cf Definition~\ref{def:MBTP}) we have, 
\begin{align*}
    \langle \Pi^m_{I,J}, C_{I,J} \rangle
    &=  h \Big( w_1(\mathbf{a},I), w_2(\mathbf{b},J), C_{(I,J)}\Big)
\end{align*}
For the Wasserstein distance $h = W_p^p$, combining all of the above we obtain
\[
h(\mathbf{a}, \mathbf{b}) \leq \langle \overline{\Pi}^h_{w,P}, C \rangle = \expect_{I,J} \langle  \Pi^m_{I,J}, C_{I,J} \rangle = 
\expect_{I,J}  h \Big( w_1(\mathbf{a},I), w_2(\mathbf{b},J), C_{(I,J)}\Big)
\]
\end{proof}

We prove some associations of reweighting functions and parametric laws on tuple which respects the marginal constraints~\eqref{app:wPadmissible}.

\begin{customlemma}{\ref{lem:admissibilitywPusual}}[Admissibility]
The uniform reweighting function $w^\mathtt{U}$ and the parametric law "with replacement" $P^\mathtt{U}$ satisfy the admissibility condition.
The admissibility condition also holds for the parametric law without replacement $P^\mathtt{W}$ with the normalized reweighting function $w^\mathtt{W}$.\\
In contrast for $w^\mathtt{U},P^\mathtt{W}$ when $\mathbf{a}$ is not uniform, the resulting OT matrix is not a transportation plan.

\end{customlemma}

\begin{proof}
Consider first $w^\mathtt{W}$ and draws with the probability law $P^\mathtt{W}$. This law only allows to draw $m$-tuples without repeated entries. Since the probability of drawing a tuples without repeated indices such that $\sum_{j \in I} a_j = 0$ is zero, without loss of generality we consider a draw $I$ such that $\sum_{j \in I} a_j > 0$. Given $1 \leq i \leq n$, we  distinguish several cases: if $i \notin I$ then $Q_I^\top w^\mathtt{W}(\mathbf{a},I) = 0$; otherwise there exists $1 \leq k \leq m$ such that $i = i_k$, hence 
\[
(Q_I^\top w^\mathtt{W}(\mathbf{a},I))_i = w_k^\mathtt{W}(\mathbf{a},I) = \frac{a_{i_k}}{\sum_{p=1}^m a_{i_p}} = \frac{a_i}{\sum_{j \in I} a_j}.
\]
As a result
\[
     \expect_{I \sim P_{\mathbf{a}}^\mathtt{W}} 
     (Q_I^\top w^\mathtt{W}(\mathbf{a},I))_i 
     = 
     \expect_{I \sim P_{\mathbf{a}}^\mathtt{W}}
     \frac{a_i}{\sum_{j \in I} a_j} \mathbf{1}_{I}(i)
     = a_i \expect_{I \sim P^\mathtt{W}_\mathbf{a}} \frac{\mathbf{1}_I(i)}{\sum_{j \in I} a_j}
 \]
 If $a_i=0$ the right hand side equals $a_i$. Assuming now $a_i>0$, we have $\sum_{j \in I} a_j>0$ for each $I$ that contains $i$, and we prove that $\expect_{I \sim P^\mathtt{W}_\mathbf{a}} \frac{\mathbf{1}_I(i)}{\sum_{j \in I} a_j}=1$. Indeed, by definition of $P^\mathtt{W}$ we have
 \begin{align*}
 \expect_{I \sim P^\mathtt{W}_\mathbf{a}} \frac{\mathbf{1}_I(i)}{\sum_{j \in I} a_j}
 &= \sum_{I \in \Pim} P_\mathbf{a}^\mathtt{W}(I) \frac{\mathbf{1}_I(i)}{\sum_{j \in I} a_j}
 = \sum_{I \in \Pim, I \ni i} P_\mathbf{a}^\mathtt{W}(I) \frac{1}{\sum_{j \in I} a_j}
 = \sum_{I \in \Pim, I \ni i} \frac{(n-m)!}{m(n-1)!}\\
 &= \frac{(n-m)!}{m(n-1)!} \cdot \sharp \{I \in \Pim, i \in I\}
 = 1
 \end{align*}
Where the last equality is from \eqref{app:num_i_all_tup}. We can conclude that $\expect_{I \sim P_{\mathbf{a}}^\mathtt{W}} 
     Q_I^\top w^\mathtt{W}(\mathbf{a},I) = \mathbf{a}$ for every $\mathbf{a}$.
   
To show that admissibility does not hold with $w^\mathtt{U}$ and $P^\mathtt{W}$, we similarly obtain
\[
 \expect_{I \sim P_{\mathbf{a}}^\mathtt{W}} 
     (Q_I^\top w^\mathtt{U}(\mathbf{a},I))_i 
   = 
     \expect_{I \sim P_{\mathbf{a}}^\mathtt{W}}
     \tfrac{1}{m} \mathbf{1}_{I}(i)
     = \tfrac{1}{m} \expect_{I \sim P^\mathtt{W}_\mathbf{a}} 
     \mathbf{1}_I(i).
\]
When $\mathbf{a}$ is not uniform, by the pigeonhole principle there is an index $i$ such that $a_i>1/m$. Since the right hand side above cannot exceed $1/m$, we conclude that $\expect_{I \sim P_{\mathbf{a}}^\mathtt{W}} 
     (Q_I^\top w^\mathtt{U}(\mathbf{a},I)) \neq \mathbf{a}$.

Consider now the pair $(w^\mathtt{U},P^\mathtt{U})$. 
For an $m$-tuple $I = (i_1,\ldots,i_m)$ we denote $m_j = m_j(I)$ the multiplicity of index $1 \leq j \leq n$ and observe that $m_1+\cdots+m_n=m$, and $\Pi_{j \in I} a_j = \Pi_{k=1}^n a_k^{m_k}$. Vice-versa, given integers $(m_1,\ldots,m_n)$ such that $m_1+\cdots+m_n=m$ there are $m!/(m_1! \cdots m_n!)$ $m$-tuples $I$ with the corresponding multiplicity.
Given $1 \leq i \leq n$, reasoning as above we obtain
\begin{align*}
    \expect_{I \sim P_{\mathbf{a}}^\mathtt{U}}  (Q_I^\top w^\mathtt{U}(\mathbf{a},I))_{i} 
    &= \expect_{I \sim P_{\mathbf{a}}^\mathtt{U}}  \frac{m_{i}}{m} \mathbf{1}_{I}(i)
    =\sum_{I \in \llbracket n\rrbracket^m} P_\mathbf{a}^\mathtt{U}(I)  \frac{m_{i}}{m}  \mathbf{1}_{I}(i) 
    =\sum_{I \in \llbracket n\rrbracket^m} (\Pi_{j \in I} a_j)  \frac{m_{i}}{m}  \mathbf{1}_{I}(i)\\
    &= \sum_{m_1+\cdots+m_n=m} \frac{m !}{m_{1}! \cdots m_{n} !}  (\Pi_{k=1}^n a_k^{m_k})  \frac{m_{i}}{m}  \mathbf{1}(m_i \geq 1)\\
    & = \sum_{\substack{m_{1}+\cdots+m_{n}=m\\ m_{i} \geq 1}}\frac{m !}{m_{1}! \cdots m_i! \cdots m_{n} !} \frac{m_{i}}{m} \Pi_{k=1}^n a_k^{m_k} \\
    & = a_{i} \sum_{m'_{1}+\cdots+m'_{n}=m-1} \frac{(m-1)!}{m'_{1}! \cdots m'_{i}! \cdots m'_{n} !} \Pi_{k=1}^n a_k^{m'_k}
     = a_{i} \left(\sum_{i=1}^n a_i\right)^{m-1} = a_{i}.
\end{align*}
In the last line, we used Newton's multinomial theorem and the fact that $\sum_{i=1}^n a_i=1$. 
\end{proof}

\paragraph{Upper bound.} We now give a upper bound of minibatch Optimal Transport. We have access to empirical data and the distance between each data can be bounded by the maximum distance between data, \emph{i.e.,} for two random data $\xx$ and $\yy$, we have : $ \| \xx - \yy \|_2^p \leqslant 2 \operatorname{max}_{1 \leqslant i,j \leqslant n} \| \xx_i - \yy_j \|_2^p$.  

\begin{lemma}[Upper bounds on OT kernels]\label{app:lemma_upper_bound} Let $\XX=(\xx_1, \cdots, \xx_n)$, $\YY=(\yy_1,\cdots, \yy_n)$ be two $n$-tuple of vectors in $\R^d$ and $C$ the ground cost matrix. Let $\aa$ and $\bb$ be two probability vectors, $w_1$ and $w_2$ be two reweighting functions and let $(I,J)$ be two $m$ tuples. Then, we have the following bounds for kernel OT $h \in \{W_p, W_p^p, W^\varepsilon, S^\varepsilon \}$:
\begin{equation}
   h \big( w_1(\mathbf{a},I), w_2(\mathbf{b},J), C_{(I,J)} (\XX, \YY)   \big) \leqslant 2 \max_{1 \leqslant i,j \leqslant n} \|\xx_i - \yy_j \|_2^p.
\label{app:upper_bound}
\end{equation}
and for $h = \GW$, let $C^1 = C(\XX,\XX)$ and $C^2 = C(\YY,\YY)$. Then,
\begin{equation}
   h \big( w_1(\mathbf{a},I), w_2(\mathbf{b},J), C_{(I,I)}^1, C_{(J,J)}^2   \big) \leqslant \max_{1 \leqslant i,j,k,l \leqslant n} \| C_{i,j}^1 - C_{k,l}^2 \|_2^p.
\label{app:upper_bound_gw}
\end{equation}
\end{lemma}
\begin{proof} We start with the case $ h =W_{\varepsilon} $ for $ \varepsilon \geq 0$. Note that with our choice of cost matrix $C=(C_{i,j})_{1\leq i,j\leq n}$ one has $ 0  \leqslant C_{i,j}  \leqslant 2 \max_{1 \leqslant i,j \leqslant n} \| \xx_i - \yy_j \|_2^p$. Denote the optimal transport plan between $I$ and $J$ as $\Pi^\star = (\Pi_{i,j})$ (with respect to the cost matrix $ C_{I,J} $), consider the transport plan $w_1(\mathbf{a},I) \otimes w_2(\mathbf{b},J)$, we directly have:

\begin{align}
&\V \langle \Pi^\star,C_{I,J} \rangle - \varepsilon H(\Pi^\star|w_1(\mathbf{a},I) \otimes w_2(\mathbf{b},J)) \V \nonumber \\
&\leq \langle w_1(\mathbf{a},I) \otimes w_2(\mathbf{b},J),C_{I,J} \rangle + \varepsilon (H(w_1(\mathbf{a},I) \otimes w_2(\mathbf{b},J)|w_1(\mathbf{a},I) \otimes w_2(\mathbf{b},J)) \nonumber \\
&\leq 2 \max_{1 \leqslant i,j \leqslant n} \| \xx_i - \yy_j \|_2^p
\end{align}
As the second term is equal to zero in first inequality's right hand side expression. The extension to is direct $h = S_{\varepsilon}$ as it is a weighted sum of three terms of the form $ W_{\varepsilon} $ one can conclude. Lastly, a similar argument gives the desired bound for the Gromov-Wasserstein distance. Let $C^1 = C(\XX,\XX)$ and $C^2 = C(\YY,\YY)$, for 2 $m$-tuples $I,J$, one can write:

\begin{align}
&\V \sum_{i,j,k,l} \|C_{i,j}^1 - C_{k,l}^2 \|_2^p \pi_{i,j} \pi_{k,l} \V \leq \max_{1 \leqslant i,j,k,l \leqslant n} \| C_{i,j}^1 - C_{k,l}^2 \|_2^p
\end{align}

Finally, in the case of data lying in a compact, the quantity $2 \max_{1 \leqslant i,j \leqslant n} \| \xx_i - \yy_j \|_2^p$ is upper bounded by a constant $M = 2 (\operatorname{diam}(\operatorname{supp}(\alpha)) \cup \operatorname{diam}(\operatorname{supp}(\beta)))^p$ and $\max_{1 \leqslant i,j,k,l \leqslant n} \| C_{i,j}^1 - C_{k,l}^2 \|_2^p$ is upper bounded by $M = (\operatorname{diam}(\operatorname{supp}(\alpha)) \cup \operatorname{diam}(\operatorname{supp}(\beta)))^{p^2}$.

\end{proof}

\subsection{Concentration theorem (bounded)}\label{app_sec:concentration_bounded}
{In what follows, we are interested in concentration bounds with $m \in \mathbb{N}^* $ fixed. For $n \in \N^*$, we denote by $\uu$ the element of $\S_n$ such that $\uu_i = \frac{1}{n} ( 1 \leq i \leq n )$. We will also often omit the dependence of various quantities (the minibatch procedure $\overline{h}$, the reweighting function $w$ etc.) in the asymptotic parameter $n$.} The purpose of this appendix is to prove Theorem \ref{thm:inc_U_to_mean}. The appendix is structured as follows:

\begin{itemize}
    \item Appendix \ref{app_subsubsec:dev_est_mean} proves the deviation between the complete estimator $\overline{h}$ and its mean.
    \item Appendix \ref{app_subsubsec:dev_incomp_comp} provides the deviation between the complete estimator $\overline{h}$ and its incomplete counter part $\widetilde{h}$.
    \item Appendix \ref{app_subsubsec:inc_U_to_mean_proof} gathers all previous propositions and lemmas to prove Theorem \ref{thm:inc_U_to_mean} and corollaries.
\end{itemize}

\subsubsection{Deviation between the complete estimator $\overline{h}$ and its mean.}\label{app_subsubsec:dev_est_mean}
We focus on the first ingredient of our proof: the deviation between the complete estimator $\overline{h}$ and its mean. This proof is based on the U-statistics concentration inequality proof but needs to be adapted due to the non uniform probability vectors $\aa$ and $\bb$. We first state the famous Hoeffding lemma:

\begin{lemma}[Hoeffding's Lemma]
Let the real random variable $X \in[a, b]$ and denote $E X=\mu$. Then for all $s\in \mathbb{R}$:
\begin{equation}
    \mathrm{E}\left[e^{s(X-\mu)}\right] \leq e^{s^{2}(b-a)^{2} / 8}.    
\end{equation}
\label{lemma : Hoef lem}
\end{lemma}

From now on, the probability vectors $(\aa^{(n)})$ and $(\bb^{(n)})$ are sequences which depend on the number of data $n$. More precisely $(\aa^{(n)})_{n \in \N}$ and $(\bb^{(n)})_{n \in \N} $ are sequences of vectors of size $n$ such  that for each $n \in \N$, $\aa^{(n)}, \bb^{(n)} \in \Sigma_n$, we denote the space of these sequences as $(\aa^{(n)}), (\bb^{(n)}) \in \Sigma$. The sequence of probability vectors $(\aa^{(n)})$ and $(\bb^{(n)})$ can not be taken arbitrarily if we want to guarantee convergence. Hence we rely on local constraints that we defined in the paper. We recall them:

\begin{customdef}{\ref{def:loc_cond}}[Local averages conditions]  Let $(\aa^{(n)}) \in \Sigma$ and two integers $n,m \in \mathbb{N}^*$ such as $n \geq m$. \\
\textup{(i)} We say that $(\aa^{(n)})$ satisfies the local arithmetic mean condition if there exists a constant $D>0$ and $\gamma >0 $ such that for any $n \in \mathbb{N}$ and $I \in \llbracket n\rrbracket^m $ we have
\begin{equation}
    \frac{1}{m} \sum_{i \in I} \aa^{(n)}_i \leq \frac{D}{n^{\gamma}}. 
\label{loc_a}
\end{equation}
We write that $(\aa^{(n)})$ satisfies $\mathtt{Loc_A}(m,\gamma,D))$ (or $\mathtt{Loc_A}(m,\gamma)$) when the constant $D$ is implicit).\\
\textup{(ii)} Analogously, $(\aa^{(n)})$ is said to verify the local geometric mean condition if there exists a constant $D>0$ and $\gamma > 0$ such that for any $n \in \mathbb{N}^*$ and $I  \in  \llbracket n\rrbracket^m $ we have
\begin{equation}
   \Big( \Pi_{i \in I} \aa^{(n)}_i \Big)    ^{\frac{1}{m}} \leq \frac{D}{n^{\gamma}}.
\label{loc_g}
\end{equation}
We write that $(\aa^{(n)})$ verifies $\mathtt{Loc_G}(m,\gamma,D))$ (or $\mathtt{Loc_G}(m,\gamma)$) when the constant $D$ is implicit).
\end{customdef}

\medskip

We record the following properties of the $\mathtt{Loc_A}$ and $\mathtt{Loc_G}$ conditions.
\begin{customlemma}{\ref{lemma:local_constraints}} Let $m \in \N^*$, $\gamma >0$ and $D > 0$. Let $(\aa^{(n)}) \in \Sigma$ be a sequence of probability vectors. The following statements hold:\\
\textup{(i)} If $(\aa^{(n)})$ verifies $\mathtt{Loc_A}(m,\gamma,D)$ or $\mathtt{Loc_G}(m,\gamma,D)$ then $\gamma \leq 1$.\\
\textup{(ii)} If $(\aa^{(n)})$ is $\mathtt{Loc_A}(m,\gamma,D)$ then $(\aa^{(n)})$ is $\mathtt{Loc_G}(m,\gamma,D)$.
\end{customlemma}

\begin{proof} We first prove \textup{(i)}. Let $(\aa^{(n)}) \in \Sigma$ which verifies $\mathtt{Loc_A}(m,\gamma,D)$ for some $m\in \N$, $D >0$ and $\gamma >1$. Fix $1 \leq \ell \leq n$ an integer. By choosing $I = (\ell, \cdots, \ell) \in \llbracket n \rrbracket^m$, we have by \eqref{loc_a}
\begin{equation}
    a^{(n)}_{\ell} \leq \frac{D}{n^{\gamma}}
\label{pf_lem:mean1}.
\end{equation}
Hence, we find by summing \eqref{pf_lem:mean1} for $1 \leq \ell \leq n$
\[ 1 \leq D n^{1 - \gamma }, \]
which contradicts $(\aa^{(n)}) \in \simplex$. The proof is similar if $(\aa^{(n)})$ verifies $\mathtt{Loc_G}(m,\gamma,D)$. Lastly, \textup{(ii)} follows from the arithmetic-geometric mean inequality.
\end{proof}
We show that in order to obtain concentration properties of the estimators $\overline{h}_{w,P}(\aa^{(n)}, \bb^{(n)})$ we need to ensure that the sequences $(\aa^{(n)})$ and $(\bb^{(n)})$ verify the \textit{local condition} with enough decay, e.g. $(\aa^{(n)}), (\bb^{(n)})$ are $\mathtt{Loc_A}(m,\gamma)$ or $\mathtt{Loc_G}(m,\gamma)$ for a $\gamma$ sufficiently close to $1$.

Hereafter, we denote by $\tau$ an absolute (and possibly large) constant. We also writes $\tau = \tau(\gamma)$ to denote constants which depend on some parameter $\gamma$.

\begin{proposition}
[Generalized U-statistics concentration bound]\label{thm:U_to_mean}
Let $\delta \in (0,1) $ and $m \geq 1$ be a fixed integer. Consider two distributions $\alpha,\beta$, two n-tuples of empirical data $\XX \sim \alpha^{\otimes n}, \YY \sim \beta^{\otimes n}$ and a kernel $h \in \{W_p, W_\epsilon, S_\epsilon, \GW_p$\}. Let the reweighting function $w$ and the probability law $P$ over $m$-tuple be as in \eqref{EQ : rew_fun}. Let $(\aa^{(n)}), (\bb^{(n)}) \in \Sigma$ satisfy $\mathtt{Loc_A}(m,\gamma,D)$ for some $\gamma \in (\frac{3}{4}, 1]$ and $D>0$. We have the following concentration bound for the sampling without replacement
\begin{equation}\label{app:hoeffding_bounded_case_without_rep}
\P \left( \big| \overline{h}_{w^\mathtt{W},P^\mathtt{W}}(\aa^{(n)}, \bb^{(n)}) - \E \overline{h}_{w^\mathtt{W},P^\mathtt{W}} (\aa^{(n)}, \bb^{(n)}) \big| \geq  2M D^2 \frac{m^{ \frac12 }}{n^{2( \gamma - \frac{3}{4}) }} \sqrt{ 2\log(2/\delta)} \right) \leq \delta,
\end{equation}
where $M = \tau(\operatorname{diam}(\operatorname{supp}(\alpha)) \cup \operatorname{diam}(\operatorname{supp}(\beta)))$. And for the sampling with replacement, let the sequence probability vectors $(\aa^{(n)}), (\bb^{(n)})$ verify $\mathtt{Loc_G}(m,\gamma,D)$ for some $\gamma \in (1-\frac{1}{4m}, 1]$ and $D>0$. We have the following concentration bound
\begin{equation}
\P \left( \big| \overline{h}_{w^\mathtt{U},P^\mathtt{U}}(\aa^{(n)}, \bb^{(n)}) - \E \overline{h}_{w^\mathtt{U},P^\mathtt{U}} (\aa^{(n)}, \bb^{(n)}) \big| \geq 2 M  \frac{D^{2m} m^{ \frac12 }}{n^{2m(\gamma - 1 + \frac{1}{4m})}} \sqrt{2 \log(2/\delta)} \right) \leq \delta,
\label{EQ:hoef_with_replcament}
\end{equation}
where $M = \tau(\operatorname{diam}(\operatorname{supp}(\alpha)) \cup \operatorname{diam}(\operatorname{supp}(\beta)))$.
\end{proposition}

\begin{remark} For $(\aa^{(n)}) = (\bb^{(n)}) = (\uu^{(n)})$ and $P= P^{\mathtt{W}}$ we find our minibatch OT losses defined in \cite{fatras2019batchwass} since $(\uu^{(n)})$ verifies $\mathtt{Loc}(m,1,1)$.
\end{remark}

\begin{proof} The proof is inspired by the two-sample U-statistic proof from \citep[section 5]{Hoeffding1963}. We start with the sampling without replacement case.
\newline

{\bfseries Sampling without replacement :}
We first consider the case of Example \ref{def:law_indices_rep2}, i.e, when the law $P$ is given by \eqref{EQ : samp_rep}. The proof is based on two-sample U-statistic Hoeffding inequalities and we give it for $h \in \{W_p, W_\epsilon, S_\epsilon\}$ as the $\GW_p$ follows the same principle. The goal is to rewrite $\overline{h}_{w^\mathtt{W},P^\mathtt{W}}$ as a superposition of terms, each of which are sums of independent random variables. Let $(\aa^{(n)})$ and $(\bb^{(n)})$ be as in the above. To ease the notations, the dependence in $n$ will be implicit.

We fix $r=\lfloor n/m \rfloor$. Let $0 \leq k \leq r-1$, we define the set $I^k := \{km+1, \cdots, km+m\} $. Then we define the function $V$ as :

\begin{align}
    &V(\xx_1, \cdots, \xx_n, \yy_1, \cdots, \yy_n) = \frac{1}{r} \sum_{k=0}^{r-1} P_{\mathbf{a}}^\mathtt{W}(I^k) P_{\mathbf{b}}^\mathtt{W}(I^k) h \big( w^\mathtt{W}(\mathbf{a},I^k), w^\mathtt{W}(\mathbf{b},I^k), C_{(I^k,I^k)}\big)
\label{EQ : def_V}
\end{align}
We recall the implicit dependence in $\XX$ and $\YY$ in the right-hand-side of \eqref{EQ : def_V} through the ground costs $C$. In the summation below, $
\sigma_x$ or $ \sigma_y$ denotes a generic permutation of $\{1, \cdots, n\}$. We compute :

\begin{align}
& \frac{1}{n!^2} \sum_{\sigma_x, \sigma_y} V(\xx_{\sigma_x(1)}, \cdots, \xx_{\sigma_x(n)}, \yy_{\sigma_y(1)}, \cdots, \yy_{\sigma_y(n)})  \\
&\qquad = \frac{(n-m)!^2}{n!^2}    \sum_{I \in \Pim} \sum_{J \in \Pim} P_{\mathbf{a}}^\mathtt{W}(I) P_{\mathbf{a}}^\mathtt{W}(J) h \Big( w^\mathtt{W}(\mathbf{a},I), w^\mathtt{W}(\mathbf{b},J), C_{(I,J)}\Big)  \\
&\qquad = \frac{(n-m)!^2}{n!^2}  \overline{h}_{w^\mathtt{W},P^\mathtt{W}}(\aa, \bb) 
\end{align}

Finally, we have : 
\begin{equation}
\overline{h}_{w^\mathtt{W},P^\mathtt{W}}(\aa, \bb) = \frac{1}{n!^2} \sum_{\sigma_x, \sigma_y} V'(\xx_{\sigma_x(1)}, \cdots, \xx_{\sigma_x(n)}, \yy_{\sigma_y(1)}, \cdots, \yy_{\sigma_y(n)}),
\end{equation}
where the function $V'$ is defined by
\begin{align*}
    V' =  \left( \frac{n!}{(n-m)!} \right)^2 V
\end{align*}
Let us define
\begin{equation}
    T(\xx_1, \cdots, \xx_n, \yy_1, \cdots, \yy_n) = V'(\xx_1, \cdots, \xx_n, \yy_1, \cdots, \yy_n) - \E \big[ V'(\xx_1, \cdots, \xx_n, \yy_1, \cdots, \yy_n) \big], \nonumber
\end{equation}
We have 
\begin{equation}
\overline{h}_{w^\mathtt{W},P^\mathtt{W}}(\aa, \bb) - \E \big[  \overline{h}_{w^\mathtt{W},P^\mathtt{W}}(\aa, \bb)  \big] = \frac{1}{n!^2} \sum_{\sigma_x, \sigma_y} T(\xx_{\sigma_x(1)}, \cdots, \xx_{\sigma_x(n)}, \yy_{\sigma_y(1)}, \cdots, \yy_{\sigma_y(n)})
\label{EQ : id_T}
\end{equation}
Note that $T$ may be rewritten as a sum as in \eqref{EQ : def_V} with $h \big( w^\mathtt{W}(\mathbf{a},I^k), w^\mathtt{W}(\mathbf{b},I^k), C_{(I^k,I^k)}\big)$ replaced by $h \big( w^\mathtt{W}(\mathbf{a},I^k), w^\mathtt{W}(\mathbf{b},I^k), C_{(I^k,I^k)}\big) - \E\big[ h \big( w^\mathtt{W}(\mathbf{a},I^k), w^\mathtt{W}(\mathbf{b},I^k), C_{(I^k,I^k)}\big) \big]$ for each $k$. 
\newline

More precisely, we write $T(\xx_{\sigma_x(1)}, \cdots, \xx_{\sigma_x(n)}, \yy_{\sigma_y(1)}, \cdots, \yy_{\sigma_y(n)}) = \frac{1}{r}  \sum_{k=0}^{r-1} T^{\sigma_x, \sigma_y}_k $ for $\sigma_x, \sigma_y$ two permutations of $\llbracket n \rrbracket$. Here, $T^{\sigma_x, \sigma_y}_k$ are independent and centered random variables such that
\begin{align}
\label{EQ : bd_Hoef_repla}
    \big|  T^{\sigma_x, \sigma_y}_k \big| & \leq 2 M  \ \Big\{ \frac{1}{m} \frac{(n-m)!}{(n-1)!} \Big\}^2 \left( \frac{n!}{(n-m)!} \right)^2  \sum_{i=km+1}^{(k+1)m} a_{\sigma_x(i)}  \sum_{i=km+1}^{(k+1)m} b_{\sigma_y(i)}  \nonumber \\
    & \leq  2 M D^2 n^{2(1- \gamma)}
\end{align}
thanks to the $\mathtt{Loc_A}(m,\gamma,D)$-condition. In what follows we write $T_{\sigma} = T \big( (\xx_{\sigma_x(i)})_i, (\yy_{\sigma_y(i)})_i  \big) $ for simplicity. From \eqref{EQ : id_T}, we get
\begin{align*}
&\mathbb{P}\left(\overline{h}_{w^\mathtt{W},P^\mathtt{W}}(\aa, \bb)- \E \big[ \overline{h}_{w^\mathtt{W},P^\mathtt{W}}(\aa, \bb) \big] \geq t\right) \leq e^{-\lambda t} \mathbb{E}\left[e^{\lambda (\overline{h}_{w^\mathtt{W},P^\mathtt{W}}(\aa, \bb)-\E [ \overline{h}_{w^\mathtt{W},P^\mathtt{W}}(\aa, \bb) ]) }\right]\\ 
&\qquad =e^{-\lambda t} \mathbb{E}\left[e^{\lambda \frac{1}{n !^2} \sum_{\sigma} T_{\sigma}}\right] \leq e^{-\lambda t} \frac{1}{(n!)^2} \sum_{\sigma} \mathbb{E}\left[e^{\lambda T_{\sigma}}\right]  \leq e^{-\lambda t}  \max_{\sigma} \mathbb{E}\left[e^{\lambda T_{\sigma}}\right],
\end{align*}
where in the first and second inequalities, we used Markov's and Jensen's inequalities respectively. Furthermore, for any $ T =  T_{\sigma}$ we have from Lemma \ref{lemma : Hoef lem} along with \eqref{EQ : bd_Hoef_repla},
\begin{align}
    \E \big[e^{\lambda T} \big] = \prod_{k=0}^{r-1} \E \big[ e^{ \frac{\lambda}{r}  T_k} \big] & \leq \prod_{k=0}^{r-1} \E \big[ e^{ \frac{2\lambda^2}{r^2}  M^2   D^4 n^{4(1-\gamma)} } \big]  = \exp \Big( 2\lambda^2  M^2 D^4 m n^{3 - 4 \gamma}  \Big) \label{EQ : bd_esp}
\end{align}
Hence,
\begin{equation}
    \mathbb{P}\left(\overline{h}_{w^\mathtt{W},P^\mathtt{W}}(\aa, \bb)- \E \big[ \overline{h}_{w^\mathtt{W},P^\mathtt{W}}(\aa, \bb) \big] \geq t\right) \leq \exp \Big( - \lambda t +  2\lambda^2 M^2 D^4 m \cdot n^{3 - 4 \gamma}  \Big)
    \nonumber
\end{equation}
Optimizing the latter over $\lambda \in \R_+$ and following a similar reasoning for $ - \big( \overline{h}_{w^\mathtt{W},P^\mathtt{W}}(\aa, \bb)- \E \big[ \overline{h}_{w^\mathtt{W},P^\mathtt{W}}(\aa, \bb) \big] \big)$ gives \eqref{app:hoeffding_bounded_case_without_rep}.

The $\GW_p$ proof follows the same principle but differs in the definition of $V$ which would be equal to:
\begin{align*}
    &V(\xx_1, \cdots, \xx_n, \yy_1, \cdots, \yy_n) = \frac{1}{r} \sum_{k=0}^{r-1} P_{\mathbf{a}}^\mathtt{W}(I^k) P_{\mathbf{b}}^\mathtt{W}(I^k) h \big( w^\mathtt{W}(\mathbf{a},I^k), w^\mathtt{W}(\mathbf{b},I^k), C_{(I^k,I^k)}^1, C_{(I^k,I^k)}^2\big),
\label{EQ : def_V_GW}
\end{align*}
with ground costs $C^1 = C^{n,p}(\XX, \XX)$ and $C^2 = C^{n,p}(\YY, \YY)$.

{
{\bfseries Sampling with replacement}
Let us now consider the sampling with replacement $P^\mathtt{U}$ with the reweighting function $w^\mathtt{U}$. We follow the same procedure as in section 5.C from \citep{Hoeffding1963}. For sake of simplicity, we abbreviate $w^\mathtt{U}(\mathbf{a},I)$ as $\uu \in \simplex_m$ and give the proof for $h \in \{W_p, W_\epsilon, S_\epsilon\}$ as the $\GW$ case can be deduced from it. In this case, it is possible to rewrite $\overline{h}_{w^\mathtt{U},P^\mathtt{U}}$ as a sum over m-tuples without replacement, i.e.,
\begin{align}
    \overline{h}_{w^\mathtt{U},P^\mathtt{U}}(\aa, \bb) &= \sum_{I_1 \in \llbracket n\rrbracket^{m}, I^2 \in \llbracket n\rrbracket^{m}} P_{\mathbf{a}}^\mathtt{U}(I_1) P_{\mathbf{b}}^\mathtt{U}(I_2) h \Big( \uu, \uu, C_{(I_1,I_2)}\Big) \\
    &= \Big(n^m \frac{(n-m)!}{n!}\Big)^2 \sum_{I_1 \in \Pim, I_2 \in \Pim} h^{\star}(I_1,I_2).
\end{align}

Where $h^{\star}$ is a weighted arithmetic mean of certain values of $h$. Let us take an example with $m=2$, we consider $(i_1, i_2) \in \mathcal{P}^2$ and $(j_1, j_2) \in \mathcal{P}^2$ we have:
\begin{align*}
h^{\star}\left((i_1, i_2), (j_1, j_2) \right) = & \Big(\frac{n-1}{n} \Big)^2 P_{\mathbf{a}}^\mathtt{U}((i_1,i_2)) P_{\mathbf{b}}^\mathtt{U}((j_1, j_2)) h(\uu, \uu, C_{(i_1,i_2), (j_1, j_2)}) \\
&+ \frac{n-1}{n^2} P_{\mathbf{a}}^\mathtt{U}((i_1,i_1)) P_{\mathbf{b}}^\mathtt{U}((j_1, j_2)) h(\uu, \uu, C_{(i_1,i_1), (j_1, j_2)}) \\
&+ \frac{n-1}{n^2} P_{\mathbf{a}}^\mathtt{U}((i_1,i_2)) P_{\mathbf{b}}^\mathtt{U}((j_1, j_1)) h(\uu, \uu, C_{(i_1,i_2), (j_1, j_1)}) \\
&+ \frac{1}{n^2} P_{\mathbf{a}}^\mathtt{U}((i_1,i_1)) P_{\mathbf{b}}^\mathtt{U}((j_1, j_1)) h(\uu, \uu, C_{(i_1,i_1), (j_1, j_1)})
\end{align*}
More examples for V-statistics can be found in (section 5.C, \citep{Hoeffding1963}). Following the above example, we see that we have the bounds : $0 \leq h^{\star}  \leq \operatorname{max}_{I_1} \operatorname{max}_{I_2} P_{\mathbf{a}}^\mathtt{U}(I_1) P_{\mathbf{b}}^\mathtt{U}(I_2) M$. From now, the proof is like the sampling without replacement proof and we only give the main differences.


We fix $r=\lfloor n/m \rfloor$. Let $0 \leq k \leq r-1$, we define the set $I^k := \{km+1, \cdots, km+m\} $. Then we define the function $V$ as :

\begin{align}
    &V(\xx_1, \cdots, \xx_n, \yy_1, \cdots, \yy_n) = \frac{1}{r} \sum_{k=0}^{r-1}  h^\star \big(I^k,I^k\big)
\label{EQ : def_V_Vstat}
\end{align}
We recall the implicit dependence in $\XX$ and $\YY$ in the right-hand-side of \eqref{EQ : def_V_Vstat} through the ground costs $C$. In the summation below, $
\sigma_x$ or $ \sigma_y$ denotes a generic permutation of $\{1, \cdots, n\}$. We compute :

\begin{align}
& (\frac{n^m}{n!})^2 \sum_{\sigma_x, \sigma_y} V(\xx_{\sigma_x(1)}, \cdots, \xx_{\sigma_x(n)}, \yy_{\sigma_y(1)}, \cdots, \yy_{\sigma_y(n)}) \\
& = (n^m)^2\frac{(n-m)!^2}{n!^2}    \sum_{I^1 \in \mathcal{P}^m} \sum_{I^2 \in \mathcal{P}^m} h^\star \Big( I^1, I^2\Big)\\
& = \overline{h}_{w^\mathtt{U},P^\mathtt{U}}(\aa, \bb) 
\end{align}

Finally, we have : 
\begin{equation}
\overline{h}_{w^\mathtt{U},P^\mathtt{U}}(\aa, \bb) = \frac{1}{n!^2} \sum_{\sigma_x, \sigma_y} V'(\xx_{\sigma_x(1)}, \cdots, \xx_{\sigma_x(n)}, \yy_{\sigma_y(1)}, \cdots, \yy_{\sigma_y(n)}),
\end{equation}
where the function $V'$ is defined by $V' =  (n^m)^2 V$. 
Let us define
\begin{equation}
    T(\xx_1, \cdots, \xx_n, \yy_1, \cdots, \yy_n) = V'(\xx_1, \cdots, \xx_n, \yy_1, \cdots, \yy_n) - \E \big[ V'(\xx_1, \cdots, \xx_n, \yy_1, \cdots, \yy_n) \big], \nonumber
\end{equation}

\noindent
We have 
\begin{equation}
\overline{h}_{w^\mathtt{U},P^\mathtt{U}}(\aa, \bb) - \E \big[  \overline{h}_{w^\mathtt{U},P^\mathtt{U}}(\aa, \bb)  \big] = \frac{1}{n!^2} \sum_{\sigma_x, \sigma_y} T(\xx_{\sigma_x(1)}, \cdots, \xx_{\sigma_x(n)}, \yy_{\sigma_y(1)}, \cdots, \yy_{\sigma_y(n)})
\label{EQ : id_T_V}
\end{equation}
Note that $T$ may be rewritten as a sum as in \eqref{EQ : def_V_Vstat} with $h \big( \uu, \uu, C_{(I^k,I^k)}\big)$ replaced by $h \big( \uu, \uu, C_{(I^k,I^k)}\big) - \E\big[ h \big( \uu, \uu, C_{(I^k,I^k)}\big) \big]$ for each $k$. 
\newline
More precisely, we write $T(\xx_{\sigma_x(1)}, \cdots, \xx_{\sigma_x(n)}, \yy_{\sigma_y(1)}, \cdots, \yy_{\sigma_y(n)}) = \frac{1}{r}  \sum_{k=0}^{r-1} T^{\sigma_x, \sigma_y}_k $ for $\sigma_x, \sigma_y$ two permutations of $\llbracket n \rrbracket$. Here, $T^{\sigma_x, \sigma_y}_k$ are independent and centered random variables such that
\begin{align}
    \V T^{\sigma_x, \sigma_y}_k \V & \leq 2 M n^{2m}  \operatorname{max}_{I_1} P_{\mathbf{a}}^\mathtt{U}(I_1) \operatorname{max}_{I_2}  P_{\mathbf{b}}^\mathtt{U}(I_2) \nonumber \\
\label{EQ : bd_Hoef_repla_2}
    & \leq 2  M D^{2m} n^{2m(1-\gamma)}  
\end{align}
Where the second inequality uses definition \ref{DEF : loc_prod}. In what follows we write $T_{\sigma} = T \big( (\xx_{\sigma_x(i)})_i, (\yy_{\sigma_y(i)})_i  \big) $ for simplicity. From \eqref{EQ : id_T_V}, we get
\begin{align*}
&\mathbb{P}\left(\overline{h}_{w^\mathtt{U},P^\mathtt{U}}(\aa, \bb)- \E \big[ \overline{h}_{w^\mathtt{U},P^\mathtt{U}}(\aa, \bb) \big] \geq t\right) \leq e^{-\lambda t}  \max_{\sigma} \mathbb{E}\left[e^{\lambda T_{\sigma}}\right],
\end{align*}
Furthermore, for any $ T =  T_{\sigma}$ we have from Lemma \ref{lemma : Hoef lem} along with \eqref{EQ : bd_Hoef_repla_2},
\begin{align}
    \E \big[e^{\lambda T} \big] \leq \prod_{k=0}^{r-1} \E \big[ e^{ 2\frac{\lambda^2}{r^2}  M^2   D^4 } \big] = \exp \Big( 2\frac{\lambda^2}{r}  M^2 D^{4m} n^{4m(1-\gamma)}   \Big) \label{EQ : bd_esp_2}
\end{align}
Hence, 
    $\mathbb{P}\left(\overline{h}_{w^\mathtt{U},P^\mathtt{U}}(\aa, \bb)- \E \big[ \overline{h}_{w^\mathtt{U},P^\mathtt{U}}(\aa, \bb) \big] \geq t\right) \leq \exp \Big( - \lambda t +  2\lambda^2  M^2 D^{4m} mn^{4m - 1 - 4m\gamma}    \Big)
    $

Optimizing the latter over $\lambda \in \R_+$ and following a similar reasoning for $ - \big( \overline{h}_{w^\mathtt{U},P^\mathtt{U}}(\aa, \bb)- \E \big[ \overline{h}_{w^\mathtt{U},P^\mathtt{U}}(\aa, \bb) \big] \big)$ gives \eqref{EQ:hoef_with_replcament} gives the desired results.
}
\end{proof}

The sampling with replacement bounds show that when the minibatch size $m$ gets bigger, $\aa$ and $\bb$ must have a $\gamma$ close to 1. Now that we have a deviation between the complete estimator and its mean, we focus on the approximation of the complete estimator with its incomplete counter part.

\subsubsection{Deviation between the incomplete and complete estimator $\overline{h}$.}\label{app_subsubsec:dev_incomp_comp}
We are now ready to give the second and last ingredient of our proof: a deviation between the complete and the incomplete estimators. And in order to prove it, we rely on the Hoeffding inequality.

\begin{lemma}[Hoeffding's inequality] Let $X_{1}, \ldots, X_{n}$ be independent random variables such that $X_{i}$ takes its values in $\left[a_{i}, b_{i}\right]$ almost surely for all $i \leq n .$ Let the random variable $$S=\sum_{i=1}^{n}\left(X_{i}-E X_{i}\right).$$ Then for every $t>0$, we have:
\begin{equation}
P\{S \geq t\} \leq \exp \left(-\frac{2 t^{2}}{\sum_{i=1}^{n}\left(b_{i}-a_{i}\right)^{2}}\right)
\end{equation}
\end{lemma}

The following lemma gives us the wanted deviation:

\begin{lemma}[Deviation bound]\label{app:inc_U_to_U} Let $(\aa^{(n)}),(\bb^{(n)}) \in \Sigma$ be two sequences of probability vectors, let $ \delta \in (0,1) $ and an integer $ k \geqslant 1 $. Consider a reweighting function $w$, a probability law over $m$-tuple $P$ as in \eqref{EQ : rew_fun} and an OT kernel $h \in \{W_p, W_\epsilon, S_\epsilon, \GW_p$\}. We have a deviation bound between $\widetilde{h}_{w,P}^k(\aa, \bb)$ and $\overline{h}_{w,P}(\aa, \bb )$ depending on the number of minibatches $k$.

\begin{equation}
\P \left( \vert  \widetilde{h}_{w,P}^k(\aa^{(n)}, \bb^{(n)}) - \overline{h}_{w,P}(\aa^{(n)}, \bb^{(n)} ) \vert \geq  M \sqrt{\frac{2 \log(2/\delta)}{k}} \right) \leq \delta,
\end{equation}
\noindent where $M = \tau(\operatorname{diam}(\operatorname{supp}(\alpha)) \cup \operatorname{diam}(\operatorname{supp}(\beta)))$
\end{lemma}
\begin{proof} Thanks to Remark \ref{inc_ber} we have

\noindent
\begin{equation*}
 \widetilde{h}_{w,P}^k(\aa^{(n)}, \bb^{(n)}) - \overline{h}_{w,P}(\aa^{(n)}, \bb^{(n)} ) = \frac{1}{k} \sum_{\ell=1}^k \omega_l
\end{equation*}

\noindent
where $ \omega_l = \sum_{I, J \in \llbracket n \rrbracket^m} \big( \mathfrak{b}_{\ell}^{P_{\mathbf{a}^{(n)}},P_{\mathbf{b}^{(n)}}}(I,J)  - P_{\mathbf{a}^{(n)}}(I)  P_{\mathbf{b}^{(n)}}(J) \big) {  h \Big( w(\mathbf{a}^{(n)},I), w(\mathbf{b}^{(n)},J), C_{(I, J)}\Big) } $. Conditioned upon $ \XX = (\xx_1, \cdots, \xx_n)$ and $\YY = (\yy_1, \cdots, \yy_n) $, the variables $ \omega_l $ are independent, centered { and bounded by $2M$ with $M = \tau(\operatorname{diam}(\operatorname{supp}(\alpha)) \cup \operatorname{diam}(\operatorname{supp}(\beta)))$ thanks to lemma \ref{app:upper_bound}}. Using Hoeffding's inequality yields

\noindent
\begin{align*}
&\mathbb{P}( \vert  \widetilde{h}_{w,P}^k(\aa^{(n)}, \bb^{(n)}) - \overline{h}_{w,P}(\aa^{(n)}, \bb^{(n)} ) \vert > \varepsilon  ) \\
&\qquad = \mathbb{E} [ \mathbb{P}( \vert  \widetilde{h}_{w,P}^k(\aa^{(n)}, \bb^{(n)}) - \overline{h}_{w,P}(\aa^{(n)}, \bb^{(n)} ) \vert > \varepsilon \vert \XX, \YY )] \\
&\qquad = \mathbb{E} [ \mathbb{P}( \vert \frac{1}{k} \sum_{l=1}^k \omega_l ) \vert > \varepsilon \vert \XX, \YY )]\\
&\qquad \leqslant  \mathbb{E} [ 2 e^{\frac{-k \varepsilon^2}{2M^2}} ] = 2 e^{\frac{-k \varepsilon^2}{2M^2}}
\end{align*}

\noindent
which concludes the proof.
\end{proof}

\subsubsection{Proof of Theorem \ref{thm:inc_U_to_mean}.}\label{app_subsubsec:inc_U_to_mean_proof}
We have now the three ingredients to prove Theorem \ref{thm:inc_U_to_mean} : 

\begin{customtheorem}{\ref{thm:inc_U_to_mean}}[Maximal deviation bound for compactly supported distributions]~~\\
Let $\delta \in (0,1) $, $k \geqslant 1$ an integer and $m \geqslant 1$ be a fixed integer. Consider two distributions $\alpha,\beta$, two n-tuples of empirical data $\XX \sim \alpha^{\otimes n}, \YY \sim \beta^{\otimes n}$ and a kernel $h \in \{W_p, W_p^p, W_\epsilon, S_\epsilon, \GW_p$\}. Let the reweighting function $w$ and the probability law over $m$-tuple $P$ be as in \eqref{EQ : rew_fun}. Let the sequences of probability vectors $(\aa^{(n)}) \in \Sigma$ and $(\bb^{(n)}) \in \Sigma$ satisfy $\mathtt{Loc_A}(m,\gamma,D)$ and let $D>0$ and $\gamma \in (\frac34 ,1]$. We have a deviation bound for the sampling without replacement between $\widetilde{h}_{w,P}^k(\aa^{(n)}, \bb^{(n)})$ and $\E \overline{h}_{w,P}(\aa^{(n)}, \bb^{(n)})$ depending on the number of empirical data $n$ and the number of batches $k$:

\begin{equation}
\P \left( \vert \widetilde{h}_{w^\mathtt{W},P^\mathtt{W}}^k(\aa^{(n)}, \bb^{(n)}) - \E \overline{h}_{w^\mathtt{W},P^\mathtt{W}}(\aa^{(n)}, \bb^{(n)})\vert \geq M \Big(2\frac{D^2m^{ \frac12 }}{n^{2( \gamma - \frac{3}{4}) }} \sqrt{ 2\log(\frac{2}{\delta})} + \sqrt{\frac{2 \log(\frac{2}{\delta})}{k}} \Big) \right) \leq \delta.
\end{equation}

where $M = \tau(\operatorname{diam}(\operatorname{supp}(\alpha)) \cup \operatorname{diam}(\operatorname{supp}(\beta)))$. And for the sampling with replacement, let the sequences of probability vectors $(\aa^{(n)}), (\bb^{(n)})$ verify $\mathtt{Loc_G}(m,\gamma,D)$ for some $\gamma \in (1-\frac{1}{4m}, 1]$ and $D>0$.
\begin{equation}
\P \left( \vert \widetilde{h}_{w^\mathtt{U},P^\mathtt{U}}^k(\aa^{(n)}, \bb^{(n)}) - \E \overline{h}_{w^\mathtt{U},P^\mathtt{U}}(\aa^{(n)}, \bb^{(n)})\vert \geq M \Big(2 \frac{D^{2m} m^{ \frac12 }}{n^{2m(\gamma - 1 + \frac{1}{4m})}} \sqrt{ 2\log(\frac{2}{\delta})} + \sqrt{\frac{2 \log(\frac{2}{\delta})}{k}} \Big) \right) \leq \delta,
\label{app:hoeffding_bounded_case}
\end{equation}
\end{customtheorem}

\begin{proof} With the triangle inequality, we have:
\begin{align}
&\vert \widetilde{h}_{w,P}^k(\aa^{(n)}, \bb^{(n)}) - \E \overline{h}_{w,P}(\aa^{(n)}, \bb^{(n)}) \vert \nonumber \\
& \qquad  \leq \vert \widetilde{h}_{w,P}^k(\aa^{(n)}, \bb^{(n)}) - \overline{h}_{w,P}(\aa^{(n)}, \bb^{(n)}) \vert + \vert \overline{h}_{w,P}(\aa^{(n)}, \bb^{(n)}) - \E \overline{h}_{w,P}(\aa^{(n)}, \bb^{(n)}) \vert 
\end{align}
Thanks to lemma \ref{app:inc_U_to_U} and \ref{thm:U_to_mean} we get the desired results.
\end{proof}

\paragraph{Corollary.} It is also possible to have a bound on the expectation over the batch couples and the data.
\begin{customcorollary}{\ref{thm:expectation_concentration_inequality}}
    With the same hypothesis and notations as in theorem \ref{thm:inc_U_to_mean}. The following inequality holds:
    \begin{align}
      \mathbb{E}[ \vert   \widetilde{h}_{w^\mathtt{W},P^\mathtt{W}}^k(\aa^{(n)}, \bb^{(n)}) - \E \overline{h}_{w^\mathtt{W},P^\mathtt{W}}(\aa^{(n)}, \bb^{(n)}) \vert  ] \leqslant 20 \cdot M \max \Big(2\sqrt{2} D^2 \frac{m^{ \frac12 }}{n^{2( \gamma - \frac{3}{4}) }}, \sqrt{\frac{2}{k}} \Big)\\
      \mathbb{E}[ \vert   \widetilde{h}_{w^\mathtt{U},P^\mathtt{U}}^k(\aa^{(n)}, \bb^{(n)}) - \E \overline{h}_{w^\mathtt{U},P^\mathtt{U}}(\aa^{(n)}, \bb^{(n)}) \vert  ] \leqslant 20 \cdot M \max \Big(2\sqrt{2} D^{2m} \frac{m^{ \frac12 }}{n^{2m(\gamma - 1 + \frac{1}{4m})}}, \sqrt{\frac{2}{k}}\Big)
    \end{align}
    \begin{proof} 
        Once again, we give the proof for the sampling without replacement and the proof for the sampling with replacement follows the same steps. The proof for the debiased minibatch is straight forward as we have three terms of the form $\widetilde{h}_{w^\mathtt{W},P^\mathtt{W}}^k(\aa^{(n)}, \bb^{(n)}) - \E \overline{h}_{w^\mathtt{W},P^\mathtt{W}}(\aa^{(n)}, \bb^{(n)})$.
        
        Let us recall the formula : for a real random variable $X$. If $X \geqslant 0$ then $ \mathbb{E}[X] = \int_0^{\infty} \mathbb{P}(X > \lambda) d \lambda $. We denote by $X$ the random variable $ \vert \widetilde{h}_{w^\mathtt{W},P^\mathtt{W}}^k(\aa^{(n)}, \bb^{(n)}) - \E \overline{h}_{w^\mathtt{W},P^\mathtt{W}}(\aa^{(n)}, \bb^{(n)}) \vert $. 
        The last Theorem \ref{thm:inc_U_to_mean} writes
        \[ \mathbb{P}\left( X > C \sqrt{ \log( \frac{2}{\delta} )} \right) \leqslant \delta \]
        where $C := 2 M \max(2\sqrt{2} D^2 \frac{m^{ \frac12 }}{n^{2( \gamma - \frac{3}{4}) }}, \sqrt{\frac{2}{k}}) $. We can rewrite it as
        \[  \mathbb{P}( X > \lambda ) \leqslant \exp{(- \frac{\lambda^2}{C^2}) }   \]
        Thus, using the formula above:
        \begin{align*}
        \mathbb{E}[X] & \leqslant \int_{0}^{\infty} \exp{(- \frac{\lambda^2}{C^2} )}  d\lambda \\
        & = C  \int_{0}^{\infty} \exp({-u^2)}  du \leqslant 10C
        \end{align*}
        as announced.
    \end{proof}
     
\end{customcorollary}

\subsection{Concentration theorem (subgaussian)}\label{app_sec:subgauss_concentration_bounded}
In this section we relax the assumption of bounded data and give a proof for Theorem \ref{thm:sugaussian_concentration}. We start by recalling the subgaussian data definition:

\begin{definition}[Subgaussian random vectors] A random vector $\xx \in \R^{d}$ is subGaussian, if there exists $\sigma \in \R$ so that:
$$
\mathbb{E} e^{\langle\yy, \xx-\mathbb{E} \xx\rangle} \leq e^{\frac{\|\yy\|^{2} \sigma^{2}}{2}}, \quad \forall \yy \in \mathbb{R}^{d}
$$
\end{definition}

We write the class of subGaussian random vectors as $\operatorname{SG}(\sigma)$. The proof for the subgaussian case rely on a truncation argument between data that lie in some compact data which do not. So instead we consider the following class of random vectors:
\begin{definition}[Norm subgaussian data \citep{Jin2019nsg}] Let $\xx \in \mathbb{R}^d$ be a random vector and $\rho_\xx = \mathbb{E}\xx$. We say that $ \xx \in \operatorname{normSG}(\rho_\xx, \sigma^2) $ for some $ \s >0 $ if the following inequality holds
\begin{equation}
\P \big( \|\xx - \rho_\xx\| \geq t \big) \leq 2 \exp \left( - \frac{t^2}{2 \s ^2} \right).
\label{Eq : tail_bounds_nSG}       
\end{equation}

\label{def:norm_subgaussian}
\end{definition}

Norm subGaussian random vectors are a generalization of both subGaussian random vectors and norm bounded random vectors. They show tighter concentration bounds than subGaussian random vectors. We also have the following inclusion: $\operatorname{SG}(\sigma/\sqrt{d}) \subset \operatorname{normSG}(\sigma) \subset \operatorname{SG}(\sigma)$, see \citep{Jin2019nsg} for a detailed review of the connections between these two random vector classes.

Hereafter, we denote by $\tau$ an absolute (and possibly large) constant. We also writes $\tau = \tau(\gamma)$ to denote constants which depend on some parameter $\gamma$. These constants may change from line to line.
\begin{customtheorem}{\ref{thm:sugaussian_concentration}}[Concentration inequality subgaussian data] Let the cost $C=C^{n,p}$ be defined as in \eqref{DEF : mb_matrix_map}. Let $(\xx_i)_{1 \leq i \leq n} $ and $(\yy_i)_{1 \leq i \leq n} $ be two \textit{i.i.d.} sequences of random vectors such that $\xx_1 \in \operatorname{normSG}(\rho_\xx, \sigma^2 _\xx)$ and $ \yy_1 \in \operatorname{normSG}(\rho_\yy, \sigma^2 _\yy)$ with $\s_\xx, \s_\yy > 0$ and $\rho_\xx, \rho_\yy \in \R
^d$. Let us introduce 
\begin{align*}
    \s & := \min(\s_\xx, \s_\yy) \\
    \rho &  := \|\rho_\xx - \rho_\yy\|_2
\end{align*}
Let the sequence probability vectors $(\aa^{(n)}), (\bb^{(n)})$ verify $\mathtt{Loc_A}(m,\gamma,D)$ for some $\gamma \in (\frac{3}{4}, 1]$ and $D>0$. We assume that $n$ verifies the following condition:
\begin{equation}
    n \geq \tau(m, \s, \rho, D,p).
\label{hyp_n}
\end{equation}

\noindent
Consider $m \geqslant 1$ be a fixed integer and a kernel $h \in \{W_p, W_\epsilon, S_\epsilon$\}.  Let the reweighting function $w^\mathtt{W}$ and the probability law over $m$-tuple $P^\mathtt{W}$ be as in examples \ref{def:reweight_function_N} and \ref{def:law_indices_rep2}. Then we have the following concentration bound for the sampling without replacement:
\begin{equation}
 \P \left( \Big| \overline{h}_{w^\mathtt{W}, P^\mathtt{W}}(\aa^{(n)}, \bb^{(n)}) - \E \overline{h}_{w^\mathtt{W}, P^\mathtt{W}}(\aa^{(n)} , \bb^{(n)}) \Big| \geq  ( 2^{3p+4} m)^{\frac 12} \s^{p} D^2 \cdot \frac{\log(4n)^{\frac{p+1}{2}}}{n^{2 ( \gamma - \frac 34 )}}  \right) \leq 4 n^{-\frac{1}{2^p}},
\label{app : concentration bound subgaussian case}
\end{equation}

\noindent
\end{customtheorem}

\begin{proof} Let us fix $\e >0$ to be chosen later. We use the following notations for $\del >0$
\begin{align}
    \Delta & :=  \overline{h}_{w,P}(\aa^{(n)}, \bb^{(n)}) - \E \overline{h}_{w, P}(\aa^{(n)} , \bb^{(n)}) 
\label{Eq : def diff}, \\
A_{\del} & := \bigcap_{i= 1}^n \{ \| \xx_i - \rho_\xx \|_2 \leq \del  \} \cap \{ \| \yy_i - \rho_\yy \|_2 \leq \del  \}
\label{Eq : def event bounded support}
\end{align}
We estimate,
\begin{align}
\P[| \Delta | \geq \e] & \leq \P\big[ | \Delta | \mathbf{1}_{A_{\del}} \geq \frac{\e}{2}\big] + \P\big[ |\Delta | \mathbf{1}_{A^c_{\del}} \geq \frac{\e}{2}\big]  , \label{inter_subgaussian}
\end{align}
First, by the union bound, we have
\begin{equation}
  \P\big[ | \Delta  |\mathbf{1}_{A^c_{\del}} \geq \frac{\e}{2}\big] \leq \P \big[ A_{\del}^c \big]  \leq 4n  \exp \left( - \frac{ \del ^2}{2 \s ^2}  \right)
\label{prb1}
\end{equation}
Next, we claim that
\begin{equation}
\begin{split}
    \P\big[ | \Delta | \mathbf{1}_{A_{\del}} \geq \frac{\e}{2}\big] & \leq 2 \exp \Big(  - \frac{n^{4 ( \gamma - \frac 34 )}}{ 4^{p+2} D^4 (\rho + \del)^{2p}   } \e^2 \Big) \\
    & \qquad + 4n \exp\big( - \frac{\del^2}{2 \s^2}\big)
\end{split}
\label{prb2}
\end{equation}
assuming the following conditions together with \eqref{hyp_n},

\noindent
\begin{align}
    & \e \geq 4 n^{1 - 2 \gamma} \cdot  n^{\frac 15} \label{hyp_e} \\
    & \del \geq \s \sqrt{2 \log(n)} \label{hyp_del}.
\end{align}

\noindent
Let us show how \eqref{prb2} comes from a slight modification of the (proof) of \eqref{app:hoeffding_bounded_case_without_rep}. Using the same notations as in the proof of \eqref{app:hoeffding_bounded_case_without_rep} we have (in place of \eqref{EQ : id_T}):
\begin{equation*}
  \Delta \mathbf{1}_{A_{\del}} = \frac{1}{n!^2}  \sum_{\sigma_x, \sigma_y} T(\xx_{\sigma_x(1)}, \cdots, \xx_{\sigma_x(n)}, \yy_{\sigma_y(1)}, \cdots, \yy_{\sigma_y(n)})\mathbf{1}_{A_{\del}},
\end{equation*}
and
\begin{equation}\label{form_T}
   T(\xx_1, \cdots, \xx_n, \yy_1, \cdots, \yy_n) = \frac{1}{r}  \sum_{k=0}^{r-1} T_k \big( (\xx_i)_{i \in I^k}, (\yy_i)_{i \in I^k} \big).
\end{equation}
We emphasize that for each $1 \leq k \leq r-1$, $T_k$ depends only on the set $I^k := \{km+1, \cdots, km+m\}$. Moreover, the variables $(T_k)_k$ are mutually independent. 

As in the proof of Proposition \ref{thm:U_to_mean} is suffices to estimate $\mathbb{P}\left( \Delta \mathbf{1}_{A_{\del}}  \geq t\right)$ up to a factor $2$. We have,
\begin{align}\label{subg1}
\mathbb{P}\left( \Delta \mathbf{1}_{A_{\del}}  \geq t\right) \leq e^{-\lambda t}  \max_{\sigma} \mathbb{E}\left[e^{\lambda T_{\sigma} \ind_{A_{\del}} }\right].
\end{align}
For simplicity, we assume that the maximum in \eqref{subg1} is attained at $\s = \operatorname{Id}$.
Observe the following equality
\begin{align}\label{ind1}
    \ind_{A_{\del}} = \prod_{i=1}^n  \ind_{B(\rho_{\xx}, \del)}(X_i) \ind_{B(\rho_{\yy}, \del)}(Y_i)  
\end{align}
We write $T = T_{\operatorname{Id}}$ and insert the indicator function $\ind_{B(\rho_{\xx}, \del)^c}(X_1)$ using \eqref{ind1} and \eqref{form_T},
\begin{align}
\mathbb{E}\left[e^{\lambda T \ind_{A_{\delta} }}\right] & = \E \Big[ \exp\Big( \frac{\la}{r} \sum_{k=1}^{r-1} T_k \ind_{A_{\del}} \Big) \Big] \notag \\
& = \E \big[\ind_{B(\rho_{\xx}, \del)^c}(X_1) \big] +  \E \Big[ \exp\Big( \frac{\la}{r} \sum_{k=1}^{r-1} T_k \ind_{A_{\del}} \Big) \ind_{B(\rho_{\xx}, \del)}(X_1) \Big]. \notag
\end{align}
Hence, by repeating this procedure we find, using the tail estimate on the variables $X_k , Y_k \ (1\leq k \leq n)$: 
\begin{align}\label{goal_sub}
  \mathbb{E}\left[e^{\lambda T_{\operatorname{Id}} \ind_{A_{\delta} }}\right] \leq 4n \exp\big( - \frac{\del^2}{2 \s^2}\big) +   \E \Big[ \exp\Big( \frac{\la}{r} \sum_{k=1}^{r-1} T_k  \Big) \ind_{A_{\del}} \Big].
\end{align}
\medskip
\noindent
The first term in the right on side is the same as in $\eqref{prb1}$ so this loss is acceptable. In the following, we estimate $ \E \Big[ \exp\Big( \frac{\la}{r} \sum_{k=1}^{r-1} T_k  \Big) \ind_{A_{\del}} \Big]$. We denote, for each $0 \leq k \leq r-1$, by $A_{\del}(k)$ the set of data from $\XX(I^k)$ which belong to the ball of radius $\del$ around their mean:
\begin{equation*}
    A_{\del}(k) :=\bigcap_{i \in I^k } \{ \| \xx_{i} - \rho_\xx \|_2 \leq \del  \} \cap \{ \| \yy_{i} - \rho_\yy \|_2 \leq \del  \},
\end{equation*}
and the set of the remaining $n - rm$ data:
\begin{align*}
    A_{\del}(r) :=\bigcap_{i = rm+1 }^{n} \{ \| \xx_{i} - \rho_\xx \|_2 \leq \del  \} \cap \{ \| \yy_{i} - \rho_\yy \|_2 \leq \del  \}.
\end{align*}
Using the independence properties of the $T_k$'s recalled above we get
\begin{align}
  \E \Big[ \exp\Big( \frac{\la}{r} \sum_{k=1}^{r-1} T_k  \Big) \ind_{A_{\del}} \Big] & = \E \Big[ \prod_{k=1}^{r-1} \Big( \exp\big( \frac{\la}{r}  T_k  \big) \ind_{A_{\del}(k)}\Big) \ind_{A_{\del}(r)}  \Big] \notag \\
  & =  \prod_{k=1}^{r-1} \E \Big[ \exp\big( \frac{\la}{r}  T_k  \big) \ind_{A_{\del}(k)} \Big] \E \big[  \ind_{A_{\del}(r)}  \big] \notag \\
  & \leq  \prod_{k=1}^{r-1} \E \Big[ \exp\big( \frac{\la}{r}  T_k  \big) \ind_{A_{\del}(k)} \Big]. \label{subg3}
\end{align}
Hence, it suffices to estimate $\E \Big[ \exp\big( \frac{\la}{r}  T_k  \big) \ind_{A_{\del}(k)} \Big] $ for a fixed $0 \leq k\leq r-1$. Let $0 \leq k\leq r-1$. From Lemma \ref{app:upper_bound}, the $\mathtt{Loc_{A}}(m,\gamma,D)$ property of our variables, we have for $\omega \in A_{\del}(k)$:
\begin{align} \label{boundT_k}
\big|   T^{\omega}_k \big( (\xx_i)_{i \in I^k}, (\yy_i)_{i \in I^k} \big) \big|  \leq  \big( 2 (\rho + \del) \big)^p D^2 n^{2(1- \gamma)} =: \mathrm{M}_{\del}.
\end{align}

\noindent
Unfortunately, we can not use Lemma \ref{lemma : Hoef lem} on $T^{\omega}_k$ because we need to have a bounded random variable. To overcome this issue we introduce:
\begin{align}\label{T_k_del}
T^{\del}_k := T_k \ind_{|T_k| \leq \mathrm{M}_{\del}} + M_{\del} \ind_{|T_k| > \mathrm{M}_{\del}}.
\end{align}
which is bounded by $\mathrm{M}_{\del}$. Note that by construction the following equality holds 
\begin{align*}
  \E \Big[ \exp\big( \frac{\la}{r}  T_k  \big) \ind_{A_{\del}(k)} \Big] = \E \Big[ \exp\big( \frac{\la}{r}  T^{\del}_k  \big) \ind_{A_{\del}(k)} \Big].
\end{align*}
Thus we can compute using Lemma \ref{lemma : Hoef lem}:
\begin{align}
    \E \Big[ \exp\big( \frac{\la}{r}  T_k  \big) \ind_{A_{\del}(k)} \Big] & = \E \Big[ \exp\big( \frac{\la}{r}  (T^{\del}_k - \E \big[ T^{\del}_k \big] ) \big) \ind_{A_{\del}(k)} \Big]  \exp{\big( \frac{\la}{r} \E \big[ T^{\del}_k \big]\big)} \notag \\
     & \leq \exp{\big( \frac{\la^2 M_{\del}^2}{8 r^2} \big)}  \exp{\big( \frac{\la}{r} \E \big[ T^{\del}_k \big]\big)}. \label{subg4}
\end{align}

\noindent
Hence, taking the products in \eqref{subg4} for $0 \leq k \leq r$, we have

\noindent
\begin{align}
    \eqref{subg3} \leq \prod_{k=0}^{r-1} \exp{\big( \frac{\la^2 M_{\del}^2}{8 r^2} \big)}  \cdot \prod_{k=0}^{r-1} \exp{\big( \frac{\la}{r} \E \big[ T^{\del}_k \big]\big)}. \label{subg3_1}
\end{align}

From Lemma \ref{LEM:exp_T} proven below, we have

\noindent
\begin{align}\label{subgg}
    \big| \E \big[ T^{\del}_k \big] \big| \leq n^{1-2\gamma + \frac 15}
\end{align}
Thus, combining \eqref{subg3_1} and \eqref{subgg} we get by optimizing in $\la$ the inequalities,

\noindent
\begin{align}
\eqref{subg1} & \leq \exp \Big( \frac{\la^2 mM_{\del}^2}{4 n} - \la ( t  - n^{1 - 2 \gamma} \cdot  n^{\frac 15}  )  \Big) \notag \\
& \leq \exp \Big(  - \frac{n}{m \mathrm{M}_{\del}^2   } (t- n^{1 - 2 \gamma} \cdot  n^{\frac 15} )^2 \Big) \notag \\
& \leq \exp \Big(  - \frac{n}{ 4 m \mathrm{M}_{\del}^2   } t^2 \Big), \notag \\
& = \exp \Big(  - \frac{n^{4 ( \gamma - \frac 34 )}}{ 4^{p+1} m D^4 (\rho + \del)^{2p}   } t^2 \Big),
\end{align}
assuming 

\noindent
\begin{align}\label{cond_e_1}
t \geq 2 n^{1 - 2 \gamma} \cdot  n^{\frac 15},
\end{align}
which comes from \eqref{hyp_e} since $t = \frac{\e}{2}$.
\noindent
Together with \eqref{goal_sub}, this shows the claim \eqref{prb2}.

Thus, combining \eqref{prb1} and \eqref{prb2} gives
\begin{align}
\eqref{inter_subgaussian} & \leq2 \exp \Big(  - \frac{n^{4 ( \gamma - \frac 34 )}}{ 4^{p+2} D^4 (\rho + \del)^{2p}   } \e^2 \Big) + 8n  \exp \left( - \frac{ \del ^2}{2 \s ^2}  \right) \nonumber \\
     & =: \Xi_1 + \Xi_2
\label{Eq : main estimate subgaussian case}
\end{align}
Setting $\del = \del_n$ such that $ \Xi_1 = \Xi_2 $  we find that $\del_n$ satisfies the equation given by:
\begin{equation}
  P(\del):=  a \del ^2 ( \rho + \del )^{2p} + b (\rho+ \del) ^{2p} + c =0
\label{Eq : fourth order poly eq}
\end{equation}
with
\begin{align}
    a & := \frac{1}{2 \s ^2}  , \label{Eq:def_a} \\
    b & := - \log(4n)   , \label{Eq:def_b} \\
    c & :=  - \frac{n^{4 ( \gamma - \frac 34 )}}{ 4^{p+2} m D^4   } \e^2 \label{Eq:def_c} 
\end{align}
Such a $\del_n $ exists by the intermediate value theorem since $P(0)<0$ and $\lim_{\del \to + \infty}P(\del) = + \infty$.
Note that (\ref{Eq : fourth order poly eq}) writes
\begin{align}
    \frac{1}{2 \s ^2}  \del_n^2 (\rho+\del_n)^{2p} = \log(4n) (\rho+ \del_n)^{2p} + \frac{n^{4 ( \gamma - \frac 34 )}}{ 4^{p+2} m D^4   } \e^2,
\label{EQ : effective equation}
\end{align}
which implies
\begin{equation}
    \del_n \geq  \s \sqrt{2 \log(4n)} 
\label{EQ : lower bound growth del}
\end{equation}
We are interested in getting an upper bound of our quantity. Let us assume that we have:
\begin{align}
   \frac{n^{4 ( \gamma - \frac 34 )}}{ 4^{p+2} m D^4   } \e^2 \leq \log(4n) (\rho+ \del_n)^{2p}.
\label{EQ : cond RHS}
\end{align}
If (\ref{EQ : cond RHS}) holds then we have
\begin{align*}
    \frac{1}{2 \s ^2}  \del_n^2 (\rho+\del_n)^{2p} \leq 2 \log(4n) (\rho+ \del_n)^{2p},
\end{align*}
and hence
\begin{align}
    \del_n \leq  2 \s \sqrt{  \log(4n)},
\label{EQ : upper bound growth del}
\end{align} 
showing that (\ref{EQ : lower bound growth del}) is essentially sharp. We now investigate under which condition (\ref{EQ : cond RHS}) holds.  From (\ref{EQ : lower bound growth del}) the condition (\ref{EQ : cond RHS}) holds if we have
\begin{align*}
    \frac{n^{4 ( \gamma - \frac 34 )}}{ 4^{p+2} m D^4   } \e^2 \leq \log(4n) \big( \rho + \s \sqrt{2 \log(4n)}  \big) ^{2p}.
\end{align*}
It suffices to have,
\begin{align}
\label{cond_upp_bnd}
     \frac{n^{4 ( \gamma - \frac 34 )}}{ 4^{p+2} m D^4   } \e^2 \leq \log(4n) ( 2 \s^{2}  \log(4n)) ^{p} .
\end{align}
We thus choose the parameter $\e$ as follows

\noindent
\begin{align}
    \e^2 = 2^{3p+4} m \s^{2p} D^4 \cdot \frac{\log(4n)^{p+1}}{n^{4 ( \gamma - \frac 34 )}}.
\label{cond_e_2}
\end{align}
Note that the condition \eqref{cond_e_2} implies that \eqref{hyp_e} is verified for $n$ large enough (i.e. under \eqref{hyp_n}). If \eqref{cond_e_2} holds, then the condition \eqref{cond_upp_bnd} is verified and hence \eqref{EQ : upper bound growth del} holds. We now compute,
\begin{align*}
    (\ref{Eq : main estimate subgaussian case}) & = 2  \Xi_1  \\
    & \leq 4 \exp \Big(  - \frac{n^{4 ( \gamma - \frac 34 )}}{ 4^{p+2} m D^4 (\rho + \del)^{2p}   } \e^2 \Big) \\
    & \leq 4 \exp \Big( - \frac{ (2\s^2)^p \log(4n)^{p+1}}{(\rho + \del)^{2p}}   \Big) \\
    & \leq 4  \exp \Big( - \frac{ \log(4n) }{ 2^p }   \Big) \\
    & \leq 4 n^{-\frac{1}{2^p}}
\end{align*}
\end{proof}

\begin{lemma}\label{LEM:exp_T} Assuming the conditions \eqref{hyp_e}, \eqref{hyp_del} and \eqref{hyp_n} we have the following bound:
\begin{align*}
   \big| \E \big[ T^{\del}_k \big] \big| \leq n^{1-2\gamma + \frac 15}
\end{align*}
\end{lemma}

\noindent
\begin{proof}
We write 
\begin{align}\label{decomp_T_del}
\begin{split}
    \E \big[ T^{\del}_k \big] & = \E \big[ T^{\del}_k \ind_{A_{\del}(k)} \big] + \E \big[ T^{\del}_k  \ind_{A_{\del}(k)^c} \big] \\
    & =: \1 + \II
\end{split}
\end{align}
From \eqref{T_k_del} we observe using the definition of $ \mathrm{M}_{\del}$ in \eqref{boundT_k}, the union bound and the inequality $e^{-x} = e^{- \frac{1}{5} x} \cdot e^{- \frac{4}{5} x} \leq \frac{\tau}{x^p} \cdot  e^{- \frac{4}{5} x} $ for $x,p  \geq  1$ and some constant $\tau = \tau(p)$.
\begin{align}
| \II | & \leq \mathrm{M}_{\del} \cdot \P\big( A_{\del}(k)^c \big) \notag  \\
& \leq \big( 2 (\rho + \del) \big)^{p} D^2 n^{2(1- \gamma)} \cdot 4m \exp \left( - \frac{\del^2}{2 \s ^2} \right) \notag \\
& \leq \tau(D,m, \s,p) \cdot \del^p  n^{2(1- \gamma)} \cdot \frac{\tau(p)}{\del^{2p}} \exp \left( -  \frac{4}{5} \cdot \frac{\del^2}{ 2 \s ^2} \right) \notag \\
& \leq  \tau(D,m, \s, p) \cdot \frac{1}{  \log(n)^{\frac p2}  }  n^{2(1- \gamma)} \cdot n^{- \frac 45} \notag \\
& \leq \frac{1}{2} n^{1- 2 \gamma + \frac 15}. \label{subg5}
\end{align}
In the above, we used the assumptions \eqref{hyp_del} and \eqref{hyp_n} in the last two inequalities.

We now estimate the contribution of $\1$. Recalling the definition of $T_k^{\del} $ in \eqref{T_k_del} we observe the following using the definition of $\mathrm{M}_{\del}$ in \eqref{boundT_k} and the fact that $T_k$ is a mean-zero random variable,
\begin{align}
 \1 & = \E \big[ T_k  \ind_{|T_k| \leq \mathrm{M}_{\del} } \ind_{A_{\del}(k)} \big] =  \E \big[ T_k  \ind_{A_{\del}(k)} \big] \notag \\
    & = - \E \big[ T_k  \ind_{A_{\del}(k)^c} \big] \label{subg7}.
\end{align}
Note that the the indicator function $\ind_{A_{\del}(k)^c}$ can be expressed as the superposition of at most $4^m$ functions of the form
\begin{align}\label{ind2}
\begin{split}
   f(J_1, J_2) :=    \prod_{k_1 \in J_1} \ind_{B(\rho_\xx, \del)^c}(\xx_{k_1}) \cdot \prod_{k_2 \in J_2} \ind_{B( \rho_\yy, \del)^c}(\yy_{k_2}),
\end{split}
\end{align}
where $J_1, J_2 \subset I^k$ with $ |J_1| + |J_2| \geq 1$. Consider a positive real numbers $r$ and a vector $\zz$, we denote an annulus around the vector $\zz$ as $\mathbb{A}(\zz, r)$, \emph{i.e.,} $\xx \in \mathbb{A}(\zz, r)$ if $r \leq \|\xx - \zz \|_2 \leq 2r$. For $i \in \{ 1,2 \}$ let $t_i = |J_i|$ and write $J_{i} = \{j_1(i), \cdots, j_{t_i}(i) \}$. By decomposing each indicator function in \eqref{ind2} dyadically we get from the (mutual) independence of the variables $\{\xx_{i}, \yy_{j} : (i,j) \in J_1 \times J_2 \}$ and Lemma \ref{app:lemma_upper_bound},

\begin{align}
  &  \big| \E \big[ T_k f(J_1, J_2)  \big] \big|  = \notag  \\
  & = \Bigg| \E \Bigg[ T_k \prod_{i_1=1}^{t_1} \Big( \sum_{\ell( j_{i_1}(1)  ) \geq 0} \ind_{ \mathbb{A}(\rho_\xx, 2^{\ell(j_{i_1}(1))} \del) }(\xx_{j_{i_1}(1)}) \Big)  \times \prod_{i_2=1}^{t_2} \Big( \sum_{\ell(j_{i_2}(2)) \geq 0} \ind_{ \mathbb{A}(\rho_\yy, 2^{\ell(j_{i_2}(2))}\del) }(\yy_{  j_{i_2}(2)  } ) \Big) \Bigg] \Bigg| \\
  & =   \Bigg|  \sum_{\substack{\ell(j_1(1)), \cdots \ell(j_{t_1}(1))  \geq 0 \\ \ell(j_1(2)), \cdots \ell(j_{t_2}(2)) \geq 0} } \E \Bigg[  T_k \prod_{\substack{1 \leq i_1 \leq t_1 \\ 1 \leq i_2 \leq t_2}} \ind_{ \mathbb{A}(\rho_\xx, 2^{\ell(j_{i_1}(1))} \del) }(\xx_{j_{i_1}(1)})   \ind_{ \mathbb{A}(\rho_\yy, 2^{\ell(j_{i_2}(2))}\del) }(\yy_{  j_{i_2}(2)  } ) \Bigg] \Bigg| \notag \\
  \begin{split}\label{subg8}
 &  \leq  \sum_{\substack{\ell(j_1(1)), \cdots \ell(j_{t_1}(1))  \geq 0 \\ \ell(j_1(2)), \cdots \ell(j_{t_2}(2)) \geq 0} }  D^2 n^{2(1- \gamma)} (\rho + 2^{\ell_{\max} +1 } \del)^p \prod_{\substack{1 \leq i_1 \leq t_1 \\ 1 \leq i_2 \leq t_2}} \E  \Big[ \ind_{ \mathbb{A}(\rho_\xx, 2^{\ell(j_{i_1}(1))} \del) }(\xx_{j_{i_1}(1)}) \Big] \E \Big[ \ind_{\mathbb{A}(\rho_\yy, 2^{\ell(j_{i_2}(2))}\del) }(\yy_{  j_{i_2}(2)  } ) \Big],
  \end{split}
\end{align}

where $\ell_{max} = \max \big( \ell(j_1(1)), \cdots \ell(j_{t_1}(1)), \ell(j_1(2)), \cdots \ell(j_{t_2}(2)   \big)$. Hence, the subgaussianity assumption on the family $\{\xx_{i}, \yy_{j} : (i,j) \in J_1 \times J_2, \}$ yields using again the inequality $e^{-x} = e^{- \frac{1}{5} x} \cdot e^{- \frac{4}{5} x} \leq \frac{C}{x^{p}} \cdot  e^{- \frac{4}{5} x} $ for $x,p \geq 1$ and some constant $\tau = \tau(p)$,
\begin{align}
    \eqref{subg8} & \leq  D^2 n^{2(1- \gamma)} \sum_{\substack{\ell(j_1(1)), \cdots \ell(j_{t_1}(1))  \geq 0 \\ \ell(j_1(2)), \cdots \ell(j_{t_2}(2)) \geq 0} }   (\rho + 2^{\ell_{\max} +1 } \del)^p \notag \\
    & \qquad \times  \prod_{\substack{1 \leq i_1 \leq t_1 \\ 1 \leq i_2 \leq t_2}} \exp{ \Big(- \frac{(2^{\ell(j_{i_1}(1))} \del )^2}{2 \s^2 } \Big)} \exp{ \Big(- \frac{(2^{\ell(j_{i_2}(2))} \del )^2}{2 \s^2 } \Big)}  \notag \\
    & \leq \tau(p,m,\s,D) \cdot  n^{2(1-\gamma)} \exp{ \Big(- \frac{  2 \del^2}{5 \s^2 } \big(t_1 + t_2 
    \big) \Big)}  \notag \\
    & \qquad \times \sum_{\substack{\ell(j_1(1)), \cdots \ell(j_{t_1}(1))  \geq 0 \\ \ell(j_1(2)), \cdots \ell(j_{t_2}(2)) \geq 0} }  ( 2^{\ell_{\max}  } \del)^p \cdot \prod_{1 \leq i_1 \leq t_1} \frac{ 1 }{ ( 2^{ \ell(j_{i_1}(1)) } \del )^{2p} } \cdot \prod_{1 \leq i_2 \leq t_2} \frac{ 1 }{ ( 2^{ \ell(j_{i_2}(2)) } \del )^{2p} }   \notag \\
    & \leq  \tau(p,m,\s,D) \cdot   n^{2(1-\gamma)} \exp{ \Big(- \frac{  2 \del^2}{5 \s^2 } \big(t_1 + t_2 
    \big) \Big)}  \cdot  \frac{ \del^p }{\del^{2(t_1 + t_2)}} \cdot \Big( \sum_{\ell \geq 0} 2^{- p \ell} \Big)^{t_1 + t_2} \notag \\
    & \leq  \tau(p,m,\s,D) \cdot n^{2(1-\gamma)} \frac{1}{\del}  \exp{ \Big(- \frac{  2 \del^2}{5 \s^2 } \Big)}  \leq \tau(p,m,\s,D) \cdot   n^{2(1-\gamma)} \frac{1}{ n^{\frac 45} \sqrt{  \log(n) } } \label{subg9}.
\end{align}
In the second inequality, we used the fact that $\rho$ is upper bounded by $2^{\ell_{\max} +1 }\delta$ due to the hypotheses \eqref{hyp_del} and \eqref{hyp_n}. In the last inequality we used that $1 \leq t_1 + t_2 \leq 2m$. Hence, from \eqref{ind2} and \eqref{subg9} we estimate 
\begin{align}
   | \1 | & \leq \max_{(J_1, J_2)} \big| \E \big[ T_k f(J_1, J_2)  \big] \big| \notag \\
   & \leq  4^m \cdot \tau(p,m,\s,D) \cdot  n^{2(1-\gamma)} \frac{1}{ n^{\frac 45} \sqrt{  \log(n) } } \leq \frac 12 n^{1 - 2 \gamma} \cdot  n^{\frac 15}, \label{subg10}
   \end{align}
for $n$ as in \eqref{hyp_n}. Combining \eqref{subg9} and \eqref{subg5} with \eqref{decomp_T_del} yields the desired result.
\end{proof}

\noindent
Now that we have bound the deviation between the complete estimator and its expectation, let us bound the deviation between the complete estimator and its incomplete counter part.

We now discuss the difference between the deviation bounds of the estimator $\overline{h}$ and its mean in the bounded and unbounded data cases.
\begin{remark}
Theorem \ref{thm:sugaussian_concentration} holds when the distributions are compactly supported. Suppose we have probability sequences $(\aa^{(n)})$ and $(\bb^{(n)})$. Setting $\del = n^{- \frac{1}{2^p}}$ in \eqref{app:hoeffding_bounded_case_without_rep} gives:
\begin{align*}
\P \left( \big| \overline{h}_{w^\mathtt{W},P^\mathtt{W}}(\aa^{(n)}, \bb^{(n)}) - \E \overline{h}_{w^\mathtt{W},P^\mathtt{W}} (\aa^{(n)}, \bb^{(n)}) \big| \geq  2^{\frac{3}{2} - \frac{p}{2}} M D^2 \frac{m^{ \frac12 }}{n^{2( \gamma - \frac{3}{4}) }} \sqrt{\log(n)} \right) \leq n^{- \frac{1}{2^p}}
\end{align*}

\noindent
Hence we essentially lose a $\log(n)^{\frac{p}{2}}$ factor in comparison with \eqref{app : concentration bound subgaussian case}.
\end{remark}
\begin{remark} The proof of Lemma \ref{app:inc_U_to_U} also yields
\begin{equation}
\P \left( \big|  \widetilde{h}_{w,P}^k(\aa^{(n)}, \bb^{(n)}) - \overline{h}_{w,P}(\aa^{(n)}, \bb^{(n)} ) \big| \geq M \sqrt{\frac{2 \log(2/\delta)}{k}} \Big| \xx_1, \yy_1, \cdots, \xx_n, \yy_n \right) \leq \delta
\end{equation}
for subgaussian data $ \xx_1, \yy_1, \cdots, \xx_n, \yy_n$ and any $\del>0$.
\end{remark}

\subsection{Distance to marginals}\label{app_sec:distance_marginals}
In this section, we give the details of the proof of Theorem \ref{thm:dist_marg}.
In what follows, we denote by $\Pi_{(i)}$ the $i$-th row of matrix $\Pi$. Let us denote by $ \mathbf{1} \in \R^n $ the vector whose entries are all equal to $1$.
\begin{customtheorem}{\ref{thm:dist_marg}}[Distance to marginals] Let $ \delta \in (0,1) $, two integers $m \leq n $ and consider two sequences of probability vectors $(\aa^{(n)}),(\bb^{(n)}) \in \Sigma$. Let a ground cost $C = C^{m,p}$ for some $p \geq 1$. Consider an OT kernel $h \in \{ W_p, W_p^p, W^{\varepsilon}, S^{\varepsilon}, \GW\}$. Suppose now that the probability law over $m$-tuples $P$ and the reweighting function $w$, as defined in \eqref{EQ : rew_fun} and \eqref{EQ : prob_fun}, satisfy the admissibility condition \eqref{eq:wPadmissible}. For all integers $ k \geqslant 1 $ and all integers $ 1 \leqslant i \leqslant n  $, we have:
\begin{equation}
\P \left(
\Big|  \widetilde{\Pi}^{h,k}_{w,P}(\aa^{(n)}, \bb^{(n)})_{(i)} \mathbf{1} - a_i^{(n)} \Big| \geq  \sqrt{\frac{2 \log(2/\delta)}{k}} \right) \leq \delta
\end{equation}
\end{customtheorem}

\begin{proof}
Let us recall that thanks to the admissibility condition \eqref{eq:wPadmissible}, $\overline{\Pi}^h_{w,P}$ is a transport plan between the input probability vectors $\aa^{(n)}$ and $\bb^{(n)}$ and hence, it verifies the marginal constraints, i.e $(\overline{\Pi}^h_{w,P}(\aa^{(n)},\bb^{(n)}))_i \times \mathbf{1} = a_i$. 
Thanks to Remark \ref{inc_ber} we have

\noindent
\begin{equation*}
\widetilde{\Pi}^{h,k}_{w,P}(\aa^{(n)},\bb^{(n)})_{(i)} \mathbf{1}= \frac{1}{k} \sum_{\ell=1}^k \omega_{\ell}
\end{equation*}

\noindent
where $ \omega_\ell = \sum_{I, J \in (\llbracket n \rrbracket^m)^2} \sum_{j=1}^n (\Pi_{I,J})_{i,j} \mathfrak{b}_{\ell}^{P_{\aa^{(n)}},P_{\bb^{(n)}}}(I,J) $. Conditioned upon $ \XX = (\xx_1, \cdots, \xx_n)$ and $\YY = (\yy_1, \cdots, \yy_n) $, the random vectors $ \omega_p $ are independent, and bounded by $ 1 $.
Moreover, one can observe that $ \mathbb{E}[ \widetilde{\Pi}^{h,k}_{w,P}(\aa^{(n)},\bb^{(n)})_i \mathbf{1} ] = \overline{\Pi}^h_{w,P}(\aa^{(n)},\bb^{(n)})_i \mathbf{1} $. Using Hoeffding's inequality yields

\noindent
\begin{align*}
\mathbb{P}( \V  \widetilde{\Pi}^{h,k}_{w,P}(\aa^{(n)},\bb^{(n)})_i \mathbf{1}  - \overline{\Pi}^h_{w,P}(\aa^{(n)},\bb^{(n)})_i \mathbf{1} )  \V > \varepsilon  ) & = \mathbb{E} [ \mathbb{P}(  \V \frac{1}{k} \sum_{p=1}^k \omega_p  - \mathbb{E}[ \frac{1}{k} \sum_{p=1}^k \omega_p]) \V  > \varepsilon  \vert \XX,\YY )] \\
& \leqslant  2 e^{-2k \varepsilon^2}
\end{align*}

\noindent
which concludes the proof.
\end{proof}

\subsection{Optimization}\label{app_sec:optimization}

In this section we provide the full statements and proofs of Theorem \ref{thm:derivative_OT_cost} and Theorem  \ref{thm:exchange_grad_exp_sm}: 

\begin{customtheorem}{\ref{thm:derivative_OT_cost}}\label{thm:derivative_OT_cost_app}
Let $\aa, \bb \in \Sigma_m$. Let $\XX$ be a $\R^{dm}$-valued random variable, and $\{\YY_{\theta} \}$ a family of $\R^{dm}$-valued random variables defined on the same probability space, indexed by $\theta \in \Theta$, where $\Theta \subset \R^{q}$ is open. Assume that $\theta \mapsto \YY_{\theta}$ is $C^1$. Denote $C = C^{m,p}$ for some $p \geq 1$ and let $h \in \{W, W^{\epsilon}\}$. Then the function $\theta \mapsto -h(\aa,\bb,C(\XX, \YY_{\theta}))$ is Clarke regular and for all $1 \leq i \leq q$ we have:
\begin{align} \label{eq:exchange_theorem_eq1_proof}
\partial_{\theta_i} h(\aa,\bb,C(\XX, \YY_{\theta})) = \{ -\text{tr}(P \cdot D^{T})\cdot (\nabla_{\theta_i} Y): & P \in \Pi(h, C(\XX, \YY_{\theta}), \aa, \bb), \\
D \in \R^{m, m}, \hspace{2pt} & D_{j,k} \in \partial_{Y} C_{j,k}(\XX, \YY_{\theta}) \} \nonumber
\end{align}
where $\partial_{\theta_i}$ is the Clare subdifferential with respect to $\theta_i$, $\partial_Y C_{j,k}$ is the subdifferential of the cell $C_{j,k}$ of the  cost matrix with respect to $Y$ and $\Pi(h, C,\aa, \bb)$ is defined in definition \ref{def:MBTP}. 

For $h = GW$ and $p >1$, the function $-h(\aa, \bb, C(\XX, \XX), C(\YY_{\theta},\YY_{\theta}))$ is also Clarke regular, and we have:
\begin{align} \label{eq:exchange_theorem_GW_subdiff}
\partial_{\theta_i} h(\aa,\bb, C(\XX, \XX), C(\YY_{\theta}, \YY_{\theta})) & = \{ - \sum_{j_1, j_2,k_1,k_2 = 1}^{m} D_{j_1,k_1,j_2,k_2} P_{j_1,k_1} P_{j_2,k_2} \cdot (\nabla_{\theta_i} Y)  : \\ P \in & \Pi(h, C(\XX, \XX), C(\YY_{\theta}, \YY_{\theta}), \aa, \bb) \\ D \in & (\R^{m})^4, \hspace{2pt}  D_{j_1,k_1,j_2,k_2} = \nabla_Y C_{j_1,k_1,j_2,k_2}(\XX, \YY_{\theta}) \} \nonumber
\end{align}
where $C_{j_1,k_1,j_2,k_2} = \Vert C_{j_1,k_1}(\XX, \XX) - C_{j_2,k_2}(\YY_{\theta}, \YY_{\theta}) \Vert^p$. 

\end{customtheorem}
\begin{proof}
We start with the case $h \in \{W, W^{\epsilon}\}$. The function $Y \mapsto C_{j,k}(X, Y)$ is equal to $\Vert \xx_{j} - \yy_{k} \Vert_2^p$. It is therefore convex, and thus Clarke regular by Proposition 2.3.6(b) \citep{clarke1990optimization}. Since $\theta \mapsto \YY_{\theta}$ is $C^1$, from Theorem 2.3.10 \citep{clarke1990optimization} it follows that $\theta \mapsto C_{j,k}^{m,p}(\XX, \YY_{\theta})$ is Clarke regular, and:
\[
\partial_{\theta_i} C_{j,k}(\XX, \YY_{\theta}) = \{ D_{j,k} \cdot \nabla_{\theta_i} Y_{\theta} : D_{j,k} \in \partial_Y C_{j,k}(\XX, \YY_{\theta}) \}.
\]
Note, that the set of admissible transport plans for any marginals $\aa, \bb \in \Sigma^m$ is compact. Furthermore, the transport cost for a given plan $P$ is a linear function of cost matrix $C$. Therefore, from Danskin's Theorem (Proposition B.25 \cite{bertsekas1997nonlinear}) it follows that for $h \in \{W, W^{\epsilon} \}$ the function $C \mapsto -h(\aa,\bb, C)$ is convex, and it's subderivative is equal to $\Pi(h,C, \aa, \bb)$. 
Therefore from Theorem 2.3.9(i) and Proposition 2.3.1 (for $s=1$) in \citep{clarke1990optimization} it follows that $\theta \mapsto -h(\aa,\bb, C(\XX, \YY_{\theta}))$ is Clarke regular for $h \in \{W, W^{\epsilon} \}$, and that \eqref{eq:exchange_theorem_eq1_proof} holds.

Assume now that $h = GW$ and that $p > 1$. The proof is analogous. In this case, the function $Y \mapsto C_{j_1, k_1, j_2, k_2}(\XX, \YY)$ is differentiable. Therefore the function $\theta \mapsto C_{j_1, k_1,j_2, k_2}(\XX, \YY_{\theta})$ is differentiable, hence Clarke regular by Proposition 2.3.6(a)\citep{clarke1990optimization}. Again the set of admissible transport plans is compact and for a given transport plan, the transpot cost is a linear function of the four dimensional tensor $C_{j_1, k_1, j_2, k_2}$. Therefore, using Danskin's Theorem (Proposition B.25 \cite{bertsekas1997nonlinear}), as well as Theorem 2.3.9(i) and Proposition 2.3.1 in \citep{clarke1990optimization} we get that $-h(\aa, \bb, C(\XX, \XX), C(\YY_{\theta}, \YY_{\theta}))$ is Clarke regular and the formula \eqref{eq:exchange_theorem_GW_subdiff} holds.
\end{proof}
\begin{customtheorem}{\ref{thm:exchange_grad_exp_sm}}\label{thm:exchange_grad_exp_sm_app}
Let $\aa, \bb, \XX, \YY$ be as in theorem \ref{thm:derivative_OT_cost}, $h \in \{W, W^{\epsilon} \}$, and assume in addition that the random variables $\XX, \{Y_{\theta}\}_{\theta \in \Theta}$ have finite $p$-moments. If for all $\theta \in \Theta$ there exists an open neighbourhood $U$, $\theta \in U \subset \Theta$, and a random variable $K_U : \Omega \rightarrow \R$ with finite expected value, such that
\begin{equation}\label{eq:integrable_condition_opt}
\Vert C(\XX(\omega), \YY_{\theta_1}(\omega)) - C(\XX(\omega), \YY_{\theta_2}(\omega) ) \Vert \leq K_U(\omega) \Vert \theta_1 - \theta_2 \Vert
\end{equation}
then we have
\begin{align} \label{eq:exchange_theorem_eq2_proof}
\partial_{\mathbf{\theta}} \expect \left[ h(\aa, \bb, C(\XX, \YY_{\theta})) \right] =  \expect \left[ \partial_{\mathbf{\theta}} h(\aa, \bb, C(\XX, \YY_{\theta})) \right].
\end{align}
with both expectation being finite. Furthermore the function $\theta \mapsto - \expect \left[ h(\aa, \bb, C(\XX, \YY_{\theta})) \right]$ is also Clarke regular. 

For $h = GW$, assume that $p > 1$ and that random variables $\XX, \{\YY_{\theta}\}$ have finite $2p$-moments. Assume also that for each $\theta \in \Theta$ there exists an open neighbourhood $U$, $\theta \in U \subset \Theta$, and a random variable $K_U : \Omega \rightarrow \R$ with finite expected value, such that
\begin{equation}\label{eq:integrable_condition_opt_G}
\Vert \tilde{C}(\XX(\omega), \YY_{\theta_1}(\omega)) - \tilde{C}(\XX(\omega), \YY_{\theta_2}(\omega) ) \Vert \leq K_U(\omega) \Vert \theta_1 - \theta_2 \Vert
\end{equation}
where $\tilde{C}(\XX, \YY) = \Vert C(\XX,\XX) - C(\YY, \YY)\Vert^p$. Then we have
\begin{align} \label{eq:exchange_theorem_eq_GW}
\partial_{\mathbf{\theta}} \expect \left[ h(\aa, \bb, C(\XX, \XX), C(\YY_{\theta}, \YY_{\theta})) \right] =  \expect \left[ \partial_{\mathbf{\theta}} h(\aa, \bb, C(\XX, \XX), C(\YY_{\theta}, \YY_{\theta})) \right].
\end{align}
with both expectation being finite. Furthermore the function $\theta \mapsto - \expect \left[ h(\aa, \bb, C(\XX, \XX), C(\YY_{\theta},\YY_{\theta})) \right]$ is also Clarke regular.
\end{customtheorem}
\begin{proof}
We start with the case $h \in \{W, W^{\epsilon} \}$. Suppose that $U \subset \Theta$ is open and $K_U$ is a function for which \eqref{eq:integrable_condition_opt} is satisfied. Then the same bound is also satisfied for the function $h(\aa, \bb, C(\XX(\omega), \YY_{\theta}(\omega)))$, since the function $C \mapsto h(\aa, \bb, C)$ is $1$-Lipshitz. Hence, given the regularity of $-h(\aa,\bb,C(\XX, \YY_{\theta}))$, the interchange \eqref{eq:exchange_theorem_eq2_proof} and regularity of $\theta \mapsto -\expect[h(\aa,\bb, C(\XX, \YY_{\theta}))]$ will follow from Theorem 2.7.2 and Remark 2.3.5 \citep{clarke1990optimization}, once we establish that the expectation on the left hand side is finite. This follows trivially from the standard bound:
\begin{equation} \label{eq:p_moment_bound}
\Vert \xx - \yy \Vert^p \leq 2^{p-1} (\Vert \xx \Vert^p + \Vert \yy \Vert^p)
\end{equation}
and the assumption that $\XX, \YY_{\theta}$ have finite $p$-moments. The same argument applies to the case when $h = GW$. The $GW$ cost depends on the four-dimensional tensor $C_{j_1,k_1,j_2,k_2}$ defined in the proof of Theorem \ref{thm:derivative_OT_cost_app} in a Lipshitz manner, since it's supdifferential is bounded. Again, the thesis for $h=GW$ will follow from Theorem 2.7.2 and Remark 2.3.5 \citep{clarke1990optimization}, once we establish that the expectation on the left hand side of \eqref{eq:exchange_theorem_eq_GW} is finite. This follows from applying the bound \eqref{eq:p_moment_bound} twice and the assumption of finite $2p$-moments.
\end{proof}

\subsection{1D case}\label{app_sec:1D_case}
We now give the full combinatorial calculus for the 1D case. We start by sorting all the data and give to each of them an index which reprensents their position after the sorting phase. Then we select and sort all the minibatches. $x_j$ can not be at a position superior to its index $j$ inside a batch. For a fixed $x_j$, a simple combinatorial arguments tells you that there are $C_{x_j}^i$ sets where $x_j$ is at the $i$-th position:
\begin{equation}
C_{i, x_j}^{m, n} = \dbinom{j-1}{i-1} \dbinom{n-j}{m-i}
\end{equation}

Suppose that $x_j$ is transported to a $y_k$ points in the target mini batch. Then, they both share the same positions $i$ in their respective minibatch. As there are several $i$ where $x_j$ is transported to $y_k$, we sum over all those possible positions. Hence our current transport matrix coefficient $\Pi_{j,k}$ can be calculated as :
\begin{equation}
\Pi_{j,k} = \sum_{i=i_{\text{min}}}^{i_{\text{max}}} C_{i, x_j}^{m, n} C_{i, y_k}^{m, n}
\end{equation}

Where  $i_{\text{min}} = \text{max}(0, m-n+j, m-n+k)$ and $i_{\text{max}} = \text{min}(j, k)$. $i_{\text{min}}$ and $i_{\text{max}}$ represent the sorting constraints. Furthermore, as we have uniform weight histograms, we will transport a mass of $\frac{1}{m}$ and averaged it by the total number of transport. So finally, our transport matrix coefficient $\Pi_{j,k}$ are:

\begin{equation}
\Pi_{j,k} = \frac{1}{m \dbinom{n}{m}^2} \sum_{i=i_{\text{min}}}^{i_{\text{max}}} C_{i, x_j}^{m, n} C_{i, y_k}^{m, n}
\end{equation}

The sampling with replacement case is much more complex and highly computationally costly. Following the same strategy as above, we sort all minibatches. In this case, $x_i$ might appears several times more or less than $y_j$ and we need to take that into account. We denote $m_i$ the number of repetitions of the $i$-th element and the summation $|\vec{m}|=\sum_{k=1}^{m} m_{k}$.  We denote $N_{m_1, m_2}^{i,j}$ the number of times that $x_i$ and $y_j$ share the same position in their respective minibatches $m_1$ and $m_2$, after sorting. For the coefficient $\pi_{i,j}$, we have : 

\begin{equation}
    \pi_{i,j} = \frac{1}{m} \left( \frac{1}{m n^{m-1}} \right)^2 \sum_{|\vec{m_s}| = m} \sum_{|\vec{m_t}| = m} N_{m_s, m_t}^{i,j}
\end{equation}

\subsection{Contributions}\label{app_sec:contributions}

In this section, we state the contributions of each author on each part of the present manuscript.
\begin{itemize}
    \item Formalism : KF, YZ, RG
    \item 1D/2D : KF, YZ, RF, NC
    \item Loss properties : KF, YZ, RF, RG, NC
    \item Debiased loss : KF, YZ, RF, NC, RG
    \item Statistical properties : YZ, KF
    \item Optimization : SM, KF, YZ
    \item Experiments : KF
    \item Experiments review : KF, RF, NC
    \item Writing--original draft preparation : KF
    \item Writing—review and editing : KF, YZ, SM, RF, RG, NC
    \item Supervision : RF, NC, RG
    \item Project administration : RF, NC
    \item Funding acquisition : RF, NC
\end{itemize}

\end{document}